\pdfoutput=1

\documentclass[11pt,a4paper]{article}

\usepackage[colorlinks=true,linkcolor=blue]{hyperref}
\usepackage{cmap}
\usepackage[utf8]{inputenc}
\usepackage[T1]{fontenc}
\usepackage[english]{babel}
\usepackage{amssymb,amsmath}
\usepackage{amsthm}
\usepackage{bbm}
\usepackage{microtype}
\usepackage{lmodern}

{
\rmfamily
\DeclareFontShape{T1}{lmr}{m}{sc}{<->ssub*cmr/m/sc}{}
\DeclareFontShape{T1}{lmr}{b}{sc}{<->ssub*cmr/b/sc}{}
\DeclareFontShape{T1}{lmr}{bx}{sc}{<->ssub*cmr/bx/sc}{}
}

\newcommand{\thmheadercommand}[1]{\textbf{\scshape{}#1.\\*}}

\newtheoremstyle{yannthm}{\topsep}{\topsep}{\slshape}{}{\scshape\bfseries}{.}{.5em}{%
\thmname{#1}\thmnumber{ #2}\thmnote{#3}%
}
\newtheoremstyle{yannthm2}{\topsep}{\topsep}{}{}{\scshape\bfseries}{.}{.5em}{%
\thmname{#1}\thmnumber{ #2}\thmnote{#3}%
}

\def\d{\operatorname{d}\!{}}

\def\Z{{\mathbb{Z}}}

\def\R{{\mathbb{R}}}

\renewcommand{\geq}{\geqslant}
\renewcommand{\leq}{\leqslant}

\renewcommand{\emptyset}{\varnothing}

\newcommand{\deq}{\mathrel{\mathop{:}}=}

\newcommand{\from}{\colon} 

\def\eps{\varepsilon}
\renewcommand{\epsilon}{\varepsilon}
\renewcommand{\phi}{\varphi}

\let\oldPr\Pr
\renewcommand{\Pr}{\oldPr\nolimits}
\newcommand{\E}{\mathbb{E}}

\DeclareMathOperator{\Ker}{Ker}
\DeclareMathOperator{\Img}{Im}
\DeclareMathOperator{\Tr}{Tr}

\DeclareMathOperator{\Id}{Id}

\DeclareMathOperator{\diag}{diag}

\newcommand{\abs}[1]{\left\lvert#1\right\rvert}
\newcommand{\norm}[1]{\left\lVert#1\right\rVert}

\newcommand{\1}{\mathbbm{1}}

{
\theoremstyle{yannthm}
\newtheorem{defi}{Definition}
\newtheorem*{defi*}{Definition}
\newtheorem{prop}[defi]{Proposition}
\newtheorem*{prop*}{Proposition}
\newtheorem{thm}[defi]{Theorem}
\newtheorem*{thm*}{Theorem}
\newtheorem{lem}[defi]{Lemma}
\newtheorem*{lem*}{Lemma}

\newtheorem*{cor*}{Corollary}

\newtheorem*{ex*}{Example}

\newtheorem*{subenonce}{}

\theoremstyle{yannthm2}

\newtheorem*{exo*}{Exercise}

\newtheorem{rem}[defi]{Remark}
\newtheorem*{rem*}{Remark}

\newtheorem*{subenonce2}{}

}

\newcommand{\transp}[1]{#1^{\!\top}\!}

\usepackage{natbib}

\usepackage{xcolor}
\usepackage[pdftex]{graphicx}
\usepackage{wrapfig}
\usepackage{caption}
\usepackage{subcaption}

\newcommand{\todo}[1]{} 
\newcommand{\option}[1]{#1} 

\def\del{\operatorname{\delta}\hspace{-.4ex}{}} 

\newcommand{\tar}{\mathrm{tar}}

\newcommand{\starget}{s_{\mathrm{tar}}}
\newcommand{\Prob}{\mathbb{P}}
\newcommand{\tildem}{\tilde m}
\newcommand{\tildeq}{\tilde q}

\newcommand{\diagrho}{\grave \rho}

\newcommand{\rhosimple}{\rho_{\mathrm{simple}}}
\newcommand{\rhoref}{\rho_{\mathrm{ref}}}
\newcommand{\rhoSA}{\rho_{\mathrm{SA}}}

\title{Learning Successor States and Goal-Dependent Values: A Mathematical Viewpoint}

\author{%
Léonard Blier, Corentin Tallec, Yann Ollivier
}

\begin{document}

\maketitle

\begin{abstract}
In reinforcement learning, temporal difference-based algorithms can be
sample-inefficient: for instance, with sparse rewards, no learning occurs
until a reward is observed. This can be remedied by
learning richer objects, such as a model of the environment, or
\emph{successor states}. Successor states model the expected future state
occupancy from any given state \citep{dayan1993improving,kulkarni2016deep}, and summarize all paths in
the environment for a given policy. They are related to
\emph{goal-dependent value functions}, which learn how to reach arbitrary
states.

We formally derive the temporal difference algorithm for successor state
and goal-dependent value function learning, either for discrete or for
continuous environments with function approximation. Especially, we
provide finite-variance estimators even in continuous environments,
where the reward for exactly reaching a goal state becomes infinitely
sparse.

Successor states satisfy more than just the Bellman equation: a
\emph{backward} Bellman operator and a \emph{Bellman--Newton} (BN) operator
encode path compositionality in the environment.  The BN
operator is akin to second-order gradient descent methods, and provides
the ``true'' update of the value function when acquiring more
observations from the environment, with explicit tabular bounds.  In the tabular case and with
infinitesimal learning rates, mixing the usual and backward Bellman
operators provably improves eigenvalues for asymptotic convergence, and
the asymptotic convergence of the BN operator is provably
better than TD, with a rate independent from
the environment. However, the BN method is more complex and
less robust to sampling noise.

Finally, a \emph{forward-backward} (FB) finite-rank parameterization of
successor states enjoys reduced variance and improved samplability,
provides a direct model of the value function, has fully understood fixed
points corresponding to long-range dependencies (but ignores small-scale
dependencies), approximates the BN method, and provides two canonical
representations of states as a byproduct.
\end{abstract}

\todo{
Mention that $L^1$ or $L^2$ error on $M$ is expected error on the value
function for a random goal state.
\\
NOT FOR NOW:\\
Check param $M=m\rho$ is included everywhere
\\
trace formulas
\\
Diracs: is it true it's better to take $s_2=s$ in the discrete case? cf
sampling the Id term helps\\
target network argument?\\
continuity L2/L1 norm\\
Gradient formulas\\
continuous-time convergence analysis with $\rho$?\\ + linear case?
\\
$FMB$ in section FB vs BN
\\Notation $g$ for $s_2$? hard to do consistently
\\Small-rank model of $P(s,\d s')=\transp{\phi(s)}\psi(s')\rho(\d s')$:
then \\$M(s_1,s_2)=\Id+\gamma \transp{\phi(s_1)} (\Id-\gamma \E_{s\sim \rho}
\psi(s)\transp{\phi(s)})^{-1}\psi(s_2)\rho(\d s_2)$. Interesting? amounts
to learning features where $P$ is linear, and using the exact $M$ of that
model. The inner term $\Id-\gamma \E_{s\sim \rho}
\psi(s)\transp{\phi(s)}$ is a finite-sized matrix that can be estimated
and inverted explicitly. Nothing to learn for $M$, but learning $P$
becomes a density estimation problem. Check literature, this model has
been considered before.
}

\tableofcontents

\section{Introduction, Overview of Results}
\label{sec:introduction}


The \emph{successor state operator} of a Markov reward process is an
object that directly encodes the passage from a reward function to the
corresponding value function. In particular, it expresses the value
functions of all possible reward functions for a given, fixed policy.

\emph{Goal-dependent value functions} are a related object with many
similar properties.  They describe the optimal value functions and
policies for a specific set of tasks: typically, for all rewards located at all
possible target states. In this case, the policy depends on the target
state.

Here we offer a formal treatment of these objects in both finite and
continuous spaces. We present several learning algorithms and associated
results. In particular, we focus on proper treatment of the infinitely-sparse
reward problem encountered by TD-style approaches in continuous spaces if
the reward is located at a precise state.

Possible advantages of working with these objects include:
\begin{itemize}
\item Contrary to TD, learning starts even before any rewards are
observed. Sucessor
state learning extracts information from every observed transition, by
learning how to reach every visited state.
Subsequent reward observations provide an instantaneous
update to the value function via the successor state operator.

This learning is done without reward signals, illustrating an
``unsupervised reinforcement learning'' approach. 
Successor state lie in between model-free and model-based reinforcement
learning approaches, providing a 
representation of the future of a state without having to synthesize
future states or unrolling synthetic trajectories. Algorithmically, they
rely on having two states as inputs rather than generating a state.

\item Successor states and goal-dependent values exploit relationships
between how to reach different states. With function approximation,
generalization occurs between different target states. But even in a
tabular setting with no generalization, these objects satisfy more
algebraic relations than the usual Bellman equation: a \emph{backward}
Bellman equation and a \emph{Bellman--Newton} equation, expressing path
compositionality in the Markov process (Fig~\ref{fig:combining_paths}).
This leads to quantifiable asymptotic gains.

\item
Successor states and goal-dependent values can be used to solve
several problems at once, such as learning to
reach arbitrary states.
Even for optimizing a single reward, 
they can be used for auxiliary tasks such as going to an arbitrary
state, which could be useful for exploration, or to provide good state
representations.
\end{itemize}

For learning value functions dependent on goal states, an obvious
approach is to apply any standard reinforcement learning algorithm, with
reward $1$ when the visited state is equal to the goal state (e.g.,
\citep{pmlr-v37-schaul15}). But this breaks down in continuous spaces, as
the reward function becomes infinitely sparse (a random trajectory is
never going to reach any predefined goal exactly). Even in discrete
spaces, the reward becomes exponentially sparse as the number of
components increase.

This problem is avoided by a suitable mathematical treatment.
The intuition behind several of our results is the following: If the goal
is to learn how to reach arbitrary states, then this is not a sparse
reward problem, although straightforward TD implementations treat it as
such; it is a problem with rewards everywhere. Approaches such as
\emph{hindsight experience replay} \citep{HER} attempt to exploit this
intuition by resampling goals a posteriori in an off-policy algorithm, but
it is unclear to us how much of the problem HER solves in continuous
spaces. 
The mathematical treatment here proves that finite-variance algorithms
exist for such problems, even in continuous spaces. 


\paragraph{Overview of results.}
In a nutshell, successor states summarize all possible paths in the
environment for a given policy (Section~\ref{sec:paths}).  For finite
spaces, the entries $M_{ss'}$ of the successor state matrix describe the
expected discounted time spent in state $s'$ by a trajectory starting at
$s$ \citep{dayan1993improving}: $M_{ss'}=\E\left[\sum_{t\geq 0}\gamma^t
\1_{s_t=s'}|s_0=s\right]$. The entry $M_{ss'}$ is also the value function
at $s$ if the reward is $1$ at $s'$ and $0$ everywhere else. As such, $M$
contains the information about reaching every state in the environment,
not just those states providing a reward. For a fixed policy, the value
function depends linearly on the reward: in a finite state space, for any
reward function, represented as a vector $R$ over states, its associated
value function is $V=MR$.

\todo{keep these explanations?}

The goal-dependent value function $V_{ss'}$ at state $s$ for goal $s'$
(another state) is defined as the value function at $s$ of the optimal
policy for reaching a unit reward located at $s'$. The difference with
$M_{ss'}$ is now that the policy depends on $s'$ instead of being fixed.
Learning this object allows for learning how to reach different goal
states. Contrary to $M_{ss'}$, $V_{ss'}$ does not contain information on
how to optimize dense rewards (mixtures of goal states), only rewards
located at a single state. It is also possible to define $V$ for more
general types of goals rather just a target state, although the goals
must be predefined and mixtures of goals are not possible a posteriori.

The bulk of the text presents theoretically well-motivated algorithms to
learn these objects
directly for any two states $s,s'$.
The main contributions of this text are the following.

\begin{itemize}

\item We formally define successor states and goal-dependent value (and
$Q$) functions in general state spaces (Sections~\ref{sec:defM} and
\ref{sec:goals}), extending the discrete case of
\citep{dayan1993improving}. For continuous states, this involves some
measure theory (Section~\ref{sec:defM_cont}), but the intuition is clear from the discrete case
(Section~\ref{sec:defM_discrete}). 

Successor states are always well-defined for a given policy
(Theorem~\ref{thm:Mdef}). But goal-dependent value functions are
generally not unique in continuous spaces
(Section~\ref{sec:Qoptexistence}); still, there exists a canonical
solution (Theorem~\ref{thm:optQ}), smaller than all others. 


\item We formally derive the 
temporal difference (TD) algorithm for successor state learning, both for
discrete spaces, and for continuous spaces with function approximation
(Theorem~\ref{thm:paramtd}), beyond the tabular setting
of \citep{dayan1993improving}.
A naive application of TD on a state-goal product space, with reward $1$
when the state reaches the goal, degenerates in continuous spaces: the
reward becomes infinitely sparse (it is $0$ with probability $1$ and
$\infty$ with probability $0$). Instead, the TD estimators we
provide have finite variance
(Section~\ref{sec:sparse}, Proposition~\ref{prop:stategoalequiv_param}).

Known convergence results for TD extend to this setting:
tabular case with any sampling policy, linear
parameterization on-policy, arbitrary function approximation assuming
reversibility of the Markov process (Section~\ref{sec:conv}). 

Likewise, we formally derive the TD algorithm for goal-dependent $Q$ and
$V$ functions (Section~\ref{sec:goals}), with finite-variance estimators
even in continuous spaces. The goals may be target states, or target
values for some vector-valued function of the states
(Section~\ref{sec:goaldependentV}).

\item Algorithmically, successor states and goal-dependent values are
represented by function approximators depending on two states (the
current state and a goal state) instead of one. TD learning
works in a black-box environment by sampling from a set of observed
transitions $s \to s'$ between states, and sampling goal states
(typically from the same distribution). No reward signal is needed.

Most variants of TD still apply: $V$ or $Q$ learning, target networks,
multi-step returns...  (Appendix~\ref{sec:basictd}). Notably,
Appendix~\ref{sec:relative} describes \emph{relative TD} to deal directly
with a decay factor $\gamma=1$ and to reduce variance for $\gamma$ close
to $1$.

Successor states and goal-dependent values can be used to learn an
optimal policy
for a particular reward, or to learn goal-dependent policies.
Many different options are described
in Section~\ref{sec:V}, such as $Q$-learning or  policy gradient, with
several ways to learn the value function from successor states.

\item 
Successor states satisfy more than one Bellman equation:
we 
introduce \emph{backward TD} for
successor states (Section~\ref{sec:backw-bellman},
Theorem~\ref{thm:right-bellman}), and the corresponding
parametric update (Theorem~\ref{thm:paramtd_right},
Appendix~\ref{sec:backwardTD_param}).

Successor states encode all paths in the Markov process for a fixed
policy (Section~\ref{sec:paths}).
The usual (forward) Bellman equation $M=\Id+\gamma PM$ adds a
newly observed transition at the front of all known paths, while the backward
Bellman equation $M=\Id+\gamma MP$ extends known paths by adding newly observed
transitions at the back. This backward equation exists for successor
states but not goal-dependent value functions. In the tabular setting
and with small learning rates, combining forward and backward TD turns
out to improve the eigenvalues of the learning process
(Section~\ref{sec:mixingfbtd}).

\item We introduce ``second-order'' methods for learning successor
states, which are to TD what Newton-type methods are to first-order
gradient descent (Section~\ref{sec:BN}). 
In addition to the usual (forward) and the backward Bellman equations,
there is a third
Bellman equation satisfied by $M$, which leads to the
\emph{Bellman--Newton operator} $M\gets 2M-M^2+\gamma MPM$
(Section~\ref{sec:BNop}). It 
also enjoys a path interpretation,
learning by path concatenation and doubling the length of known
paths (Proposition~\ref{prop:doublingpaths}). The forward and backward
Bellman operators only increase the length of known paths by $1$
(Appendix~\ref{sec:bellman-newton}).

Asymptotically and in the small learning rate limit, the Bellman--Newton
operator converges provably faster than TD
(Section~\ref{sec:BNconttime}), with an asymptotic rate \emph{independent
of the environment and policy}. However, in practice this method is less resistant to sample
noise: smaller learning rates are necessary, so the comparison with TD
is less clear. There is also a parametric version of the Bellman--Newton
operator (Theorem~\ref{thm:paramBN}), but it is numerically fickle.

We also study the estimation of $M$ by direct inversion of $\Id-\gamma
\hat P$ using an empirical estimate $\hat P$ of the transition matrix $P$
in a finite state space. The resulting update of $M$ when adding each new
observation is the same as a Bellman--Newton update with learning rate
$1/t$ (Theorems~\ref{thm:onlineM} and~\ref{thm:onlineM_exp}). In finite
spaces, we provide
an explicit non-asymptotic bound for the convergence of $M$ and the value
function $V$ based on these empirical estimates
(Theorem~\ref{thm:convergence}).

\item Representing the successor state operator as a dot product
$\transp{F(s_1)}B(s_2)$ between features of the starting state $s_1$ and
the target state $s_2$ has many nice properties (Section~\ref{sec:FB}).
Here, the ``forward'' and ``backward'' feature functions $F$ and $B$ are
both learned to approximate $M$: this may have independent interest
for representation learning.

First, this method provides a direct representation of the value function without
additional learning (Eq.~\ref{eq:VFB}). 

Second, when learning 
$F$ and $B$ by any of the
algorithms above, in expectation the updates factorize between $s_1$ and
$s_2$ (Proposition~\ref{prop:FBupdates}). This allows for variance
reduction, and for purely trajectory-wise algorithms which only use the
currently observed transition $s\to s'$ without sampling an additional
target state $s_2$ (Section~\ref{sec:FBeq}), in contrast to the general
form of TD for $M$.

Third, this representation keeps some properties of the Bellman--Newton
method without its shortcomings; they actually coincide when the
transition matrix of the process is symmetric
(Theorem~\ref{thm:FBandBN}).

Finally, the fixed points of TD for this representation can be fully
characterized in the tabular and overparameterized cases
(Propositions~\ref{prop:fbFB}--\ref{prop:bfFB} in Appendix~\ref{app:FB},
and Section~\ref{sec:FBeq}). They are related to eigenspaces of the
transition matrix $P$. Notably, in the tabular or overparameterized case,
if forward TD is used to learn $F$ and backward TD to learn
$B$, then the fixed points are exactly local minimizers of the error
between the $\transp{F}B$ model and the true successor state operator
(Proposition~\ref{prop:fbFB}). In contrast, for ordinary TD on the value
function and a linear model, the fixed points are not minimizers of the
error to the true value function.

\end{itemize}

%

\paragraph{Some related work on successor states.}
\todo{move to relevant places?}
The successor state operator is linked to various existing objects under
various names (fundamental matrix, occupation matrix, successor
representations, successor features...).
Successor states have even been identified
in the neurosciences
\citep{stachenfeld2017hippocampus}.

For discount factor $\gamma=1$, the successor
matrix $M$ is known as the
\emph{fundamental matrix} \citep{kemenysnell1960,
bremaud1999markov,grinstead1997introduction} of a Markov process
(up to subtracting the invariant measure). 
\footnote{Namely, in Markov chain theory, the
fundamental matrix is defined with an additional rank-one term which
avoids all problems with $\gamma=1$ and is analogous to \emph{relative
TD}. 
The case $\gamma<1$ is
obtained from it
\citep[\S5.1.1]{bertsekasvol2_ed4}.
In this introduction, to stay closer
to RL practice, we take
$\gamma<1$ and define $M$ without this term.
The case $\gamma=1$
is treated
in Appendix~\ref{sec:relative}
(relative TD for $M$).}
The fundamental matrix encodes many properties of the Markov
chain, such as value functions (\citep{bertsekasvol2_ed4}, as we use here) or hitting times
\citep{kemenysnell1960}. 
In a reinforcement learning context, and with $\gamma<1$,
this matrix goes back at least to \citep{dayan1993improving}.

Learning successor states by temporal difference is mentioned in
\citep{dayan1993improving} for the tabular case and with linear
approximations; the parametric case has never been derived as far as we
know.

In a deep learning context,
several recent
works have used the related \emph{successor representations}
\citep{kulkarni2016deep},
e.g.,
for transfer \citep{NIPS2017_6994, borsa2018universal,8206049,
lehnert2017advantages,ma2018universal,barreto2020fast},
hierarchical RL \citep{c.2018eigenoption} or exploration
\citep{machado2019countbased}.

In particular, the Deep Successor Representation algorithm 
\citep{kulkarni2016deep} approximates successor states
by learning a state representation $\phi(s)$ together with
a successor representation $m(s)$ defined as the
expected discounted representation of future states from $s$: $m(s) =
\E\left[\sum_{t\geq0}\gamma^t \phi(s_{t})|s_{0}=s\right]$.
As $\phi = m = 0$ is
a fixed point of the method, a reconstruction loss must be used to prevent
collapse. Here we directly learn the successor states $M_{ss'}$ for every
pair of states in the
original space.

Successor states provide the value function for every goal state: this
is related to learning multiple RL tasks
\citep{sutton2011_horde,pmlr-v37-schaul15, Jin2020RewardFreeEF,
Pinto2017LearningTP} which performs joint $V$- or
$Q$-learning for a set of goals. \option{To some extent, this makes it possible
to reach or transfer to previously unseen goals \citep{pmlr-v37-schaul15}.}

Recently, \citep{van2020expected} proposed an algorithm to learn a model
of eligibility traces; we prove in Appendix~\ref{sec:traces} that the
expected eligibility traces at each state is proportional to the
transpose of the successor state matrix (``predecessor'' states).

Our second-order algorithms in Section~\ref{sec:BN} are based on an implicit process estimation approach. Process estimation
is also used in
\citep{Pananjady2019ValueFE} to obtain convergence bounds for the value
function in finite MDPs,
under a ``synchronous'' setting \option{(a transition is observed from every
state at every step)}. They
prove that process estimation is minimax-optimal for this setting.


More generally, successor state learning comes in the context of
\emph{unsupervised RL}, in which relevant features of the
environment are learned without the supervision of a reward signal. Many
works have suggested that unsupervised RL improves sample efficiency \citep{Sun2019ModelbasedRI}.
Notably, this
includes model-based methods \citep{franccois2018introduction}.
Contrary to the latter, successor state learning does not
require synthesizing accurate future states; \option{to some extent, a transition model
is implicitly learned via a function $m(s,s')$ that describes how much
$s'$ lies in the future of $s$ with the current policy.} 

\section{Notation for Markov Reward Processes}
\label{sec:notation}
\label{sec:defin-prop}


We consider a \emph{Markov reward process} (MRP) $\mathcal{M}=\langle
\mathcal{S}, P, r, \gamma\rangle$ with state space $\mathcal{S}$
(discrete or continuous), transition probabilities $P_{ss'}$ from $s$ to
$s'$, random reward signal $r_s$ at state $s$, and discount factor $0 \leq
\gamma < 1$ \citep{sutton2018reinforcement}.
We do not assume that the state
space $\mathcal{S}$ is finite.

\newcommand{\BS}{B(\mathcal{S})}

In the finite
case, $P_{ss'}$ can be viewed as a matrix.  In the general case, for each
$s\in \mathcal{S}$, $P(s,\d s')$ is a \emph{probability measure} on $s'$ that
depends on $s$. From now on, we use the notation $P(s,\d s')$ to cover
both cases. \footnote{
Formally, we take the setting from
\citep{Hairer_markovnotes1}. The
state space $\mathcal{S}$ is assumed to be a complete, separable metric
space (\emph{Polish space}),  such as a finite or countable space or
$\R^n$. It is equipped with its Borel $\sigma$-algebra (the $\sigma$-algebra
generated by all open sets). This guarantees that integration behaves as
expected.  $P(s,\d s')$ is assumed to be a \emph{Markov kernel}, namely, a
measurable map from $\mathcal{S}$ to probability measures over
$\mathcal{S}$.} 

A Markov \emph{decision} process, with a given policy,
with actions
$a\in \mathcal{A}$, transition probabilities $P(s,a,\d s')$, and policy
$\pi(s,a)$, defines two Markov reward processes:
one on states via $P(s,\d s')\deq \sum_a
\pi(s,a)P(s,a,\d s')$,
and another on state-action pairs via
$P((s,a),(\d s',a'))\deq P(s,a,\d s')\,\pi(s',a')$.
(start at $(s,a)$, get $s'$, then choose the next action at $s'$).
Thus, we work on states and value functions, but all results extend to
state-action pairs and $Q$ functions.

For now the policy is fixed: we deal with policy evaluation and
successor states under that policy. Goal-dependent policies are treated
in Section~\ref{sec:goals}.

We denote $R(s)\deq \E[r_s]$ the expected reward at $s$. The \emph{value
function} $V$ is $V(s_0) \deq
\E\left[\sum_{k}\gamma^{k}r_{s_k}\right]$ where $s_0,s_1,\ldots,s_t$ is a
trajectory starting at $s_0$ sampled from the process.
We denote by $\1_{s}$ the vector equal to $1$ at the $s$ coordinate and $0$
elsewhere.

\paragraph{Data model.}
We assume access to 
observations from the Markov reward process, such as a fixed dataset of
stored transitions, or some sample trajectories.
Each observation is a
triplet $(s,s',r_s)$ with $s'\sim P(s,\d s')$ and $r_s$ the associated
reward. Consecutive observations need not be independent.
We denote by $\rho(\d s)$ be the distribution
of states $s$ coming from the observations. We cannot choose the states $s$: $\rho$ is unknown and we do not make
any assumptions on it.
For instance, if we have access to trajectories 
from the process, obtained by some exploration policy, then $\rho$ would
be the law of states visited under that policy. If we just have a finite
dataset of transitions, $\rho$ would be the (unknown) law from which this
dataset was sampled.


\paragraph{Markov kernels as operators.} Interpreting $P$ and the
successor state as operators on functions over $\mathcal{S}$
clarifies the statements of the results below.
We follow the standard theory of Markov kernels
\citep{Hairer_markovnotes1, Hairer_markovnotes2}.
We denote by $\BS$ the set of bounded measurable functions on
$S$. $P$ acts on such functions as follows.
If
$f$ is a function in $\BS$, $Pf$ is
defined as $(Pf)(s) \deq \E_{s'\sim P(s, \d s')}\left[f(s')\right]$.  This
is compatible with the matrix notation $Pf$ in the finite case, viewing
$f$ as a vector. In the text, we freely identify Markov kernels with the
corresponding operators.

If $P_{1}$ and $P_{2}$ are two such Markov kernel
operators, their composition $P_{1}P_{2}$ is again a Markov kernel operator,
and coincides with matrix multiplication in the finite case. In
particular, $P^{n}$ represents $n$ steps of $P$. The identity
operator $\Id$ corresponds to always staying in the same state, namely, a
transition operator $P(s,\d s')=\delta_s(\d s')$ with $\delta_s$ the
Dirac measure at $s$.

We denote $\Delta\deq \Id-\gamma P$, the discrete Laplace operator of the
Markov process.
Finally,
if $A$ is an operator acting on functions over $\mathcal{S}$, we denote
its inverse by $A^{-1}$, if it exists. 

\paragraph{Norms.} Both $P(s,\d s')$ and the successor state operator $M(s,\d s')$ are
measures on $s'$ that depend on $s$. We will use the following norms on
such objects: if $\rho(\d s)$ is some reference probability measure on
$\mathcal{S}$, and $M_1(s,\d s')$ and $M_2(s,\d s')$ are two such objects, we define
\begin{equation}
\label{eq:norm}
\norm{M_1-M_2}^2_\rho \deq \E_{s\sim \rho,\,s'\sim \rho}\,
(m_1(s,s')-m_2(s,s'))^2
\end{equation}
where $m_1(s,s')\deq M_1(s,\d s')/\rho(\d s')$ is the density of $M_1$
with respect to $\rho$ (if it exists; if not, the norm is infinite), and likewise for $M_2$. We will also use
the \emph{total variation} norm
\begin{equation}
\norm{M_1-M_2}_{\rho,\mathrm{TV}}\deq \E_{s\sim \rho}
\norm{M_1(s,\cdot)-M_2(s,\cdot)}_{\mathrm{TV}}
\end{equation}
with $\norm{p_1-p_2}_{\mathrm{TV}}\deq \sup_{A\subset \mathcal{S}} \abs{p_1(A)-p_2(A)}$
the usual total variation distance between two measures.

\section{The Successor State Operator of a Markov Process}
\label{sec:defM}

As an introduction before defining successor states over general state
spaces, we start with the case of finite state spaces, for which all the
objects can be seen as vectors matrices. This is the case treated in
\citep{dayan1993improving}.

\subsection{The Successor State Matrix in a Finite State Space}
\label{sec:defM_discrete}

Informally, for finite state spaces, given two states $s_1$ and $s_2$ in a Markov process, the
successor state matrix $M$ is a matrix whose entry $M_{s_1s_2}$ is the
expected discounted time spent at $s_2$ if starting the process at $s_1$
\citep{dayan1993improving}.

$M_{s_1s_2}$ is also the value function at $s_1$ if the reward is located
at $s_2$ ($R_s=\1_{s=s_2}$). Thus, columns of $M$ contain the value functions
of all single-target rewards. For a fixed Markov process (e.g., fixed
environment and policy), the value function is a linear function of the
reward. Thus,
by linearity, for any reward, the
associated value function is $V=MR$. 
Namely, $M$ contains information about the value function of \emph{every}
reward.

We gather several equivalent definitions of the matrix $M$ in the
following Proposition. Since this is a particular case of the more general
results below, we do not include a proof.

\begin{prop}[ (Successor state matrix of a finite Markov process)]
Consider a Markov process on a finite state space, with transition matrix
$P$.
The following definitions of the \emph{successor state matrix} $M$ are
equivalent:
\begin{enumerate}
\item
$M$ is the inverse of the Laplace operator $\Delta=\Id-\gamma P$,
\begin{equation}
M=(\Id-\gamma P)^{-1}=\sum_{n\geq 0} \gamma^n P^n.
\end{equation}

\item $M$ is the matrix that transforms a reward function into the
corresponding value function: for any reward function $R$, the associated
value function is
\begin{equation}
V=MR.
\end{equation}

\item For each state $s$, the column $s$ of the matrix $M$ represents the
value function of a Markov reward process whose reward is $1$ when at state $s$ and $0$
everywhere else ($R=\1_s$).

\item $M$ is the unique fixed point of the Bellman operator
\begin{equation}
M\gets \Id -\gamma P M
\end{equation}
or equivalently of the \emph{backward} Bellman operator
\begin{equation}
M\gets \Id - \gamma MP.
\end{equation}

\item For each state $s$, the row $s$ of the matrix $M$ represents the
expected occupation time at each state, for trajectories starting at $s$,
with discounting $\gamma$:
\begin{equation}
M_{ss'}=\sum_{t\geq 0} \gamma^t \Prob(s_t=s'|s_0=s)
\end{equation}
where $s_0,\ldots,s_t$ is a random trajectory in the Markov process.

\item The entry $ss'$ of the matrix $M$, is the number of paths from $s$
to $s'$, weighted by their probability in the process, and with decay $\gamma$
according to their length:
\begin{equation}
M_{ss'}=\sum_n \sum_{\begin{subarray}{c}p=(s_0s_1...s_n)\\\text{ path
from}\\\text{$s_0=s$ to $s_n=s'$}\end{subarray}}
\!\! \gamma^n \prod_{i=1}^n P(s_i|s_{i-1}).
\end{equation}

\end{enumerate}
\end{prop}

\subsection{The Successor State Operator in a General State Space}
\label{sec:defM_cont}

$M$ is also well-defined in general state spaces, using the Markov
process formalism of Section~\ref{sec:notation}, as follows. This
extends \citep{dayan1993improving} to arbitrary $\mathcal{S}$.
(All proofs are given in the Appendix.)

\begin{thm}
  \label{thm:Mdef}
  The \emph{successor state operator} $M$ of a Markov reward process is defined as
  \begin{equation}
    \label{eq:M-def}
    M\deq \sum_{n\geq 0} \gamma^n P^n,\qquad M(s_1,\d s_2)=\sum_{n\geq 0}
    \gamma^n P^n(s_1,\d s_2).
  \end{equation}
  where $P^0\deq \Id$. Thus,
  for each $s_{1}$, $M(s_{1}, \d s_2)$ is a measure on
  $s_2$, with total mass $\frac{1}{1-\gamma}$.
  
  Then $M$ is a well-defined
  operator over the set $\BS$ of bounded measurable functions on
  $\mathcal{S}$. Moreover,
  \begin{equation}
    M = (\Id - \gamma P)^{-1}
    \end{equation}
    as operators over $\BS$, and
    \begin{equation}
    \quad \text{and} \quad  V = MR
  \end{equation}
  for any reward function $R$. (Note that $M$ does not depend on $R$.)
\end{thm}

$M$ can be interpreted as
\emph{paths} in the Markov process:
$M(s_1,\d s_2)$ represents the number of paths from $s_1$ to $s_2$,
weighted by their probability and discounted by their length.
This
will be relevant to compare the algorithms below.
Indeed, in the finite-state case and using matrix notation, $P^{n}_{ss'}$ is the
probability to go from $s$ to $s'$ in $n$ steps; therefore
\begin{align}
  \label{eq:m_path}
  M_{ss'} &= \sum_{n\geq 0} \gamma^n (P^n)_{ss'}=\sum_{n \geq 0} \gamma^{n}\sum_{s=s_{0}, s_{1}, ...,
  s_{n-1}, s_{n}=s'} \hspace{-7ex} P_{s_{0}s_{1}}\cdots P_{s_{n-1}s_{n}}
  \\&= \hspace{-2ex}\sum_{p
  \text{ path from } s \text{ to } s'} \hspace{-4ex} \gamma^{|p|}\,\mathbb{P}(p)
\end{align}
where, if $p = (s_{0}, ..., s_{n})$ is a path,  $\mathbb{P}(p) =
P_{s_{0}s_{1}}\cdots P_{s_{n-1}s_{n}}$ is its probability and $|p| = n$
its length. The same holds with integrals instead of sums in continuous
spaces.

Yet another interpretation of $M$ is via expected eligibility traces: indeed,
when visiting a state $s$, the expectation of the eligibility trace vector
$(e_{s'})$ 
is directly related to $M_{s's}$. The details are given in
Appendix~\ref{sec:traces}; see also the discussion of ``predecessor
features'' in \citep{van2020expected}.

\paragraph{Successor states and successor representations.} Given a function
$\phi$ over the state space $\mathcal{S}$, the expectation of the cumulated,
discounted future values of $\phi$ given the starting point $s_0$ of a
trajectory $(s_t)$ is
\begin{equation}
\E \left[
\sum_{t\geq 0} \gamma^t \phi(s_t)\right]=\sum_{t\geq 0} \gamma^t (P^t
\phi)(s_0)=(M\phi)(s_0).
\end{equation}
Thus, the successor representation
(e.g., in the sense of
\citep{kulkarni2016deep}) of a state $s$ is obtained by applying $M$ to
some user-chosen function $\phi$.

\paragraph{Representing and learning the successor state operator.}
\label{sec:M-func-approx}
With continuous states, $M$ cannot be represented as a matrix. Instead,
we will learn a function of a pair of states. Namely,
we will learn a parametric model of $M$ via its density with
respect to the data distribution $\rho$ over states (this choice makes
every algorithm samplable from the data).
We present two
versions of this. The first version represents $M$ as
\begin{equation}
\label{eq:Mmodeltilde}
M(s_1,\d s_2)\approx \tildem_\theta(s_1,s_2) \rho(\d s_2)
\end{equation}
and the second version as
\begin{equation}
\label{eq:Mmodel}
M(s_1,\d s_2)\approx \delta_{s_1}(\d s_2)+m_\theta(s_1,s_2) \rho(\d
s_2)
\end{equation}
where $\delta_{s_1}$ is the Dirac measure at $s_1$, and where
$\tildem_\theta$ and $m_\theta$ are functions over pairs of states,
depending smoothly on some parameter $\theta$.
We will derive well-principled algorithms to learn the functions $\tildem(s_1,s_2)$ and $m(s_1,s_2)$
from observations of the Markov process. The data distribution $\rho$ is unknown, but all
algorithms below only require the ability to sample states from $\rho$,
which we can do by definition since $\rho$ is the distribution of states
in the dataset.
These two models correspond, respectively, to
\begin{equation}
\label{eq:Vmtilde}
V(s_1)=\E_{s_2\sim \rho} [\tildem_\theta(s_1,s_2)R(s_2)]
\end{equation}
and
\begin{equation}
V(s_1)=R(s_1)+\E_{s_2\sim \rho} [m_\theta(s_1,s_2)R(s_2)].
\end{equation}

The first version is simpler. The motivation for the second version is as
follows.
In continous spaces, $M$ has a singular part,
corresponding to the immediate reward in $V$, and to the term $\Id$ in
the series for $M$: for each $s_1$, the measure $M(s_1,\cdot)$ comprises
a Dirac mass at $s_1$. In continuous spaces, this singular part cannot be represented
as $\tildem(s_1,s_2)\rho(\d s_2)$ for a smooth function $\tildem$. 
But since this singular part $\delta_{s_1}$ is known, we can just
parameterize and learn the absolutely continuous part $m(s_1,s_2)$. Thus,
the second version may represent $M$ exactly (at least if $P$ is smooth), while in general the first
version cannot. Still, the first version may provide useful
approximations. \todo{Third possibility: just define the reward starting
at the next step, $\delta_{s_2}(\d s')$, namely, $M=P(\Id-\gamma
P)^{-1}$. Then the equations are those for $m_\theta$ without a $\gamma$,
and there is no $R(s_1)$ in the representation of $V$ from $m_\theta$.
Might be useful for $Q$-learning with rewards that depend on $s'$.}

The function $m_\theta(s_1,s_2)$ can be interpreted as a (directed) similarity measure between
$s_1$ and $s_2$, coming from the structure of the Markov process.

In this text,
we define several algorithms for learning $m_\theta$:
the extension of temporal difference (TD) to successor states
(Section~\ref{sec:forw-bellman}); backward
TD for successor states (Section~\ref{sec:backw-bellman}); and second-order-type methods
(Section~\ref{sec:BN}). The matrix-factorized forward-backward
parameterization
$\tildem_\theta(s_1,s_2)=\transp{F_\theta(s_1)}B_\theta(s_2)$ has many additional properties and is
treated in Section~\ref{sec:FB}.

\bigskip

A learned model of $M$ can be used in several ways:
\begin{itemize}
\item $M$ may be used to improve learning for a given reward. For
instance, with a sparse reward located at a \emph{known} target state
$s_\tar$, then
$V(s)=M(s,\d s_\tar)$. In that case, learning $M$ directly provides the value
function, while ordinary TD would not work because of the sparse reward.
With dense rewards, $M$ can be used in the learning of the value
function (Section~\ref{sec:V}).
\item Objects similar to $M$ may be used to learn goal-dependent policies, such as learning
how to reach any arbitrary state. This does not cover dense rewards, but
extends to reaching states with arbitrary values for some features. This
is covered in Section~\ref{sec:goals}.
\end{itemize}

Section~\ref{sec:V} gives more details about the ways to use $M$ to learn
value functions and policies.

\section{TD Algorithms for Deep Successor State Learning}
\label{sec:td-alg}

\subsection{The (Forward) TD Algorithm for Successor States}
\label{sec:forw-bellman}

\subsubsection{The Forward Bellman Equation}

\begin{thm}[ (Bellman equation for successor states)]
\label{thm:left-bellman}
The successor state operator $M$ is the only operator which satisfies the
Bellman equation $M=\Id+\gamma PM$.
\end{thm}

This Bellman equation makes sense, as operators, on any state space,
discrete or continuous.
In finite spaces, each column of the matrix $M$ contains the value
function for a reward located at a specific target state, and the Bellman
equation for $M$ is just the collection of the standard Bellman equations
for every target state; the $\Id$ term is the reward for
reaching state $s$ when the target is $s$.

This Bellman operator on $M$ has the same contractivity properties as the
usual Bellman operator.

\begin{prop}[ (Contractivity of the Bellman operator on $M$)]
\label{prop:contract}
Equip the space of functions $\BS$ with the sup norm $\norm{f}_\infty\deq
\sup_{s\in S}\abs{f(s)}$.
Equip the space of bounded linear operators from $\BS$ to $\BS$
with the operator norm $\norm{M}_{\mathrm{op}}\deq \sup_{f\in \BS,\,f\neq
0} \norm{Mf}_\infty/\norm{f}_\infty$.

Then the Bellman operator $M\mapsto
\Id+\gamma PM$ is $\gamma$-contracting for this norm.

Consequently, for any learning rate $\eta\leq 1$, iterated application of the
Bellman operator $M\gets (1-\eta)M+\eta(\Id+\gamma PM)$ converges to the
successor state operator.
\end{prop}

\subsubsection{Forward TD for Successor States: Tabular Case}

Given that the Bellman equation on $M$ is a collection of ordinary Bellman
equations for every target state, an obvious algorithm to learn $M$ in finite state spaces is to
perform ordinary TD in parallel for all these single-state rewards, as in
\citep{dayan1993improving}. Let
$\starget$ be some target state and consider the reward $\1_{\starget}$.
Upon observing a transition $s\to s'$,
ordinary TD for this reward updates $V$ by $V_s\gets V_s+\eta \del
V_s$, where $\eta$ is some learning rate and
$
\del V_s=\1_{s=\starget}+\gamma V(s')-V(s)
$.
Performing
TD in parallel for every column of $M$ with target state $\starget$ is equivalent to the
following \citep{dayan1993improving}.

\begin{defi}[ (Tabular temporal difference for successor
states)]
\label{def:tabulartd}
The TD algorithm for $M$, in a finite state space, maintains $M$ as a
matrix. Upon observing a transition $s\to s'$ in the Markov process, $M$
is updated by $M\gets M+\eta \del M$ where $\eta$ is a learning rate
and $\del M$ has entries
\begin{equation}
\label{eq:tabularTD}
\del M_{ss_2}\deq \1_{s=s_2}+\gamma M_{s's_2}-M_{ss_2} \qquad \forall s_2
\end{equation}
\end{defi}

In the tabular case and with deterministic rewards,
learning $M$ via TD, then estimating $V$ via the matrix
product $V=MR$, is equivalent to directly learning $V$ via tabular TD
(Appendix~\ref{sec:tdistd}): tabular TD on $M$ treats all
target states $s_2$ as independent learning problems, and no learning
gain is achieved.

However, this equivalence does not hold with function
approximation, which introduces generalization between states.
Since any target state is
reached with zero probability, applying parametric TD naively in parallel
for every target state would always provide reward $0$ in continuous
environments. 
The parametric TD updates we present below
are not equivalent to this naive TD: they have the same expectation but
avoid the zero-reward problem.



%
%
%

\subsubsection{Forward TD for Successor States: Function
Approximation}
\label{sec:paramtd}

In continuous environments, it is not possible to store $M$ as a matrix.
But we
can maintain a model $m_\theta$ of
the density of $M$, as explained in Section~\ref{sec:M-func-approx}.
As in usual parametric TD, 
we learn $\theta$ by defining an ``ideal'' update given by the Bellman
equation, 
and update $\theta$ so that $M$ gets closer to it.

\begin{thm}[ (TD for successor states with
function approximation)]
\label{thm:paramtd}
Maintain a parametric model of $M$
as in
Eq.~\ref{eq:Mmodel}
via
$M_{\theta_t}(s_1,\d s_2)=\delta_{s_1}(\d s_2)+m_{\theta_t}(s_1,s_2)\rho(\d s_2)$, with $\theta_t$ the value
of the parameter at step $t$, and with $m_\theta$ some smooth family of
functions over pairs of states.

Define a target update of $M$ via the Bellman equation, $M^\tar\deq
\Id+\gamma P M_{\theta_t}$.
Define the loss between $M$ and $M^\tar$ via $J(\theta)\deq \frac12
\norm{M_\theta-M^\tar}_\rho^2$ using the norm \eqref{eq:norm}.
Then the
gradient step on $\theta$ to reduce this loss is
\begin{multline}
\label{eq:paramtd}
-\partial_\theta J(\theta)_{|\theta=\theta_t}=
\E_{ s\sim \rho, \,s'\sim P(s,\d s'), \,s_2\sim\rho }
\left[
\gamma \,\partial_\theta m_{\theta_t}(s,s')
\right.
\\+\left.
\partial_\theta m_{\theta_t}(s,s_2)\,
(\gamma m_{\theta_t}(s',s_2)-m_{\theta_t}(s,s_2))
\right].
\end{multline}

For the model variant in Eq.~\ref{eq:Mmodeltilde}, $M_{\theta_t}(s_1,\d
s_2)=\tildem_{\theta_t}(s_1,s_2)\rho(\d s_2)$, the gradient step on
$\theta$ to reduce the loss $J(\theta)$ is
\begin{multline}
\label{eq:paramtdtilde}
-\partial_\theta J(\theta)_{|\theta=\theta_t}=
\E_{ s\sim \rho, \,s'\sim P(s,\d s'), \,s_2\sim\rho }
\left[\,
\partial_\theta \tildem_{\theta_t}(s,s)
\right.
\\+\left.
\partial_\theta \tildem_{\theta_t}(s,s_2)\,
(\gamma \tildem_{\theta_t}(s',s_2)-\tildem_{\theta_t}(s,s_2))
\right].
\end{multline}
\end{thm}


This gradient step is ``samplable''. Namely, we can define a stochastic update
$\widehat{\del\theta_{\text{TD}}}$ with expectation
\eqref{eq:paramtd}: sample a transition $s\to s'$
from the dataset of transitions, 
and \emph{another} independent ``destination'' state $s_{2}$ from the
dataset, then set
\begin{equation}
    \widehat{\del\theta_{\text{TD}}} =  \gamma\,\partial_{\theta} m_{\theta}(s,s')
    +\partial_{\theta} m_{\theta}(s,s_2) \left(\gamma m_\theta(s',s_2)-
    m_{\theta}(s, s_{2}) \right)
\end{equation}
or likewise for $\tildem$ (only the first term is different).

This algorithm uses a transition $s\to s'$ and one additional random
state $s_2$, independent from $s$ and $s'$. The Bellman--Newton
update (Section~\ref{sec:paramBN}) will use two additional
random states $s_1$ and $s_2$ (but no additional transition).
 The law of
$s_2$ is $\rho$, which means $s_2$ is just another state sampled from the
dataset. For instance, if the dataset consists of a sampled trajectory
trajectory $(s_t)_{t\geq 0}$, when 
observing a transition $s_t\to s_{t+1}$, additional
independent state samples can be
obtained by using states $s_{t'}$ at times $t'$ independent from $t$
(such as a random $t'\leq t$). This requires maintaining a replay buffer
of observed states.

Several variants avoid having to sample $s_2$ independently from
$s\to s'$. In the FB representation of $M$
(Section~\ref{sec:FB}), the expectation over $s_2$ can be estimated
online using just the observed transition $s\to s'$, with no additional
state. Appendix~\ref{sec:sss} also describes the possibility of
using a ``cheap''  source for the additional states $s_2$ instead of
actual states, as long as
the transitions $s\to s'$ come from the true process. Finally,
Theorem~\ref{thm:goalTD} makes it possible to use a joint rather than
independent distribution for $s$ and $s_2$ (such as choosing a target
state $s_2$ and following an $s_2$-dependent policy for some time).

\subsubsection{Infinitely Sparse Rewards and Forward TD vs TD on State-Goal Pairs}
\label{sec:sparse}

Why don't the Dirac rewards show up in the parametric TD algorithm of
Theorem~\ref{thm:paramtd}?
Why don't
the rewards become infinitely sparse with continuous states?

The
tabular TD algorithm \eqref{eq:tabularTD} for $M$ features a sparse
reward $\1_{s=s_2}$. Why don't these sparse rewards vanish completely in
the continuous state limit, where an equality of states never occurs?
This is simply because we know exactly when these terms make a
contribution: namely, we know we can just take $s_2=s$. In the continuous
case, with a model $\tildem(s,s_2)$, the sparse reward is a Dirac
$\delta_s(\d s_2)$, and it shows up in TD as a term $\partial_\theta \tildem (s,s_2)\delta_s(\d s_2)$. When integrated over $s_2$, this term is
just $\partial_\theta \tildem(s,s)$. Thus the contribution from the
infinitely sparse Dirac term is actually finite and nonzero.

Intuitively, we are solving RL problems with an infinity of infinitely
sparse target states $s_2$. But at every time step, when we visit
state $s$, we know that we just visited the target state $s_2=s$: every
step brings a reward. This
knowledge is exploited in the expressions we give for TD, resulting in a
finite contribution $\gamma \,\partial_\theta
m_{\theta_t}(s,s')$ in \eqref{eq:paramtd}.

Algorithmically, it is
quite important to use this. In algorithms that sample a target state $s_2$
fully independently from $s$ (such as picking a random goal $g$ in
\citep{pmlr-v37-schaul15}), the contribution from the reward
$\1_{s=s_2}$ is sometimes nonzero in the tabular case, but gets
infinitely sparse and eventually vanishes in the continuous case. We
provide more details in Section~\ref{sec:sparse} (see also
Section~\ref{sec:goals} for a discussion of state-goal resampling
strategies such as hindsight experience replay \citep{HER}).

On the other hand, 
successor states learned via Theorem~\ref{thm:paramtd} can in principle
learn an
infinite number of infinitely sparse rewards, with every transition being
informative.

\paragraph{The state-goal process.}
In expectation, one can view
forward TD for $M$ as
ordinary TD on the space of pairs $(s,s_2)$, as follows. For the tabular
case this holds without expectations, but for the
parametric case, this equivalence
holds only in expectation: ordinary parametric TD
on pairs $(s,s_2)$ would have infinite variance on continuous spaces due
to the
Dirac reward $\delta_{s}(\d s_2)$, but the
successor state update in Theorem~\ref{thm:paramtd} avoids this infinite
variance, as
discussed above.

In the tabular case, the equivalence is a direct consequence of the
Definition~\ref{def:tabulartd} for tabular forward TD on $M$.

\begin{prop}[ (Tabular forward TD on $M$ as ordinary TD on state-goal pairs)]
\label{prop:stategoalequiv_tabular}

Let $P$ be the transition matrix of the Markov process on state space
$S$. We call \emph{state-goal Markov process} the Markov process on
$S\times S$ whose transition matrix is $P\otimes \Id$, namely $(s,s_2)$
goes to $(s',s_2)$ with $s'\sim P(\d s'|s)$.

Let $S$ be discrete. Then tabular TD for successor states on
$S$ (Definition~\ref{def:tabulartd}) is equivalent to ordinary tabular TD on the
value function of the state-goal process for the reward function
$R(s,s_2)=\1_{s=s_2}$.
\end{prop}

The parametric case is handled as follows. In discrete or continuous
state spaces, the successor state operator $M(s,\d s_2)$ satisfies the
Bellman equation $M(s,\d s_2)=\delta_s(\d s_2)+\gamma \E_{s'\sim P(\d
s'|s)} M(s',\d s_2)$ as measures over $s_2$. Consider the
parameterization \eqref{eq:Mmodeltilde}, $M(s,\d s_2)=\tildem_\theta(s,s_2)\rho(\d
s_2)$ where $m_\theta$ is some parametric function (the parameterization
\eqref{eq:Mmodel} with $m_\theta$ is similar). The Bellman equation
rewrites as $\tildem_\theta(s,s_2)\rho(\d s_2)=\delta_s(\d s_2)+\gamma
\E_{s'\sim P(\d
s'|s)} \tildem_\theta(s,s_2)\rho(\d s_2)$. If $S$ is discrete, the ratio
of measures $\delta_s(\d
s_2)/\rho(\d s_2)$ is an ordinary function and we can rewrite the
successor state Bellman equation as
\begin{equation}
\label{eq:stategoalTD}
\tildem_\theta(s,s_2)=\frac{\delta_s(\d s_2)}{\rho(\d s_2)}+\gamma
\E_{s'\sim P(\d
s'|s)} \tildem_\theta(s,s_2).
\end{equation}

This is the Bellman equation over state-goal pairs $(s,s_2)$ for the
reward function $R(s,s_2)\deq \delta_s(\d s_2)/\rho(\d s_2)$ and
transition matrix $P\otimes \Id$. It is similar to goal-dependent value
functions (as in, e.g., \citep{pmlr-v37-schaul15}), up to the
$1/\rho(\d s_2)$ factor necessary to turn measures into functions.  Parametric
TD using this equation is just the average of parametric TD for the
individual value functions associated to each goal $s_2$.

Naive TD on this state-goal Bellman equation does not behave well due to
the sparse reward $\delta_s(\d s_2)$: most pairs have reward $0$ and this
induces high variance. In continuous spaces, TD on this equation
degenerates: the reward is $0$ with probability $1$ but its variance is
infinite due to the infinite Dirac function $\delta_s(\d s_2)/\rho(\d
s_2)$. However, the \emph{expected} TD update can be computed
algebraically and results in the finite-variance update for successor
states. Thus we have the following result.

\begin{prop}[ (Parametric TD on $M$ as finite-variance version of
parametric TD on goal-state pairs)]
\label{prop:stategoalequiv_param}

Let the state space $S$ be discrete. Then the Bellman equation
$M=\Id+\gamma PM$ for successor states is equivalent to the ordinary
Bellman equation \eqref{eq:stategoalTD} for the state-goal process on
pairs $(s,s_2)$ with reward function $R(s,s_2)\deq \delta_s(\d
s_2)/\rho(\d s_2)$.

Moreover, in expectation over state-goal samples $(s,s_2)\sim \rho\otimes \rho$, the
ordinary parametric TD update for the Bellman equation
\eqref{eq:stategoalTD} of the state-goal process is equal to the
parametric TD update for successor states from Theorem~\ref{thm:paramtd},
both for the parameterizations $m_\theta$ and $\tildem_\theta$.

Let the state space $S$ be continuous, with $\rho$ covering the whole
space.
Then
ordinary parametric TD for the Bellman equation \eqref{eq:stategoalTD} on
the state-goal process is undefined: the reward term is $0$ with probability
$1$ but has infinite variance. On the other hand, its expectation
is
well-defined, and the
parametric TD update for successor states from Theorem~\ref{thm:paramtd}
has the same expectation but finite variance (for smooth and bounded
$m_\theta$).
\end{prop}

\subsubsection{Convergence properties for TD on successor states}
\label{sec:conv}

Forward TD for $M$ converges in the same conditions as ordinary TD for
the value function. This is obtained by viewing forward TD for $M$ as
ordinary TD on the space of pairs $(s,s_2)$, as in
Section~\ref{sec:sparse}.
Thus, interpreting
TD for successor states as TD on the state-goal process immediately
transfers existing convergence results for ordinary TD to successor
states.

We consider three such results: convergence of tabular TD, convergence of
TD on-policy with a linear parameterization, and convergence of TD
on-policy for any parameterization if the random walk is reversible. In
each case, we refer to the original works for additional technical
conditions (learning rates, smoothness...)

\begin{itemize}
\item In the tabular case, forward TD on $M$
(Definition~\ref{def:tabulartd}) converges, with pairs $(s,s_2)$ sampled
at each step from essentially any selection scheme (stochastic or
deterministic) that ensures every pair is selected infinitely often, and
with suitable learning rates \citep{tsitsiklis1994tdconv}.

\item TD with linear parameterization on discrete spaces is known to converge
\emph{on-policy} \citep{tsitsiklis1997approxtd}, namely, with states sampled according to a
steady-state distribution of the Markov process (assumed to be nonzero on
every state). For successor states this translates to the following.
Assume
$S$ is discrete and the successor state operator is parameterized as
\begin{equation}
M(s,\d s_2)\approx \sum_i \theta_i\, \phi_i(s,s_2)\rho(\d s_2)
\end{equation}
or equivalently $\tildem_\theta(s,s_2)=\sum_i \theta_i \, \phi_i(s,s_2)$,
where $\theta=(\theta_1,\ldots,\theta_k)$ is the parameter to be learned,
and $\phi_1,\ldots,\phi_k$ are fixed functions. Assume
$\rho$ is a positive steady-state distribution of the Markov operator $P$, and let
$\rho_2$ be any positive distribution over $S$. Then $\rho\otimes \rho_2$
is a steady-state distribution of the Markov operator $P\otimes \Id$ over
state-goals, and parametric TD for the Bellman equation
\eqref{eq:stategoalTD} with pairs $(s,s_2)$ sampled from $\rho\otimes
\rho_2$, is convergent for suitable learning rates. This also covers the
parametric update in Theorem~\ref{thm:paramtd}, which has the same
expectation by Proposition~\ref{prop:stategoalequiv_param}.

\item For TD with arbitrary parametric families, convergence is known
assuming that the Markov operator $P$ is \emph{reversible}, namely, that
$\rho$ is 
its steady-state distribution and further satisfies the \emph{detailed
balance} condition $\rho(\d s)P(\d s'|s)=\rho(\d s')P(\d s|s')$, in other
words,
steady-state flows from state $s$ to $s'$ and $s'$ to $s$ are equal.
Then,
parametric TD is a stochastic gradient descent of a global loss between the
approximate and true value function \citep{ollivier2018approximate}. This
result extends to MDPs which are ``reversible enough''
\citep{brandfonbrener2019geometric}. Applying the result of
\citep{ollivier2018approximate} to successor
states via the state-goal process yields the following. Assume that the
space $S$ is finite and that the Markov operator $P$ is reversible. Let
$\tildem_\theta$ be any smooth parametric model for successor states. Let
$\tildem^\ast$ be the true value, namely, let the true successor state
operator be $M(s,s_2)=\tildem^\ast(s,s_2)\rho(\d s_2)$. Define the loss
function
\begin{equation}
\label{eq:revloss}
\ell(\theta)\deq
(1-\gamma)\norm{\tildem_\theta-\tildem^\ast}^2_{\rho\otimes \rho}+
\gamma \norm{\tildem_\theta-\tildem^\ast}^2_{\mathrm{Dir}}
\end{equation}
where $\norm{f}^2_{\rho\otimes \rho}\deq \E_{s\sim \rho,\,s_2\sim
\rho}f(s,s_2)^2$ and the \emph{Dirichlet norm} is 
\begin{equation}
\norm{f}^2_{\mathrm{Dir}}\deq \frac12 \E_{s\sim
\rho,\,s_2\sim
\rho,\,s'\sim P(\d s'|s)} (f(s',s_2)-f(s,s_2))^2.
\end{equation}
Then the parametric TD step for $M$ (Theorem~\ref{thm:paramtd}) is equal
to the gradient of this loss, $-\frac12 \partial_\theta \ell(\theta)$.
(This is a global loss between the parametric model and the true
value $\tildem^\ast$, contrary to the loss in Theorem~\ref{thm:paramtd} which uses a
loss with respect to the right-hand-side of the Bellman equation, which
depends on the current estimate.)

Thus, in the reversible case with $\rho$ the stationary distribution, parametric TD for $M$ converges to a local
minimum of the global loss \eqref{eq:revloss}, under the general conditions for convergence
of stochastic gradient descent.
\end{itemize}

\subsubsection{Variants of Forward TD: Target Networks, Multi-Step
Returns, $\gamma=1$, Using Features as Targets...}

The variants of TD used in practice also exist for successor states.

In Appendix~\ref{sec:basictd} we provide the parametric updates for two
variants: using a \emph{target network} (namely, performing several gradient
steps toward $\Id+\gamma P M^\tar$ without updating $M^\tar$), and using
\emph{multi-step returns}.

Appendix~\ref{sec:relative} describes \emph{relative TD} for
successor states: this makes it possible to deal directly with $\gamma=1$
and to reduce variance for $\gamma$ close to $1$.

Appendix~\ref{sec:sss} deals with using different probability
distributions for $s$ and $s_2$ (e.g., using synthetic states for $s_2$
to have more samples), and using a different reference measure for the
parameterization of $M$ (e.g., representing $M$ by its density with
respect to the uniform measure rather than the unknown distribution
$\rho$ in \eqref{eq:Mmodel} and \eqref{eq:Mmodeltilde}).

Appendix~\ref{sec:targetfeatures} mentions situations where the reward is
known to depend only on some features $\phi(s)$ of the state $s$ (such as
a subset of coordinates of $s$). A typical example would be a specific
target value for $\phi(s)$. In that case, it is enough to learn the
successor state operator $M(s,\d g)$ with the second argument in the
space of features, $g=\phi(s)$. Then $M(s,\d g)$ directly provides the
value function of the problem with a reward when $\phi(s)$ is equal to
$g$. It can be
used to express the value function of any reward that depends only on
$g$. The forward TD updates for $M(s,\d g)$ are similar to the case of
$M(s,\d s_2)$.

\subsection{Backward TD for Successor States}
\label{sec:backw-bellman}

\begin{thm}
\label{thm:right-bellman}
  The successor state operator $M$ is the only operator which
  satisfies the backward Bellman equation, $M=\Id+\gamma M P$.
\end{thm}

This equation has no analogue on $V$. It is similar to an update of
expected eligibility traces (see Appendix \ref{sec:traces}). The
resulting operator has the same contractivity properties as the usual
(forward) Bellman operator.

\begin{prop}[ (Contractivity of the backward Bellman operator on $M$)]
\label{prop:contract2}
Equip the space of functions $\BS$ with the sup norm $\norm{f}_\infty\deq
\sup_{s\in S}\abs{f(s)}$.
Equip the space of bounded linear operators from $\BS$ to $\BS$
with the operator norm $\norm{M}_{\mathrm{op}}\deq \sup_{f\in \BS,\,f\neq
0} \norm{Mf}_\infty/\norm{f}_\infty$.

Then the backward Bellman operator $M\mapsto
\Id+\gamma MP$ is $\gamma$-contracting for this norm.
\end{prop}

The corresponding parametric update to bring $M$ closer to $\Id +\gamma
MP$, similar to
Theorem~\ref{thm:paramtd}, is
\begin{equation}
\label{eq:parambackwardtd}
\widehat{\del\theta_{\text{BTD}}} = \gamma\, \partial_{\theta} m_{\theta}(s,s')
	+m_{\theta}(s_1,s) \left(\gamma\,\partial_\theta
			m_\theta(s_1,s')- \partial_\theta m_{\theta}(s_1, s) \right)
\end{equation}
for the model \eqref{eq:Qmodel} using $m_\theta$, and
\begin{equation}
\widehat{\del\theta_{\text{BTD}}} =  \partial_{\theta} \tildem_{\theta}(s,s)
	+\tildem_{\theta}(s_1,s) \left(\gamma\,\partial_\theta
			\tildem_\theta(s_1,s')- \partial_\theta \tildem_{\theta}(s_1, s) \right)
\end{equation}
for the model \eqref{eq:Qmodeltilde} using $\tildem_\theta$.
Here a transition $s\to s'$ and another, independent state $s_1$ are both sampled
from the dataset. A precise statement is given in
Theorem~\ref{thm:paramtd_right} (Appendix~\ref{sec:backwardTD_param}). \todo{or give full statement?}

%
Backward TD for $M$ is not structurally different from forward TD: it
corresponds to forward TD for the ``time-reversed'' Markov process
(Appendix~\ref{sec:traces}). But since states are typically observed in a 
time-ordered sequence, this might produce a difference. In general,
the backward TD update \eqref{eq:parambackwardtd} does not look like a time-reversal of
the forward TD update \eqref{eq:sampleparamtd}: \eqref{eq:parambackwardtd} involves Bellman
gaps of gradients $\partial m$ while \eqref{eq:sampleparamtd} involves
Bellman gaps of $m$. This difference is superficial and disappears in
expectation under the stationary distribution: if we assume that
$\rho$ is the stationary distribution of the process, then
\eqref{eq:parambackwardtd} is equal in expectation to
\begin{equation}
\label{eq:parambackwardtd_rhoinv}
\gamma\, \partial_{\theta} m_{\theta}(s,s')
	+\partial_\theta m_\theta(s_1,s')\left(
        \gamma m_{\theta}(s_1,s) - m_{\theta}(s_1, s') \right)
\end{equation}
which looks more like a time-reversal of the forward TD update
\eqref{eq:sampleparamtd}, with (time-reversed) Bellman gaps of $m$.
\footnote{The
difference between \eqref{eq:parambackwardtd_rhoinv} and
\eqref{eq:parambackwardtd} just lies in shifting terms around along a
trajectory $s\to s'\to s''\to\ldots\,$: in one case, the term $m_{\theta}(s_1,
s') \partial_\theta
m_\theta(s_1,s')$ is grouped with the previous transition $s\to s'$, in
the other case, with the next transition $s'\to s''$. Thus the difference is
minor if working online along trajectories, but
\eqref{eq:parambackwardtd} is valid even if $\rho$ is not the stationary
distribution.}

Moreover, contrary to forward TD, learning $M$ by backward TD then
setting $V=MR$ is \emph{not} equivalent to learning $V$ via TD in the
tabular case.

Mixing forward and backward TD can change the learning of $M$
in various ways. In the tabular case and in the infinitesimal
learning rate limit, such mixing substantially reduces the dimension of the
subspace of $M$ where convergence is slowest
(Section~\ref{sec:mixingfbtd}). With the
matrix-factorized parameterization of Section~\ref{sec:FB}, using forward, or
backward TD, or a mixture of the two, provides approximations of $M$ using
slightly different criteria (Appendix~\ref{app:FB}).

There is no version of backward TD for the goal-dependent optimal $Q$
function of Section~\ref{sec:optQ}. Performing a random step on the
goal state does not commute with optimizing an action depending on the
goal state. With a fixed policy, backward TD is forward TD on a
time-reversed Markov process, but when choosing actions, time reversal is
not possible: in the expectimax problem \eqref{eq:expectimax}, each
action choice may depend both on previous actions and on the goal state,
and reversing time is not possible. Similarly, there is no backward TD
for the target-feature version of Section~\ref{sec:targetfeatures}, as
the features do not generally contain full information about the next
transition and the future features.

\subsection{Path Combinatorics Interpretation: Incorporating Newly
  Observed Transitions}
\label{sec:paths}

\begin{figure}[t]
  \begin{subfigure}{0.32\textwidth}
    \includegraphics[scale=.112]{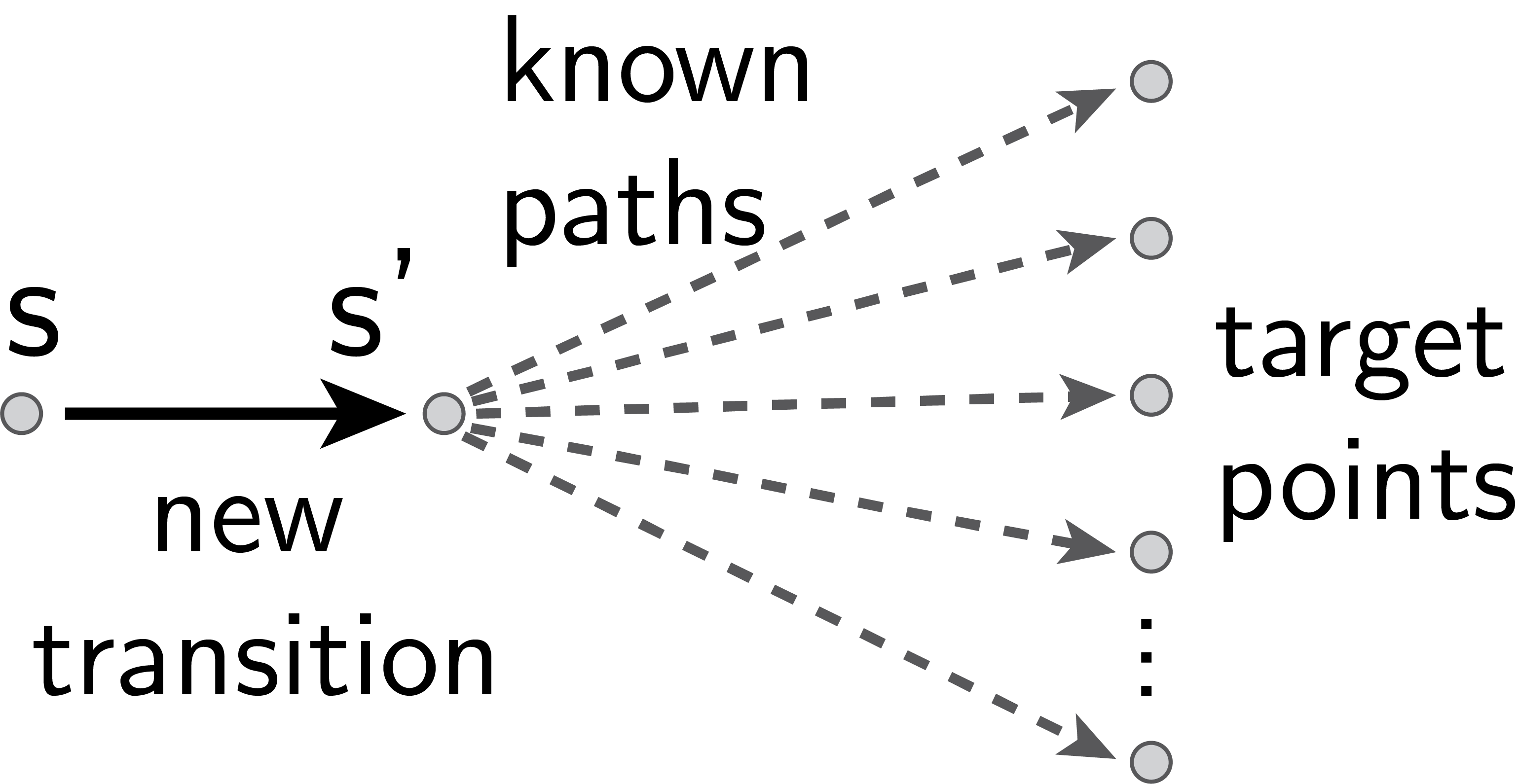}
    \label{fig:ltd_paths}
  \end{subfigure}
  \begin{subfigure}{0.32\textwidth}
    \includegraphics[scale=.112]{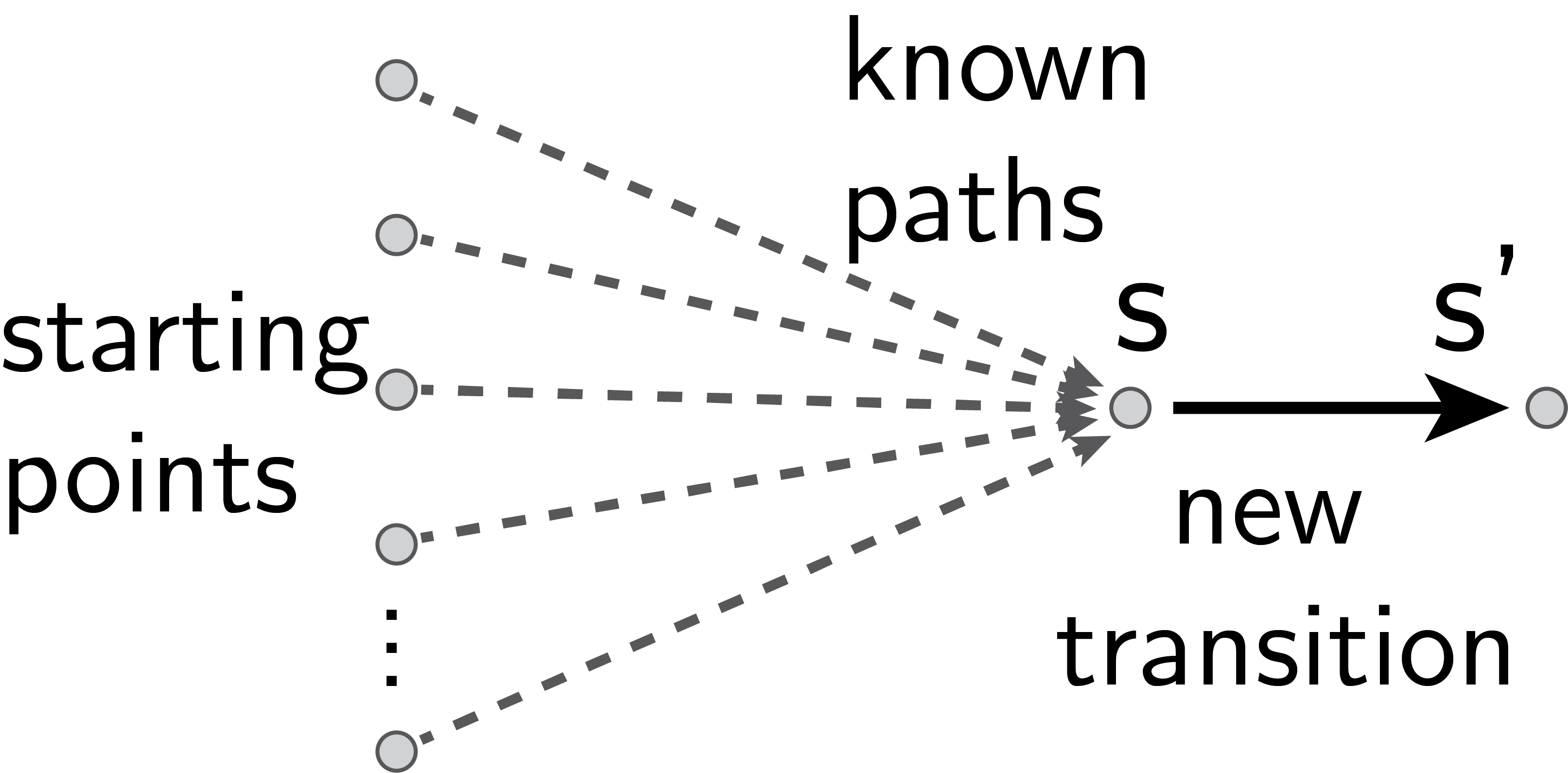}
    \label{fig:rtd_paths}
  \end{subfigure}
  \begin{subfigure}{0.32\textwidth}
    \includegraphics[scale=.112]{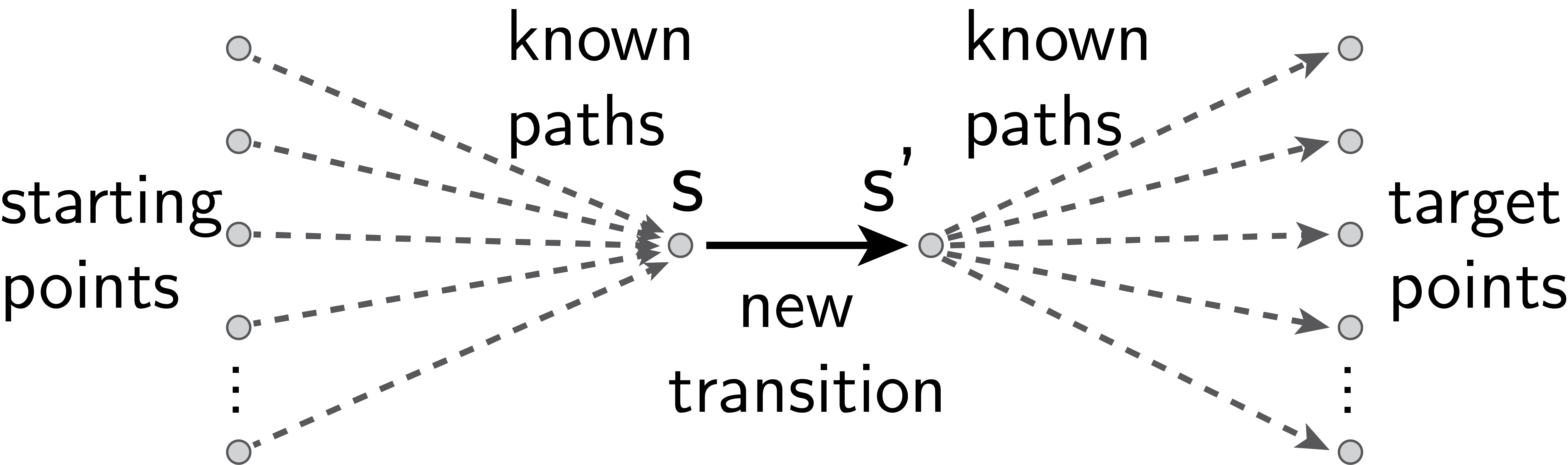}
    \label{fig:bn_paths}
  \end{subfigure}
  \caption{Combining paths: forward TD, backward TD,
  and path composition (Bellman--Newton).}
  \label{fig:combining_paths}
\end{figure}
The difference between forward and backward TD for $M$ is
best understood in the path viewpoint on $M$ (Eq.~\ref{eq:m_path}).
Indeed, the current estimate of $M_{s_1s_2}$ contains a current estimation on the
number of paths from $s_1$ to $s_2$, weighted by their estimated
probabilities in the Markov process. TD replaces $M$ with $PM$, and adds
$\Id$, which represents the trivial paths from $s$ to $s$. Backward TD 
uses $MP$ instead. In both cases, the operator $P$ is sampled via
an observed transition $s\to s'$. Thus, $PM$ builds new known paths by
taking all paths contained
in $M$ and adding the transition $s\to s'$ at the front of each path, while $MP$ adds
the transition $s\to s'$ at the back of each path in $M$
(Fig.~\ref{fig:combining_paths}). Forward TD reasons at fixed \emph{target states}
(rewards) \citep{greydanus2019the},
while backward TD reasons at fixed \emph{starting points}.

Thus, TD and backward TD on $M$ differ in how they learn new paths
from known paths when each new transition is observed. Arguably, both
are reasonable ways to update a mental model of paths in an environment
when discovering new transitions (e.g., if a new street $s\to s'$ opens
in a city).

There is a third way to build new paths when observing a new transition
$s\to s'$: take all known paths to $s$, all known paths from $s'$, and
insert $s\to s'$ in the middle (Fig.~\ref{fig:combining_paths}). This exploits
path concatenation, roughly
doubling the length of known paths.
This operation is involved in the way that $M$ actually changes when the
process is changed by increasing $P(s,\d s')$ (the way possible paths
actually change when a new street opens). This is the basis of the
``second-order''
algorithms we present for $M$ in Section~\ref{sec:BN}.

\section{Multiple Policies: Goal-Dependent $Q$ and $V$ functions}
\label{sec:goals}

The principles above can be used to learning goal-dependent policies and
a goal-dependent value or $Q$ function, just by letting the policy be
goal-dependent in the results above.  This option only covers rewards
located at a given target state, not dense rewards; it can also cover
target \emph{features} of states rather than a fully specified target
state (Section~\ref{sec:goaldependentV}), namely, having target values
for some function $\phi$ of the state.  A first application is to
learn all optimal policies to reach any goal state $s_2$, either via
$Q$-learning (Section~\ref{sec:optQ}) or $V$-learning
(Section~\ref{sec:goaldependentV}).

This approach partially solves the well-known
sparse reward problem in goal-dependent learning.
For instance, let us consider the goal-dependent value function
$V(s,\d s_2)$: for every target state $s_2$, it solves the $s_2$-dependent Bellman equation
\begin{equation}
\label{eq:goaldependent-intro}
V(s,\d s_2)=\delta_{s}(\d s_2)+\gamma \E_{s'\sim P(\d s'|s,s_2)} V(s',\d
s_2)
\end{equation}
with reward when $s=s_2$, and $s_2$-dependent policy $P(\d s'|s,s_2)$.
(In the continuous
case, the goal-dependent value function is a \emph{measure} on $s_2$,
because the probability to exactly reach a state is usually $0$. We will
learn
its density with respect to a reference measure.)

Using TD directly on this equation leads to sparse reward problems: in
continuous state spaces, a reward is never observed (and rarely observed
in large discrete spaces).

However, the contribution of
the reward $\delta_{s}(\d s_2)$ to the TD update can be computed exactly in
expectation.
The resulting update does not involve sparse rewards
anymore: every transition is informative because it shows how to reach
the currently visited state (as discussed in Section~\ref{sec:paramtd}). 
This update is the same as with successor states: the update for
$\tilde m$ in Theorem~\ref{thm:paramtd} can be directly used to train
goal-dependent policies by seeing $\tildem(s,s_2)$ as the value function
at $s$ when the goal state is $s_2$ (see the example after
Theorem~\ref{thm:goalTD}).

Existing workarounds for this sparse reward issue include strategies for resampling
state-goal pairs that more frequently lead to nonzero rewards, such as
\emph{Hindsight Experience Replay} (HER) \citep{HER}, which works with any
$Q$-learning method, assuming knowledge of the reward function associated
to each goal.
It is not clear to us whether HER
actually solves the infinitely-sparse-reward issue or not. \footnote{For
instance, 
with noisy dynamics in a continuous space, the probability to 
reach a state exactly is always $0$, so if the reward is $1$ when
reaching the state, the $Q$-function computed by HER would be $0$. Here
we have used infinite (Dirac) rewards when reaching a goal: this leads
to a well-defined, nonzero $Q$ function, but rescaling the reward by an
infinite factor would result in infinite HER
updates. On the other hand, in some non-noisy continuous MDPs with continuous
actions, it is possible to reach a state exactly, and in that instance
HER would work without modification. This point needs more investigation.}
The results described here are not mutually
exclusive with using HER: HER is a sampling strategy for transitions in
the training set, which can be used with any $Q$-learning method, such as
those presented here; so 
in principle HER could be used as the state-goal distribution $\rho_{SG}$
for $Q$-learning in Theorem~\ref{thm:goalTD}.

\bigskip

We start with learning the optimal $Q$ function for every target state
(Section~\ref{sec:optQ}).
We first describe the precise
meaning of goal-dependent Bellman equations such as
\eqref{eq:goaldependent-intro}, and present the resulting parametric
update.

Next we turn to a more general statement involving either the $V$
or $Q$ function, and target \emph{features} instead of target states
(Section~\ref{sec:goaldependentV}). We discuss three use cases:
$Q$-learning with any goal feature function, $V$-learning conditioned to
goal states, and $V$-learning conditioned to goal features, which
presents some subtleties. The
goal-dependent $V$ function can be used to train goal-dependent policies
by any policy gradient method.

In Section~\ref{sec:Qoptexistence} we provide mathematical details
for the existence and uniqueness of goal-dependent Bellman equations, in
the case of the $Q$ function. Having to work with measures of potentially
infinite mass results in
non-uniqueness of the solution, but there is still a ``natural''
solution, equal both to the smallest solution and to the limit of the
finite-horizon solution.

\subsection{The Optimal $Q$-function for Every Goal State}
\label{sec:optQ}


Several works have attempted to learn optimal $Q$ functions indexed by
an additional ``goal'' which encodes a variable reward. The simplest case
is when the reward is located at a single goal state $g$. Computing the
$Q$ function $Q(s,a,g)$ for every goal state $g$ fully solves the
navigation problem in an environment, although this function does not provide
the optimal policies for ``mixed'' rewards, only for single-state
rewards.

The viewpoint presented here allows for a more principled approach to
this object $Q$; notably, it can avoid the sparse reward problem of
algorithms that sample a state $s$ and a goal state $g$ independently,
with reward
$\1_{s=g}$.
This is
avoided thanks to the direct algebraic treatment of Diracs or sparse rewards discussed
above.

So far, the
successor state operator was defined for a given, fixed policy. The goal-dependent $Q$ function uses a different (optimal) policy
for every goal. It can be defined through the optimal Bellman operator.

\begin{defi}[ (Optimal Bellman operator for successor states)]
\label{def:optQ}
Let $Q(s,a,\d s_2)$ be a measure on $s_2$ depending on a state-action
pair $(s,a)$. Define the \emph{optimal Bellman operator} $T$ via
\begin{equation}
\label{eq:optQ}
(TQ)(s,a,\d s_2)\deq\delta_s(\d s_2)+\gamma \,\E_{s'\sim P(s'|s,a)} \sup_{a'} Q(s',a',\d
s_2).
\end{equation}
\end{defi}

In the discrete case, this is just the usual optimal Bellman operator in
parallel for every goal state $s_2$, namely,
$(TQ)(s,a,s_2)=\1_{s=s_2} +\gamma \,\E_{s'\sim P(s'|s,a)} \max_{a'}
Q(s',a', s_2)$.
In the continuous case, for
each
state-action $(s,a)$, $Q(s,a,\cdot)$ is a measure over the state space,
and the supremum $\sup_{a'} Q(s',a',\d
s_2)$ is a supremum of measures over $s_2$.
\footnote{In general, the supremum of $k$ measures
$\mu_1,\ldots,\mu_k$ is defined as follows: for every measurable set $A$,
$(\sup_i \mu_i)(A)\deq \sup_{(B_i)} \sum\mu_i(B_i)$ where the supremum is taken
over all partitions of $A=B_1\sqcup B_2 \sqcup \cdots \sqcup B_k$ into
disjoint measurable sets $(B_i)$. This is also the smallest measure that is
larger than every $\mu_i$. Each $B_i$ is the set where 
$\mu_i$ is the largest measure in the family. This means
that at each point in state space, we select the measure with the highest
value; thus, the $\sup$ over actions in \eqref{eq:optQ} depends on the
goal states $s_2$.

The definition assumes that the set of actions is
countable; otherwise, additional smoothness assumptions are required for
existence.}

In the discrete case, a fixed point exists by standard
contractivity arguments; however, with continuous states, the situation
is
tricky, see Section~\ref{sec:Qoptexistence}. In particular, with continuous states
the measure $Q$ may have either finite or infinite mass; intuitively, the
total mass of $Q$ indicates how many distinct policies we can follow to
reach different states. The total mass of $Q(s,a)$ is the total
number of distinct points that can be reached from $(s,a)$ by taking
different action sequences, weighted by the probability and discounted
by time. In contrast, the successor state operator of a single fixed policy
(Sections~\ref{sec:defM}--\ref{sec:td-alg}) always has total mass
$\sum\gamma^t=\frac{1}{1-\gamma}$: there is no choice of actions so the total
probability of states is $1$ at each time step.

To see this, consider two extreme examples. In the first, the environment
just ignores every action and sends the agent to a random uniform state at each
time step. Then for any $(s,a)$, $Q(s,a,\d s_2)$ is $\delta_s(\d
s_2)+\frac{\gamma}{1-\gamma} \d s_2$, with total mass
$\frac{1}{1-\gamma}$. In the second example, for every state we have an action
that sends us directly to that state. Then $Q(s,a,\d s_2)$ is a measure
for which \emph{every} single state $s_2\neq s$ has mass
$\frac{\gamma}{1-\gamma}$, and the total mass of $Q(s,a,\cdot)$ is
infinite. This can be arranged even with finite action spaces: generally,
at horizon $t$ the mass may be as large as $\gamma^t (\# A)^t$ if every
action sequence leads to a different part of the state (examples in
Appendix~\ref{sec:exqopt}). In the fixed-policy case, the mass at horizon $t$ was
always $\gamma^t$ and the total mass was always finite.

\paragraph{Parametric goal-dependent $Q$-learning.}
Let us consider parametric models for $Q$.
As before, we consider two models given by
\begin{equation}
\label{eq:Qmodel}
Q_\theta(s,a,\d s_2)\deq \delta_s(\d s_2)+q_\theta(s,a,s_2)\rho(\d s_2)
\end{equation}
and
\begin{equation}
\label{eq:Qmodeltilde}
Q_\theta(s,a,\d s_2)\deq \tildeq_\theta(s,a,s_2)\rho(\d s_2)
\end{equation}
respectively, and we will learn $q_\theta$ and $\tildeq_\theta$. For
instance, up to the factor $\rho$, the models in \citep{pmlr-v37-schaul15} correspond
to $\tildeq_\theta(s,a,s_2)=h(\phi_\theta(s,a),\psi_\theta(s_2))$.
\footnote{The factor $\rho$, or some other measure, is needed to get a well-defined object in
continuous state spaces. In discrete spaces, it results in an
$s_2$-dependent scaling of the $Q$ function, which still has the same
optimal policy for each $s_2$.}

The resulting parametric update is as follows. The update is off-policy:
we assume access to a dataset of transitions $(s,a,s')$ in a Markov decision
process. Let $\rhoSA$ be the distribution of the state-action pair
$(s,a)$ in the dataset; its marginal over $s$ is $\rho$ as before. Given
a measure-valued function of $(s,a)$, such as $Q(s,a,\d s_2)$, we define its norm
similarly to \eqref{eq:norm} as
\begin{equation}
\label{eq:normq}
\norm{Q}^2_{\rhoSA,\rho}\deq \E_{(s,a)\sim\rhoSA,\,s_2\sim \rho}[
q(s,a,s_2)^2]
\end{equation}
where $q(s,a,s_2)\deq Q(s,a,\d s_2)/\rho(\d s_2)$ is the density of $Q$
with respect to $\rho$, if it exists (otherwise the norm is infinite).

\begin{thm}[ (Parametric $Q$-learning for every goal state)]
\label{thm:paramq}
Consider a parametric model of $Q$ given by \eqref{eq:Qmodel} or
\eqref{eq:Qmodeltilde}, where $q_\theta(s,a,s_2)$ or $\tildeq(s,a,s_2)$
are smooth functions depending on the parameter $\theta$.

Let $\theta_0$ be some value of the parameter. Define a target update
$Q^\tar$ of
$Q$ via the optimal Bellman operator \eqref{eq:optQ} applied to
$Q_{\theta_0}$, namely, $Q^\tar(s,a,\d s_2)\deq (TQ_{\theta_0})(s,a,\d
s_2)=\delta_s(\d
s_2)+\gamma\,\E_{s'\sim P(s'|s,a)} \sup_{a'} Q_{\theta_0}(s',a',\d
s_2)$. Define the loss between $Q_\theta$ and $Q^\tar$ via $J(\theta)\deq
\frac12 \norm{Q_{\theta}-Q^\tar}^2_{\rhoSA,\rho}$ using the norm
\eqref{eq:normq}.

Then the gradient step on $\theta$ to reduce this loss
is
\begin{multline}
-\partial_\theta J(\theta)=
\E_{(s,a)\sim \rhoSA,\,s'\sim P(s'|s,a),\,s_2\sim\rho}
\left[
\gamma \, \partial_\theta q_\theta(s,a,s')
\right.
\\
\left.
+\,\partial_\theta q_\theta(s,a,s_2)\left(
\gamma \sup_{a'}q_{\theta_0}(s',a',s_2)-q_{\theta}(s,a,s_2)
\right)
\right]
\end{multline}
for the model \eqref{eq:Qmodel} using $q_\theta$, and
\begin{multline}
-\partial_\theta J(\theta)=
\E_{(s,a)\sim \rhoSA,\,s'\sim P(s'|s,a),\,s_2\sim\rho}
\left[
\partial_\theta \tildeq_\theta(s,a,s)
\right.
\\
\left.
+\,\partial_\theta \tildeq_\theta(s,a,s_2)\left(
\gamma \sup_{a'}\tildeq_{\theta_0}(s',a',s_2)-\tildeq_{\theta}(s,a,s_2)
\right)
\right]
\end{multline}
for the model \eqref{eq:Qmodeltilde} using $\tildeq_\theta$.
\end{thm}

Here we have presented the update using a fixed ``target network'' with parameter
$\theta_0$ (typically a previous value of $\theta$), a common practice
for parametric $Q$-learning.

This update is ``samplable'': sample a transition $(s,a,s')$
from the dataset, another independent transition $(s_2,a_2,s'_2)$ from
the dataset ($a_2$ and $s'_2$ are discarded), and estimate the gradient
by
\begin{equation}
    \label{eq:sampleparamtd}
    \widehat{\del\theta} =  \gamma\,\partial_{\theta} q_{\theta}(s,a,s')
    +\partial_{\theta} q_{\theta}(s,a,s_2) \left(\gamma \sup_{a'}
    q_{\theta_0}(s',a',s_2)-
    q_{\theta}(s,a, s_{2}) \right)
\end{equation}
or likewise for $\tildeq$ (only the first term is different).

This update is perfectly analogous to the successor state updates for
$m_\theta$ and $\tildem_\theta$ in
Theorem~\ref{thm:paramtd}, except that $q_\theta$ and $\tildeq_\theta$
depend on the actions, and that the policy follows a supremum over
actions instead of being fixed.

As before, the infinite, infinitely sparse rewards $\delta_{s}(\d s_2)$
of the every-goal problem
produce the finite contribution $\gamma \,\partial_{\theta}
q_\theta(s,a,s')$ or $\partial_\theta \tildeq(s,a,s)$ in this parametric
update. Sampling two independent states $s$ and $s_2$ is still needed, but for the
Bellman gap term, not for the reward term.

%

%

\subsection{Value and $Q$ Functions with State Features as Goals}
\label{sec:goaldependentV}

We now turn to a general result covering both value and $Q$ functions
($Q$ functions are obtained as the value function of the state-action
Markov process, as explained in Section~\ref{sec:notation}). We also
cover target \emph{features} rather than target states: namely, we are
given a feature function $\phi$ on state space, and the reward is nonzero
on states $s$ such that $\phi(s)$ achieves a particular goal value $g$.
Target states correspond to $\phi=\Id$.

Covering $V$ functions requires the ability to work on-policy. Thus, we
assume that goal-dependent policies are given, yielding goal-dependent
transitions $s\to s'|g$ defined by their transition probabilities $P(\d
s'|s,g)$.

Thus we wish to find solutions to the goal-dependent Bellman equation
\begin{equation}
\label{eq:goalbellman}
V(s,\d g)=\delta_{\phi(s)}(\d g)+\gamma \,\E_{s'\sim P(\d s'|s,g)}
V(s',\d g)
\end{equation}
with reward on states such that the features $\phi(s)$ are equal to $g$.
Full target states amount to $\phi=\Id$: a nonzero reward when $s=g$.
This can be used in turn to train the goal-dependent policies, for
instance by policy gradient.
(The technical meaning of this equation is similar to the case of $Q$
above. For a discussion on existence and uniqueness we refer to
Section~\ref{sec:Qoptexistence}.)

Here the training dataset is made of triplets $(s\to s'|g)$: transitions
indexed by a goal.  For
the $Q$ function this is not restrictive: working on state-action pairs,
given a state-action $(s,a,s')$ it is always possible to sample a goal $g$
a posteriori, and to define the next action $a'$ according to policy $g$
in state $s'$. For the $V$ function this is more restrictive: typically,
the training set would be made of trajectories where a goal is selected
at random and kept for some time. This results in some empirical
distribution over state-goal pairs $(s,g)$ in the training set, with $s$
and $g$ not independent.

A major issue is to avoid using the sparse rewards $\delta_{\phi(s)}(\d
g)$. Indeed, the most obvious approach to the Bellman equation
\eqref{eq:goalbellman} is to view this problem as an ordinary Markov
process on the augmented state space of state-goal pairs $(s,g)$. 
The TD
update for
this problem is
\begin{align}
\label{eq:goalTD_naive}
\del \theta&=\E_{(s,g)\sim \rho_{SG},\,s'\sim P(\d s'|s,g)}
\left[
\partial_\theta v_\theta(s,g)\left(
\frac{\delta_{s}(\d g)}{\tau(\d g)}+\gamma
v_\theta(s',g)-v_\theta(s,g)
\right)
\right]\nonumber
\end{align}
where $\rho_{SG}$ is the distribution of state-goal pairs $(s,g)$ in the
training set, and where the $V$ function has been parameterized as
$V_\theta(s,\d g)=v_\theta(s,g)\tau (\d g)$ for some arbitrary measure
$\tau(\d g)$ on goal space. In a continuous state space, no reward would
ever be observed.

Sparse rewards can be avoided by just using the goal $g=\phi(s)$ for the
sparse term: $\partial_\theta v_\theta(s,g)\,\delta_s(\d g)\leadsto
\partial_\theta v_\theta(s,\phi(s))$. The price to pay is computing the value
function only \emph{up to a goal-dependent scaling}. Namely, there is a
simple sparsity-free TD update for the related problem
\begin{equation}
V(s,\d g)=\alpha(s,g)\,\delta_{\phi(s)}(\d g)+\gamma \,\E_{s'\sim P(\d s'|s,g)}
V(s',\d g).
\end{equation}
Here the reward is nonzero only if $\phi(s)=g$, but with an unknown
factor $\alpha(s,g)$ that depends on the solution reached.

If $\alpha(s,g)$ depends only on $g$, then optimal policies are not
affected: for every goal $g$,
we just compute
the correct value function for this goal
up to a $g$-dependent scaling.
This happens in many use cases, notably for
$Q$-learning or if $\phi=\Id$ (goals are full states), as shown below.

The least favorable use case is $V$-learning with $\phi\neq
\Id$. Then the scaling $\alpha$ may also vary among the states $s$ which
achieve $\phi(s)=g$: this may result in policies which do solve the
problem of finding a state $s$ with $\phi(s)=g$, but not necessarily in
an optimal way. (In that case, another option is to explicitly provide a
full state $s_g$ such that $\phi(s_g)=g$ and use the full state $s_g$ as
the goal instead, thus going back to $\phi=\Id$.)

We now turn to the technical, general statement and discuss some explicit use
cases. The theorem is stated for $V$ functions; the case of $Q$
functions follows by applying it to the state-action Markov process.
\todo{specify loss and norm in the statement?}

\newcommand{\Goal}{\mathcal{G}}

\begin{thm}[ (Goal-dependent TD)]
\label{thm:goalTD}
Let $\phi\from S\to \Goal$ be a function from the state space to some
goal space $\Goal$ (discrete if $S$ is discrete and continuous if $S$ is
continuous).

Assume that the training set consists of transitions $(s\to s'|g)$
indexed by a goal.
Let the joint distribution of state-goal pairs in the training set be
$\rho_{SG}(\d s,\d g)$. Let $\rho_S$ and $\rho_G$
be its marginals over $s$ and $g$, namely, the distributions of states and
of goals in the training set. Assume that the density of $\rho_{SG}$ with
respect to $\rho_S \rho_G$ is nonzero everywhere (every state-goal pair
appears with some positive probability).

Parameterize the value function as $V_\theta(s,\d g)=v_\theta(s,g)
\rho_G(\d g)$.

Then the parameter update
\begin{align}
\label{eq:goalTD}
\del \theta&=\E_{(s,g)\sim \rho_{SG},\,s'\sim P(\d s'|s,g)}
\left[\, \partial_\theta v_\theta(s,\phi(s))
+\partial_\theta v_\theta(s,g)\left(\gamma
v_\theta(s',g)-v_\theta(s,g)
\right)
\right]
\end{align}
is the TD update associated with the Bellman equation for the goal-dependent value function 
\begin{equation}
\label{eq:goaldependentV}
V(s,\d g)=\alpha(s,g)\,\delta_{\phi(s)}(\d g)+\gamma \,\E_{s'\sim P(\d s'|s,g)}
V(s',\d g).
\end{equation}
where
$\alpha(s,g)\deq \rho_S(\d s)\rho_G(\d g)/\rho_{SG}(\d s,\d g)$.
(Note that the value of $\alpha(s,g)$ is used only on states such that $\phi(s)=g$.)

In the following cases, $\alpha$ depends only on $g$:
\begin{itemize}
\item If the distributions of states and goals are independent in the training
set, namely, if $\rho_{SG}(s,g)=\rho_S(s)\rho_G(g)$, then
$\alpha(s,g)=1$.
\item If $\phi=\Id$ (goals are full states) then the statement also holds with
$\alpha(g,g)$ instead of $\alpha(s,g)$ in \eqref{eq:goaldependentV}.
\end{itemize}
\end{thm}

Concretely, a stochastic update $\del \theta$ is obtained by sampling
from the dataset a
transition $(s\to s'|g)$ indexed by a goal $g$, and then updating by
\begin{equation}
\del \theta=
\partial_\theta v_\theta(s,\phi(s))
+\partial_\theta v_\theta(s,g)\left(\gamma
v_\theta(s',g)-v_\theta(s,g)
\right).
\end{equation}
This is similar to the update of $\tildem_\theta$ 
in successor states (Theorem~\ref{thm:paramtd}), except here the policy
depends on the goal. This can be used in turn to train a goal-depedent
policy (Section~\ref{sec:V}).

This theorem can work out in three different ways:
\begin{itemize}
\item 
$Q$-learning works for any goal features $\phi$, using an ordinary
off-policy training set of transitions $((s,a)\to s')$. In that case,
there is no need for transitions to be indexed by a goal.
This follows from the
theorem applied to the state-action process, and yields $\alpha=1$.
\item $V$-learning works best with full goal states ($\phi=\Id$). This
requires a training set of transitions $(s\to s'|g)$ each indexed by a goal
state (such as exploring with a given goal for some time). A goal-dependent policy can be trained by any policy gradient
method. In that case, $\alpha$ depends only on $g$, thus, computing the
value function for every $g$ up to a $g$-dependent scaling that does not
affect the optimal policy for $g$.
\item $V$-learning can be applied with any goal features $\phi$, but the
resulting algorithm implicitly reweights the rewards among those
states which achieve a given goal. Goal-dependent policies training by
policy gradient will still reach a state such that $\phi(s)=g$, but not
necessarily in an optimal way, with certain states implicitly preferred.
\end{itemize}

Let us discuss the first two cases in more detail.

With $Q$-learning, it is possible to pick any goal a posteriori for any
observed transition $((s,a)\to s')$. So goals and states can be picked
independently, resulting in $\alpha=1$. This plays out as follows: Assume
the training set is made of transitions $((s,a)\to s')$, that we have a
set of goals $g\in \Goal$, and that we maintain the value function
$v_\theta((s,a),g)$ over state-action pairs. Assume we have $g$-dependent
policies $\pi_g$, such as the greedy policy obtained from the $Q$-function
$v_\theta((s,a),g)$. Then the expected TD update \eqref{eq:goalTD} can be
realized by picking at random a transition $((s,a)\to s')$ in the dataset,
picking at random a goal $g\sim \rho_G(\d g)$ according to any
user-chosen distribution on goals, picking an action $a'\sim \pi_g(s')$,
and updating the parameter via
\begin{equation}
\del \theta=
\partial_\theta v_\theta((s,a),\phi(s,a))
+\partial_\theta v_\theta((s,a),g)\left(\gamma
v_\theta((s',a'),g)-v_\theta((s,a),g)\right).
\end{equation}

With $V$-learning, it is unreasonable to assume that goals and states are
independent in the training set: this would require an exploration policy
which randomly changes goals at every step. A more reasonable exploration
policy would pick a goal $g\sim \rho_G$ and keep it for some time, using
the goal-dependent policy $\pi_g$. This
results in a set of transitions $(s\to s'|g)$ indexed by their goals, with
a non-independent distribution of goals and visited states,
thus $\alpha\neq 1$. If $\phi=\Id$ the theorem states that $\alpha$ only
depends on $g$ so that optimal policies are not affected. The expected TD
update \eqref{eq:goalTD} can be
realized by picking at random a transition $(s\to s'|g)$ in the dataset
and updating the parameter via
\begin{equation}
\del \theta=
\partial_\theta v_\theta(s,s)
+\partial_\theta v_\theta(s,g)\left(\gamma
v_\theta(s',g)-v_\theta(s,g)
\right).
\end{equation}
This update of $v_\theta$ is identical to the update of $\tildem_\theta$
in successor states (Theorem~\ref{thm:paramtd}), except here the
transitions (or policy)
depends on the goal.

\bigskip

There is no variance from sparse rewards in these expressions: the reward
term produces the term $\partial_\theta v_{\theta}(s,\phi(s))$, namely, a
term directly evaluated at the goal $g=\phi(s)$ associated with the
currently visited state $s$.
(But there is still some variance from the Bellman gap part of the
expression.)
Thus, when learning goal-dependent value or $Q$ functions with sparse
rewards, it is possible to avoid the sparse reward problem by directly
setting the goal $g=\phi(s)$ for the reward term in the TD update.

For comparison, algorithms such as hindsight experience replay store a
mixture of state-related and state-independent goals in a training
dataset of transitions, to be used with any off-policy learning
algorithm. As in our setting, they assume knowledge of the reward
function (such as $\delta_{\phi(s)}(\d g)$) and access to a way to build
goals from states, such as $\phi$. This provides a strategy for
building a relevant state-goal distribution in the training set. Such an
approach is independent from our results, which directly reduce the
variance in the $Q$-learning update. Thus in principle both approaches
can be used simultaneously. \todo{NOT SURE, DOUBLE-CHECK! see comment above}

Multi-step, horizon-$k$ versions of TD (Appendix~\ref{sec:horizonk}) do
not seem to be available in the goal-dependent setting in a version that
avoids the infinitely-sparse Dirac reward problem.

\subsection{Existence and Uniqueness of Optimal Successor States}
\label{sec:Qoptexistence}

We now turn to finding a solution to the optimal goal-dependent Bellman equation $TQ=Q$.

For discrete, infinite Markov reward processes, the value function that solves the
Bellman equation is in general not unique; it is unique under additional
constraints such as boundedness. \footnote{For instance, consider the
simple random walk on the state space $\Z$, which goes right or left with
probability $1/2$. Given $\gamma<1$, let $\phi$ be the solution to
$\phi=\frac{\gamma}{2}(1+\phi^2)$ and define $f(s)\deq \phi^s$ for $s\in
\Z$. Then by construction,
$f(s_t)=\gamma \E_{s_{t+1}|s_t} f(s_{t+1})$. Thus, if $V$ is any solution
to the Bellman equation, then $V+f$ is another solution. Such solutions
``believe there is an infinite reward at infinity''.}

For the optimal goal-dependent $Q$-function, we cannot impose
boundedness, since the solution sometimes has infinite mass. \footnote{For the
same reason, contractivity arguments will not work in the proofs, as it
is hard to find a norm that is finite and nonzero in every situation. The
arguments rely on monotonicity of the Bellman operator.}
Instead, we prove that the solution for the horizon-$t$
problem exists and converges to the \emph{smallest}
solution of the Bellman equation when $t\to\infty$. 

\newcommand{\zerom}{\mathbf{0}}

Let $Q_t$ be the goal-dependent $Q$-function at bounded horizon $t$,
obtained by expanding the expectimax problem at horizon $t$,
namely
\begin{multline}
\label{eq:expectimax}
Q_t(s_1,a_1,\d s_g)\deq \delta_{s_1}(\d s_g)\,+
\\
\E_{s_2\sim P(s_2|s_1,a_1)}
\sup_{a_2}\left[
\gamma \delta_{s_2}(\d s_g)+
\cdots \E_{s_t\sim P(s_t|s_{t-1},a_{t-1})}
\sup_{a_t}\left[
\gamma^t \delta_{s_t}(\d s_g)
\right]
\cdots
\right].
\end{multline}
This $Q_t$ can also be described via the optimal Bellman operator $T$ as
$Q_t=T^t\zerom$, with $\zerom$ the
zero measure. (In the following, ``$Q$ is a measure'' is short for ``for
every state-action $(s,a)$, $Q(s,a,\cdot)$ is a measure''.)

\begin{thm}
\label{thm:optQ}
Let $T$ be the optimal Bellman operator of Definition~\ref{def:optQ}. Let
$\zerom$ be the measure with mass $0$.

Let $Q_t\deq T^t\zerom$ be the goal-dependent $Q$-function at horizon
$t$. Then when $t\to\infty$, $Q_t$ converges strongly \footnote{Namely,
for every state-action $(s,a)$ and for every measurable set $A$,
$Q_t(s,a,A)$ converges to $Q(s,a,A)$. } to a measure $Q^\ast$. This limit
solves the Bellman equation $TQ^\ast=Q^\ast$, and is the smallest such solution. In
finite state spaces, it is the only solution with finite mass.
\end{thm}

The solution is never unique: the measure that gives infinite mass to
every set is another. Hence the interest of considering the
\emph{smallest} solution, where the values come from rewards actually
picked at some time $t$.

\section{Matrix Factorization and the Forward-Backward (FB) Representation}
\label{sec:FB}

\newcommand{\thetaF}{{\theta_F}}
\newcommand{\thetaB}{{\theta_B}}

\subsection{Advantages of Matrix Factorization for $M$}

In this section we study a specific parametric model for the successor
state operator, which has many advantages: a ``matrix-factorized''
representation. We consider the model \eqref{eq:Mmodeltilde}, namely $M(s_1,\d
s_2)\approx \tildem_\theta(s_1,s_2)\rho(\d s_2)$, with the particular choice
\begin{equation}
\tildem_\theta(s_1,s_2)=\transp{F_\thetaF(s_1)}B_{\thetaB}(s_2)
\end{equation}
where $F\from S\to \R^r$ and $B\from S\to \R^r$ are two learnable
functions from the state space to some representation space $\R^r$,
parameterized by $\theta=(\thetaF,\thetaB)$. This provides an
approximation of $M$ by a rank-$r$ operator. Such a factorization is used
for instance in \citep{pmlr-v37-schaul15} for the goal-dependent
$Q$-function (up to the factor $\rho$).

Intuitively, $F$ is a
``forward'' representation of states and $B$ a ``backward''
representation: if the future of $s_1$ matches the past of $s_2$, then
$M(s_1,\d s_2)$ is large.
The learning algorithms presented above (forward and backward TD for $M$)
can be directly applied to this parameterization, leading to explicit
updates for $F$ and $B$ (Section~\ref{sec:FBeq}).

This representation of $M$ has a number of advantages and some
shortcomings. (In this section we deal mostly with successor states; for
goal-dependent value functions, this representation has fewer
advantages.) The advantages are as follows.

\begin{itemize}
\item It provides a direct representation of the value function at every
state, without
learning an additional model of $V$. Namely,
\begin{equation}
\label{eq:VFB}
V(s)\approx \transp{F(s)}B(R),\qquad B(R)\deq \E_{s\sim \rho} [r_s\,B(s)]
\end{equation}
where the ``reward representation'' $B(R)$ can be directly estimated by
an
online average of $B(s)$ weighted by the reward $r_s$ at $s$. This
is a direct consequence of \eqref{eq:Vmtilde}. For instance, with sparse
rewards, each time a reward is observed, the value function is updated
everywhere. \footnote{The model of $M$ using $m$ instead of $\tildem$ is less
convenient for $V$, leading to $V(s)=R(s)+\transp{F(s)}B(R)$, thus
requiring a model of the expected reward $R(s)$.}

This point applies to successor states, but not to goal-dependent value
functions, which cannot handle arbitrary rewards.

\item It simplifies the sampling of a pair of states $(s,s_2)$ needed for
forward TD. Indeed, the forward TD update \eqref{eq:paramtdtilde}
factorizes as an expectation over $s$, times an expectation over $s_2$
(Section~\ref{sec:FBeq} below), which can be estimated independently. The
same applies to backward TD.
This can potentially reduce variance a lot, and even allows for purely
``trajectory-wise'' online estimates using only the current transition
$s\to s'$, without sampling of
another independent state $s_2$. 
(Once more, this works for successor states but not for goal-dependent
value functions, since in that case the transitions $s\to s'$ depend on
$s_2$.)

\item It produces two (policy-dependent) representations of states, a
forward and a backward one, in a natural way from the dynamics of the MDP
and the current policy. This could be useful for other purposes.

\item Even in the tabular case, when the state space is discrete and unstructured, this
provides a form of prior or generalization between states (based on a
low-rank prior for the successor state operator). States that
are linked by the MDP dynamics get representations $F$ and $B$ that are
close.

\item It has some of the properties of the
second-order methods of Section~\ref{sec:BN}, without their complexity.
This is proved in Appendix~\ref{sec:FBandBN}.
\end{itemize}

The shortcomings are as follows:

\begin{itemize}
\item It approximates the successor state operator by an operator of rank
at most $r$. This is never an exact representation unless the
representation dimension $r$ is at least the number of distinct states.

\item The best rank-$r$ approximation of $(\Id-\gamma P)^{-1}$ erases the
small
singular values of $P$: thus this representation will tend to
erase ``high frequencies'' in the reward and value function, and provide
a spatially smoother approximation focusing on long-range behavior. This is fine as long as the reward
is not a ``fast-changing'' function made up of high frequencies (such as a ``checkerboard''
reward).

This can be expected: learning a reward-agnostic object such as
$M$ cannot work equally well for all rewards.
For these reasons, it may be useful to use a mixed model for the value
function with the FB model as a baseline, such as
\begin{equation}
V_\phi(s)=\transp{F(s)}B(R)+v_\phi(s)
\end{equation}
where $F$ and $B$ are learned via successor states, $B(R)$ is as in
\eqref{eq:VFB}, and $\phi$ is learned
via ordinary TD on the remainder. The $FB$ part will catch reward-independent,
long-range behavior, while the $v_\phi$ part will be needed to catch high
frequencies in a particular reward function.

\end{itemize}

\paragraph{Why is a matrix-factorized form relevant for $M$?} Small-rank
approximations of a matrix are relevant when the matrix has a few large
eigenvalues and many small eigenvalues (or singular values, depending on
the precise criterion).
Since the successor state operator is the inverse
of $\Id-\gamma P$, this means the approximation is reasonable if
$\Id-\gamma P$ has few small eigenvalues and many large eigenvalues.

The spectrum of Markov operators is a well-studied topic. For
continuous-time operators associated with random diffusions, possibly
with added drift, the spectrum generally follows \emph{Weyl's law}
\citep{wikipedia_weyllaw}: in
dimension $d$, the continuous-time analogue of $\Id-P$ has roughly
$k^{d/2}$
eigenvalues of size $\leq k$, thus, few small and many large eigenvalues.

The simplest example is a random walk on a discrete torus
$[1;n]$. The operator $P$ is diagonal in Fourier representation,
with eigenvectors $e^{2i\pi k x/n}$ with $k$ an
integer. The corresponding eigenvalue of $P$ is $
\cos (2\pi k/n)$, yielding an eigenvalue $(1-\gamma)+2\gamma \sin^2
(\pi k/n)$ for $\Id-\gamma P$.
The largest eigenvalue of $P$ is $1$ (for $k=0$)
corresponding to the smallest eigenvalue $1-\gamma$ for $\Id-\gamma P$.
For $\gamma$ close to $1$, $(\Id-\gamma P)^{-1}$ has a very large
eigenvalue $1/(1-\gamma)$, 
then an eigenvalue of order $n^2/2\pi^2$, and the next eigenvalues behave
like $n^2/2k^2\pi^2$, thus decreasing like $1/k^2$. In this case, a
small-rank approximation is reasonable. A similar computation holds for
periodic grids $[1;n]^d$ in higher dimension.

How general is this example? The best studied case is for continuous-time
diffusions in continuous spaces such as a subset in $\R^d$. In continuous
time, the analogue of the operator $\Id-\gamma P$ is the
\emph{infinitesimal generator} operator of the continuous-time Markov
process. For the standard Brownian motion, this operator is the Laplacian
$\Delta=\sum_{i=1}^d \partial^2/\partial x_i^2$. Its inverse
$\Delta^{-1}$ plays the role of the successor state operator and provides the value function in
continuous time.
The spectrum of the Laplacian is
well-known and follows Weyl's law: there are about $k^{d/2}$
eigenvalues of size $\leq k$ \citep{wikipedia_weyllaw}. In particular, $\Delta$ has few small
eigenvalues and many large eigenvalues, so that the successor state
operator (given by $\Delta^{-1}$, which provides the value function in
continuous time) has few large eigenvalues and many small eigenvalues as
needed.

This applies not only to Brownian motion, but to basically any diffusion
with drift and variable coefficients on a subset of $\R^d$: indeed, in
this case the infinitesimal generator is an \emph{elliptic operator} and
also follows Weyl's law \citep{garding53_ellipticoperators}. The same law also holds for diffusions on
Riemannian manifolds, as the Riemannian Laplace operator also follows
Weyl's law \citep[Chapter 9.7.2]{berger_riem}. These continuous estimates are still valid when
discretizing the state space \citep{xu2017weyl}.
So this situation is quite general.

\subsection{The TD Updates for the FB Representation of $M$}
\label{sec:FBeq}

We now describe the explicit parametric TD updates for the FB
representation of successor states. These follow directly from the
general expressions for forward TD and backward TD.

However, the particular factorized structure gives rise to more variants:
pure forward (forward TD on $F$ and $B$), forward-backward (forward TD
update for $F$ but backward TD update for $B$), etc. These will lead to
slightly different fixed points and different dynamics for feature
learning, as we will explore later.

\newcommand{\DF}{D_F}
\newcommand{\DB}{D_B}
\newcommand{\SigmaF}{\Sigma_F}
\newcommand{\SigmaB}{\Sigma_B}

\begin{prop}[ (Successor state TD updates in the FB representation)]
\label{prop:FBupdates}
Consider the parameterization
$\tildem_\theta(s_1,s_2)=\transp{F_\thetaF(s_1)}B_{\thetaB}(s_2)$ of the
successor state operator $M$ where $F$ and $B$ are two functions from $S$
to $\R^r$, parameterized by $\theta=(\thetaF,\thetaB)$. Abbreviate
$F$ for $F_\thetaF$ and $B$ for $B_\thetaB$.

Then
the forward TD update
\eqref{eq:paramtdtilde} for $F$ is equal to
\begin{equation}
\label{eq:forwardF}
\E_{s\sim \rho}\, (\partial_{\thetaF} \transp{F(s)}) B(s)
+
\E_{s\sim \rho, s'\sim P(s,\d s')}\, (\partial_{\thetaF}
\transp{F(s)})\,
\SigmaB (\gamma F(s')-F(s))
\end{equation}
where $\SigmaB$ is the matrix
\begin{equation}
\SigmaB\deq 
\E_{s_2\sim \rho} \, B(s_2)\transp{B(s_2)}.
\end{equation}
The forward TD update for $B$ is equal to
\begin{equation}
\label{eq:forwardB}
\E_{s\sim \rho}\,(\partial_\thetaB \transp{B(s)})F(s)
+
\E_{s_2\sim \rho} \, (\partial_\thetaB \transp{B(s_2)})\,\DF B(s_2).
\end{equation}
where $\DF$ is the matrix
\begin{equation}
\DF\deq \E_{s\sim
\rho,\,s'\sim P(s,\d s')}\,
F(s)\transp{(\gamma F(s')-F(s))}.
\end{equation}
The backward TD update for $F$ is equal to
\begin{equation}
\E_{s\sim \rho}\, (\partial_\thetaF \transp{F(s)})B(s)
+
\E_{s_1\sim \rho}\, (\partial_\thetaF \transp{F(s_1)})\, \DB F(s_1),
\end{equation}
where $\DB$ is the matrix
\begin{equation}
\DB\deq
\E_{s\sim \rho,\,s'\sim P(s,\d s')} 
(\gamma B(s')-B(s))\transp{B(s)}.
\end{equation}
The backward TD update for $B$ is equal to
\begin{equation}
\E_{s\sim \rho} \,(\partial_\thetaB
\transp{B(s)})F(s)+
\E_{s\sim \rho,\,s'\sim P(s,\d s')} \left(
\gamma\,\partial_\thetaB \transp{B(s')}
-
\partial_\thetaB \transp{B(s)}
\right)
\SigmaF B(s)
\end{equation}
where $\SigmaF$ is the matrix
\begin{equation}
\SigmaF\deq
\E_{s_1\sim \rho} \,F(s_1)\transp{F(s_1)}.
\end{equation}
\todo{invariant measure case?}
\todo{explain Jacobian matrices}
\end{prop}

Proposition~\ref{prop:tabularFB} (Appendix~\ref{app:FB}) describes how
these updates play out for finite spaces in the ``tabular on FB''
setting, in which a value of $F$ and $B$ is maintained for each state.

Forward or backward TD may be used separately for $F$ and $B$, giving
rise to four algorithms: forward on $F$ and forward on $B$
(ff-FB), forward on $F$ and backward on $B$ (fb-FB),
backward on $F$ and forward on $B$ (bf-FB), and backward on $F$
and backward on $B$ (bb-FB). These algorithms behave quite
differently on how they learn features and on the fixed points obtained, as discussed below.

\paragraph{Consequences for sampling and variance.} A key feature of the
FB updates is their decomposition as a product of an
expectation over a transition $s\to s'$, times an expectation over
another independent state $s_2$ or $s_1$.

This has important consequences
algorithmically for variance reduction via minibatching. Indeed, a
natural way to sample these updates would be to sample a minibatch of
transitions $s\to s'$, another minibatch of states $s_2$, and evaluate
\eqref{eq:forwardF} on the minibatch $s\to s'$ with the value of
$\SigmaB$ obtained on the minibatch $s_2$. This would not be possible for
other parameterizations of $M$: in general, \eqref{eq:paramtdtilde} would
require to compute a separate quantity for each $(s\to s',s_2)$, thus
requiring smaller minibatches in practice.

Furthermore, these updates lend themselves to a purely
trajectory-wise online
estimation, without even sampling another independent state $s_2$ or
$s_1$: indeed, \eqref{eq:forwardF} can be estimated at the current
transition $s\to s'$, while the matrices $\SigmaF$ etc., may be estimated
online by an exponential moving average over past or recent states.

\paragraph{Fixed points, feature learning.} 
With $F$ and $B$ of dimension $r$, each of the $r$ components of $F$ and
$B$ defines a function $F_i(s)$ or $B_i(s)$ on the state space. We call
these functions \emph{features}. The features of $F$ provide a basis for
approximating the value function $V$. In addition, the model for $V$
ignores any part of the reward function that is uncorrelated to the
features of $B$.

More precisely, the kernel of $B$ and the image of
$\transp{F}$ directly encode which features of states are ignored.
Namely, if $R\in \Ker B\rho$ then the corresponding value function is estimated to
$0$: $MR=\transp{F}B\rho R=0$. Thus $\Ker B\rho$ encodes the subspace of reward
functions that is unseen by the model. $\Ker B\rho$ is exactly the space
of functions which are $L^2(\rho)$-orthogonal to all the features in $B$.
Likewise, for any reward function $R$, the approximate value
function is $\transp{F}B\rho R=0$ which lies inside $\Img \transp{F}$:
thus $\Img \transp{F}$ is the space of features used to express
the value functions.

The four algorithms ff-FB, fb-FB, bf-FB, and bb-FB greatly differ on how
new features are learned:
\begin{itemize}
\item ff-FB learns new features by applying the operator
$P$ to existing features in $F$. These new features are put into both $F$
and $B$.
The fixed points of ff-FB correspond to eigenvectors of the matrix $P$
and $M$ (Proposition~\ref{prop:ffFB}).
\item bb-FB learns new features by applying the operator $\transp{P}$ to
existing features in $B$, and putting them into both $F$ and $B$.
\footnote{The action of $\transp{P}$ on a positive vector $v$ corresponds to the
law of a state at time $t+1$ if the state at time $t$ is distributed
according to $v$. Thus, $\transp{P}$ naturally acts on probability
distributions over states.}  The fixed points of bb-FB correspond to
eigen-probability densities of $P$ and $M$ (Remark~\ref{rem:bbFB},
Proposition~\ref{prop:bbFB}).
\item fb-FB learns new features both by applying $P$ to features in $F$,
and $\transp{P}$ to features in $B$. The fixed points of fb-FB are the
\emph{rank-$r$ truncated SVD decompositions} of the matrix $M$.
\item bf-FB may not learn any features beyond the initialization of $F$
and $B$. For any subspace of
features, there is a fixed point of bf-FB which lies in that subspace.
\end{itemize}

We refer to
Appendix~\ref{app:FB} for precise statements of these properties.
In addition, fb-FB and bf-FB preserve the symmetry with respect to time
reversal of the process, while ff-FB and bb-FB do not.

\todo{These differences in behavior are clearly visible on the rank-one case.
INCLUDE?}

\paragraph{Relationship with successor representation learning and with
linear TD with learned features.} To help interpreting these relations, we will relate them to
objects from the literature. We make two claims. First, for fixed and
orthonormal $B$, the forward update of $F$ corresponds to standard
successor representation learning with state representation (features)
$B$.  Second, for fixed $F$, the forward update of $B$ corresponds to
learning the value function for every target state via linear TD with
fixed features $F$.

To simplify things, in this paragraph we consider
the ``tabular-FB'' setting, in which $F$ and $B$ are parameterized just
by listing the value of $F(s)$ and $B(s)$ on every state $s$, assuming a
finite state space. \footnote{This is different from a tabular setting
for $M$, which would parameterize $M$ by listing the values $M(s_1,s_2)$
for every pair of states.} For instance, the forward TD update
\eqref{eq:forwardF} for $F$, with learning rate $\eta>0$, becomes
$F(s)\gets F(s)+\eta \del F(s)$ where
\begin{equation}
\label{eq:forwardF_tab}
\del F(s)=B(s)+\SigmaB (\gamma F(s')-F(s))
\end{equation}
upon sampling a transition $s\to s'$.

If $B$ is a fixed, $L^2(\rho)$-orthonormal collection of
feature functions (namely, if $\SigmaB=\Id$), then this forward TD
equation to learn $F$ is identical to standard deep successor
representation learning using $B$ as the representation.
Indeed, standard deep successor representation learning
\citep{kulkarni2016deep} starts with given features $\phi(s)$
on the state space, and learns the successor features $m$ as the
expected discounted future value of $\phi$ along a trajectory $(s_t)$:
$m(s)=\sum_{t\geq 0} \gamma^t
\E[\phi(s_t)|s_0=s]$. Such an $m$ is the fixed point of the Bellman equation
$m=\phi+\gamma Pm$. Via identifying $m=F$ and $\phi=B$, ordinary TD for this Bellman equation is equivalent
to \eqref{eq:forwardF_tab} when $\SigmaB=\Id$. However, this is not the
case if $\SigmaB\neq \Id$. This is because scalings are different: With
the successor state operator, if $B$
is doubled, then $F$ is halved so that $M=\transp{F}B$ is fixed. With
successor representations, if the state representation $\phi$ is doubled,
then $m$ is doubled as well.

Next, if $F$ is fixed, we claim that the forward TD update for $B$
corresponds to linear TD to learn all value functions corresponding to
individual reward $\1_{s_2}$ at all target states $s_2$, with $F$ as the
feature basis. Indeed, if the reward function is $R=\1_{s_2}$, and the
corresponding value function is represented as $V=\transp{F}w$ for some
learned vector $w$ (this is linear TD with feature basis $F$), then the TD update
of $w$ when observing a transition $s\to s'$ is
$
\del w=F(s)\left(\1_{s=s_2}+\gamma \transp{F(s')}w-\transp{F(s)}w\right).
$
Of course, the resulting $w$ depends on $s_2$. Thus, if we learn a vector
$w(s_2)$ this way for every $s_2$, we get an update $\del
w(s_2)=F(s)\left(\1_{s=s_2}+\gamma
\transp{F(s')}w(s_2)-\transp{F(s)}w(s_2)\right)$.
By identifying $w(s_2)$ and $B(s_2)$, then
on expectation over $s$ and $s_2$
sampled from $\rho$, this is equal to \eqref{eq:forwardB} for tabular
$B$.
Thus, for fixed $F$, the forward TD update
\eqref{eq:forwardB} just puts into every $B(s_2)$ a representation
of the value function for reward $\1_{s_2}$ as a linear combination of
the features $F(s)$.

Thus, when learning $F$ and $B$ jointly, the ``FB-tabular'' forward TD
update on $F$ and $B$ can be seen as a simultaneous learning of all value
functions for all reward $\1_{s_2}$, by linear TD in a \emph{learned}
feature basis $F$.

\todo{guarantees? When reversible, ff-FB is gradient descent of the mixed
Dirichlet$\otimes L^2(\rho)$ loss (check). Moreover in that case, the FB
factorization does not introduce any new local minima (\todo{check that
this theorem extends to the mixed Dirichlet loss}), si if both F and B
are overparameterized and the policy is reversible, then we get to the
globally best rank-$r$ approx. But in general, even linear TD diverges,
so even fixed $F$ could diverge} 

\section{Second-Order Methods for Successor States: Implicit Process
Estimation and Bellman--Newton}
\label{sec:BN}

We now turn to more complex, ``second-order'' algorithms for estimating
successor states and value functions. First, we study the 
best online estimate of $M$ and $V$ in the tabular case, obtained by
directly estimating the transition matrix and reward function, and
exactly solving the Bellman equation in this estimated process. We
provide a convergence theorem for this method
(Theorem~\ref{thm:convergence}). 

This provides an explicit online evolution equation for $M$ and $V$ from
observed transitions, in which the transition matrix does not appear
(\emph{successor states via implicit process estimation},
Theorem~\ref{thm:onlineM}). Interestingly, this ``true'' update of $V$ is
TD preconditioned by $M$ (Theorem~\ref{thm:onlineM_exp}). This is
related to viewing $M_{s_1s}$ as an expectation of the eligibility trace
at state $s$ (Appendix~\ref{sec:traces}).

The resulting ``true'' update of $M$, taken in expectation, defines a
\emph{Bellman--Newton operator} (Definition~\ref{def:BN}), so called
because it corresponds exactly to the Newton method for inverting the
matrix $M$. Intuitively, this operator proceeds by concatenating known
paths of the MDP, thus doubling the length of known paths, while TD and
backward TD just add one transition to the set of known paths (see
intuition in Section~\ref{sec:paths}). This intuition is formalized in
several ways (Proposition~\ref{prop:doublingpaths},
Appendix~\ref{sec:bellman-newton}). This also translates as much better
asymptotic convergence in the continuous-time limit
(Section~\ref{sec:BNconttime}).

All these properties are exact analogues of the convergence properties of
second-order Newton-like methods compared to simple first-order gradient
descent. Thus, online estimation of $M$ and the Bellman--Newton operator can be seen as
``second-order'' TD algorithms. Accordingly, they are also numerically
trickier. Strengths and weaknesses are discussed in
Section~\ref{sec:BNdiscussion}.

Finally, we derive the parametric version of the Bellman--Newton operator,
extending it beyond full-matrix tabular updates to sampling in arbitrary
state spaces. However, this update has a large variance unless some kind
of forward-backward (FB) representation is used.

\subsection{Estimating a Markov Process Online}
\label{sec:SSIPE}

We now introduce estimates of $M$ and $V$ by online estimation of the Markov
process, first in the tabular case, then via function approximation.
The process estimation is \emph{implicit}: it 
does not appear in the
resulting algorithms for $M$. (In particular, we never store an estimated
transition matrix $\hat P$, which would not make sense for continuous
spaces; this excludes solving the problem by planning via the model
$\hat P$.)

In a (small) finite
state space, an
obvious approach to learn $M$ is to first learn an estimate $(\hat P,\hat
R)$ of the transition matrix $P$ and reward vector $R$ of the Markov reward process, by
direct empirical averages;
then set $M$ and $V$ to their true values in the estimated Markov
process, namely, $\hat M=\sum_{n\geq 0} \gamma^n \hat P^n=(\Id-\gamma
\hat P)^{-1}$ and $\hat
V=\hat M\hat R$.

The empirical averages $\hat P$ and $\hat R$ are updated for each new
transition $s\to s'$ with reward $r_s$, by updating the row $s$
of the transition matrix $\hat P$, and the value $\hat R_s$ at $s$:
\begin{align}
\label{eq:Pupdate}
\hat P_{ss_2} &\gets (1-1/n_s)\hat P_{ss_2} + (1/n_s) \1_{s_2=s'}\quad
\forall s_2,
\qquad
\hat R_s \gets (1-1/n_s)\hat R_s+(1/n_s) r_s
\end{align}
with $n_s$ the number of visits to state $s$ up to time $t$.
The initialization of $\hat P$ and $\hat R$ is forgotten after the first
observation at each state ($n_s=1$), but to fix ideas we initialize to
$\hat P=\hat R=0$.
The corresponding estimates $\hat M=(\Id -\gamma \hat P)^{-1}$ and
$\hat V=\hat M \hat R$ converge to their true values, as shown by the
following non-asymptotic bound.


\begin{thm}[ (Convergence bounds for process estimation)]
   \label{thm:convergence}
   Consider a finite Markov reward process with $S$ states and $E$ edges ($(s, s')$ is an edge if
  $P_{s s'} > 0$), rewards almost surely bounded by $R_\mathrm{max}$, and
  stationary distribution $\rho$. Update $\hat P$ and $\hat R$ online via
  \eqref{eq:Pupdate}, initialized to $\hat P=\hat R=0$.

  Then after $t$ i.i.d.\ observations
  $(s \mathbin{\sim} \rho,\, s'\mathbin{\sim} P_{ss'})$, with probability $1-\delta$,
  the estimates $\hat M=(\Id -\gamma \hat P)^{-1}$ and $\hat V= \hat M
  \hat R$ satisfy
    \begin{equation}
    \|\hat M - M\|_{\rho, \mathrm{TV}} \leq
\frac{2\gamma}{(1-\gamma)^{2}}\sqrt{\frac{2 E}{t} \log\frac{2}{\delta}}
\end{equation}
and
\begin{equation}
\sum_s \rho(s)\abs{\hat V(s) - V(s)} \leq \frac{3
R_{\mathrm{max}}}{(1-\gamma)^{2}}\sqrt{\frac{2E}{t}
\log\frac{4S}{\delta}}.
\end{equation}
\end{thm}

These bounds do not depend on the sampling measure $\rho$,
although the norm used to define the error does. Thus, rarely visited
points have no impact on these bounds. \todo{contrast with existing
theorems in $L^2$ norm, explain $L^1/L^2$...}

Direct matrix inversion is inconvenient. But since
\eqref{eq:Pupdate} is a rank-one update of the matrix $\hat P$, one can compute the
update of $\hat M$ resulting from \eqref{eq:Pupdate}; this update does not
explicitly involve $\hat P$ anymore. This will form the basis for the parametric
version.

We call the resulting algorithm \emph{successor states via implicit
process estimation} (SSIPE).

\begin{thm}[ (SSIPE: Tabular online update of $M$)]
\label{thm:onlineM}
When a transition $s\to s'$
is added to an empirical
estimate of a Markov reward process via
\eqref{eq:Pupdate},
the successor state matrix $\hat M$
of the estimated process becomes
$\hat M\gets \hat M+\del M$ 
with
\begin{alignat}{2}
\label{eq:onlineM}
& \del M_{s_1s_2}=\frac{1}{n_s}
\,\hat M_{s_1 s} \,\frac{
\1_{s_2=s} 
+\gamma \hat M_{s's_2}-\hat M_{ss_2}
}
{1-\frac{1}{n_s}(
\gamma \hat M_{s' s}-\hat M_{ss}+1)
}
&\qquad  \forall s_1,s_2
\end{alignat}
with $n_s$ the number of times state $s$ has been sampled.
The estimated value
function $\hat V$ becomes
$\hat V\gets \hat V+\del V$ with
\begin{alignat}{2}
\label{eq:onlineV}
\del V_{s_1} =\tfrac{1}{n_s}(r_s+\gamma \hat V_{s'}-\hat V_s) \,\hat
M_{s_1s}+o(1/n_s)
& \qquad  \forall s_1
\end{alignat}
where $r_s$ is the observed reward.
\end{thm}


This describes the ``true'' change of $M$ 
when the Markov process
changes by increasing $P_{ss'}$.
This update contains a two-sided term $M_{s_1s}M_{s's_2}$: in terms
of paths, this term combines all known paths from $s_1$ to $s$, the transition
$s\to s'$, then all known paths from $s'$ to $s_2$
(Fig.~\ref{fig:combining_paths} and Appendix~\ref{sec:bellman-newton}).

The update of $V$ has the form $\delta V=M\cdot(\text{Bellman gap at
}$s$)$.  The matrix $M$ can be seen as a ``credit assignment''
to transfer the Bellman gap $R+\gamma PV-V$ observed at a state $s$ to
``predecessor'' states. 

The update \eqref{eq:onlineM} of $M$ is also its TD update multiplied on the left by
$M$ (compare \eqref{eq:onlineM} and \eqref{eq:tabularTD}). This is most
clear when taking expectations over the next transition $s\to s'$, as
follows. \todo{put this first?}

\begin{thm}[ (The true change of $M$ and $V$ is TD preconditioned by $M$)]
\label{thm:onlineM_exp}

Estimate the successor matrix and value function of a finite MRP by
$\hat
M=(\Id -\gamma \hat P)^{-1}$ and $\hat V=
\hat M \hat R$ where $\hat P$ and $\hat R$ are estimated directly by
the empirical averages
\eqref{eq:Pupdate}.

Consider the updates of these estimates after observing a new transition
$s\to s'$.
Then, in expectation over the transition $s\to s'$ sampled at time $t$,
the update \eqref{eq:onlineM} of $\hat M$ is equal to
\begin{equation}
\label{eq:EonlineM}
\E_{s \mathbin{\sim} \rho,\, s'\mathbin{\sim} P_{ss'}} [\del M]=\tfrac1t
\hat M
(\Id+\gamma P \hat M-\hat M)+o(1/t)
\end{equation}
when the number of observations $t$ tends to infinity. The resulting update of the value function is $\hat V\gets \hat V+\del V$ with
\begin{equation}
\label{eq:EonlineV}
\E_{s \mathbin{\sim} \rho,\, s'\mathbin{\sim} P_{ss'}} [\del V]=\tfrac1t
\hat M(R+\gamma P\hat V-\hat V)+o(1/t).
\end{equation}
\end{thm}

Thus, preconditioning the TD update by $M$ itself produces an update that
tracks the ``true'' value of $M$ and $V$ given all observations
available so far. The learning rate $1/t$ is inherited from the direct
estimate of $P$ and $R$ via empirical averages in \eqref{eq:Pupdate}.

This update of $V$ is consistent with the view of $M$ as an expected
eligibility trace (Appendix~\ref{sec:traces}).  Indeed, eligibility
traces also update the value function at states $s_1$ that are connected
to $s$ via a trajectory.  Actually, in expectation, these updates are the
same: with $\lambda=1$, the eligibility trace vector at a state $s$ is an
unbiased estimator of the column $M_{s_1s}$ (Theorem~\ref{thm:elig} in
Appendix~\ref{sec:traces}).  From this viewpoint, learning $M$ via a
parametric model, or using TD($1$), are both ways of estimating the
``predecessor states'' of a state $s$. Eligibility traces are unbiased
but can have large variance, while the model of $M$ has no variance
but may have bias if not learned well.

Such a preconditioning is analogous to second-order methods in
optimization using the inverse Hessian, which directly jump to the the
location of the new optimum when one more data point becomes available.
However, in second-order methods, the preconditioning matrix is symmetric
definite positive, while this is not the case here; this can produce
numerical problems.

In small-scale experiments, using the full matrix online update of $\hat
M$ resulted in much faster convergence of the value function than TD,
consistently with the theoretical prediction of
Section~\ref{sec:BNconttime}.
But with this
method,
each update requires $O(\abs{S}^2)$ computation time. This is
useful only if sample efficiency is the main concern.
\todo{Keep?say more?}

\subsection{The Bellman--Newton Operator}
\label{sec:BNop}

Thus, when estimating a Markov process online, in expectation, each new observation replaces the estimate $\hat M$ with $\hat M+\E[\del
M]=\hat M(1+\frac1t)-\frac1t \hat M(\Id-\gamma P)\hat M+o(1/t)$
by \eqref{eq:EonlineM}.
Interestingly, this expected update does not depend on the distribution
$\rho$ of sampled states $s$. This is because the $1/n_s$ factors behave
asymptotically like $1/(t\rho_s)$, thus compensating the sampling
probabilities $\rho_s$. The fluctuations between $n_s$ and $t\rho_s$ are
absorbed in the $o(1/t)$ terms. We gather this behavior in the following
operator.

\begin{defi}[ (Bellman--Newton operator)]
\label{def:BN}
We call \emph{Bellman--Newton operator} with learning rate $\eta>0$ the
operator $M\mapsto M(1+\eta)-\eta M(\Id-\gamma P)M$.
\end{defi}

The reason for the name is the following: With learning rate $\eta=1$,
this operator is $M\mapsto 2M-M(\Id-\gamma P)M$.
Inverting a matrix $A$ by iterating 
$M\gets 
2M-MAM$ is the \emph{Newton method} for matrix inversion,
going as far back as 1933
\citep{pan1991improved}. The Newton method has superexponential
convergence, squaring the error (doubling precision) at each step. This
property translates as follows in our context.

In terms of paths, the quadratic term in $M$
realizes the path concatenation operation in
Fig.~\ref{fig:combining_paths}. This is formalized as follows, and proved
and further discussed in Appendix~\ref{sec:bellman-newton}. In contrast,
forward and backward TD only increase the length of known paths by $1$.

\begin{prop}[ (Bellman--Newton doubles the length of known paths)]
\label{prop:doublingpaths}
Assumes that $M$ represents exactly the successor states up to $k$ steps,
namely, $M=\sum_{i=0}^k \gamma^i P^i$ (as matrices or as operators). Then after
one step of the Bellman--Newton operator with learning rate $\eta=1$, $M$
represents exactly the successor states up to $2k+1$ steps, namely,
$2M-M(\Id-\gamma P) M=\sum_{i=0}^{2k+1} \gamma^i P^i$.
\end{prop}

Unfortunately, this method does not always converge. In particular, it is
initialization-dependent. For instance, the initialization $M=0$ is a
fixed point. In general, the Bellman--Newton operator preserves the kernel and image of
$M$, so there are many fixed points. Still, $M=(\Id-\gamma P)^{-1}$ is
the only full-rank fixed point.

Convergence of the Newton method for matrix inversion is quite well
understood \citep{pan1991improved} and works if the spectral radius
of $\Id-AM$ is less than $1$ at initialization. 
Otherwise, the method can
diverge. For instance,  $A=\Id-\gamma P$ for successor states, so initializing to $M=\Id$
converges.

Learning rates $\eta\ll 1$ improve convergence properties. In
Section~\ref{sec:BNconttime} we study convergence with infinitesimal learning
rates, proving much faster asymptotic convergence than with simple
TD on $M$. This is analogous to the faster convergence of second-order
methods with respect to simple gradient descent. Even with $\eta\ll 1$,
some initializations still diverge; however, if the
initialization is of the form $M=(\Id-\gamma P_0)^{-1}$ for some stochastic
or substochastic matrix $P_0$ (e.g., $P_0=0$, initializing $M$ to $\Id$) then
the infinitesimal learning rate version converges.

\paragraph{Sampled Bellman--Newton update.} Like the Bellman operator for
TD on $M$, the Bellman--Newton operator lends itself to sampling the
states at which the values are updated.

This works out as follows. Assume that $S$ is discrete so that $M$ is a
matrix (we deal with the parametric case in the next section). Let as
usual $\rho$ be the probability distribution from which states are
sampled, and let $\diagrho$ be the matrix with diagonal entries $\rho$.
Set $\tildem\deq M\diagrho^{-1}$ (this corresponds to the
parameterization $\tildem$ in \eqref{eq:Mmodeltilde}). Then the
Bellman--Newton update is equivalent to $\tildem \mapsto \tildem
(1+\eps)-\eta \tildem (\diagrho-\gamma \diagrho P)\tildem$. In
expectation, this update can be realized by sampling a state $s\sim
\rho$ and a transition $s'\sim P(\d s'|s)$. Indeed, in that case we have
$\E \1_s \transp{\1}_{s'}=\diagrho P$ and $\E \1_s
\transp{\1}_{s}=\diagrho$, and therefore the update
\begin{equation}
\tildem_{s_1s_2}\gets (1+\eta)\tildem_{s_1s_2}-\eta\,
\tildem_{s_1s}\,\tildem_{ss_2}+\eta\,\gamma\,
\tildem_{s_1s}\,\tildem_{s's_2}\qquad \forall s_1,s_2
\end{equation}
is equal to the Bellman--Newton update in expectation over $(s,s')$.
\footnote{This is not quite equivalent to
the online update \eqref{eq:onlineM}: using $n_s\approx
t\rho_s$, the latter yields $\tildem_{s_1s_2}\gets
\tildem_{s_1s_2}+\eta \,\tildem_{s_1s}(\1_{s=s_2}/\rho_s -
\tildem_{ss_2}+\gamma\,
\tildem_{s's_2})+o(\eta)$ with $\eta=1/t$. This difference disappears after
taking expectations over $s_1$ and $s_2$ in addition to $(s,s')$.}

This is still a full-matrix update: the value $\tildem_{s_1s_2}$ is
updated for every $s_1$ and $s_2$, even if $(s,s')$ is sampled. This is
not scalable. It is possible to sample the states $s_1$ and $s_2$ from
$\rho$ as well: with this option, the expectation of the update is
multiplied by $\diagrho$ on the left and right.

\subsection{Parametric Bellman--Newton Update}
\label{sec:paramBN}

Perhaps surprisingly, the full-matrix tabular update of $M$ lends itself well to a
parametric version, by following the standard TD strategy of updating the
parameter to bring $M$ closer to its new value.

\begin{thm}[ (Bellman--Newton update with
function approximation)]
\label{thm:paramBN}
Maintain a parametric model of $M$ via $m_{\theta_t}$ or
$\tildem_{\theta_t}$ as in
Section~\ref{sec:M-func-approx},
with $\theta_t$
the parameter at step $t$.

Let $s\to s'$ be the transition in the Markov process observed at step
$t$, with reward $r_s$.
Define a target update of $M$ by $M^\tar\deq M_{\theta_t}+\del M$ with $\del M$
given by the online tabular estimate \eqref{eq:onlineM}.
Define the loss between $M$ and $M^\tar$ via $J(\theta)\deq \frac12
\norm{M_\theta-M^\tar}_\rho^2$ using the norm \eqref{eq:norm}.

Then the
gradient step on $\theta$ to reduce this loss is
\begin{multline}
\label{eq:Mparam}
-\partial_\theta J(\theta)_{|\theta=\theta_t}=
{\textstyle \frac1t}\, \E_{ s_1\sim \rho,\,s_2\sim\rho}
\left[
\gamma\,\partial_\theta m_{\theta_t}(s,s')
+
\gamma \, m_{\theta_t}(s_1,s)\,
\partial_\theta m_{\theta_t}(s_1,s')
\right.
\\+\left.
(\gamma m_{\theta_t}(s',s_2)-m_{\theta_t}(s,s_2))
\left(\partial_\theta m_{\theta_t}(s,s_2)+
m_{\theta_t}(s_1,s)
\,\partial_\theta m_{\theta_t}(s_1,s_2)\right)
\right]+o(1/t)
\end{multline}
for the model \eqref{eq:Mmodel} using $m_\theta$, and
\begin{multline}
\label{eq:Mtildeparam}
-\partial_\theta J(\theta)_{|\theta=\theta_t}=
{\textstyle \frac1t}\, \E_{ s_1\sim \rho,\,s_2\sim\rho}
\left[
\tildem_{\theta_t}(s_1,s)\,
\partial_\theta \tildem_{\theta_t}(s_1,s)
\right.
\\+\left.
\tildem_{\theta_t}(s_1,s)(\gamma \tildem_{\theta_t}(s',s_2)-\tildem_{\theta_t}(s,s_2))
\,\partial_\theta \tildem_{\theta_t}(s_1,s_2)
\right]+o(1/t)
\end{multline}
for the model \eqref{eq:Mmodeltilde} using $\tildem_\theta$.\todo{proof
not included}
\end{thm}

The update of $V$ via $M$ is discussed in Section~\ref{sec:V}.

Here the learning rate $1/t$ is inherited from the direct
estimate of $P$ via empirical averages, but can be replaced with any
learning rate. As with TD, the update was derived from a tabular update,
but makes sense in continuous state spaces. In particular, the parametric
gradient does not involve the state counts $n_s$ from
\eqref{eq:onlineM}: a cancellation occurs because $n_s\sim t \rho_s$ when $t\to\infty$.

Implementing this update requires sampling \emph{two} additional states
$s_1$ and $s_2$ from the dataset, in addition to the transition $s\to
s'$.
See the discussion after Theorem~\ref{thm:paramtd} for possible ways to
sample these
additional states.
\option{TD for $M$ required only one: this reflects the full matrix
update \eqref{eq:onlineM}, while TD only updates the $s$ row of $M$ when
observing a new transition $s \to s'$ (Eq.~\ref{eq:tabularTD}).}

For parametric Bellman--Newton, the model $m_\theta$ can be initialized
to $0$ while the model $\tildem_\theta$ cannot. Indeed, setting
$m_\theta$ to $0$ corresponds to setting $M$ to $\Id$, a valid
initialization for the Bellman--Newton operator, while setting
$\tildem_\theta$ to $0$ corresponds to setting $M$ to $0$, an unwanted
and unstable
fixed point of the Bellman--Newton operator.

\subsection{Discussion: strengths and weaknesses of second-order
approaches}
\label{sec:BNdiscussion}

In a tabular setting, the full-matrix online update
\eqref{eq:onlineM} of $M$ (where a transition $s\to s'$ is sampled, but
with the value $M_{s_1s_2}$ updated for every state $s_1$ and $s_2$)
converges much faster than TD to compute the value function, empirically.
This is in line with the asymptotic convergence properties
of the Bellman--Newton versus ordinary Bellman operator
(Section~\ref{sec:BNconttime}).

However, this results in an $O(\abs{S}^2)$ cost per time step, so it is
only interesting if sample efficiency is the main issue. The alternative
is to sample a few states $s_1$ and $s_2$ and only update $M_{s_1s_2}$
for those states.  But in practice, we have found that this introduces
many instabilities and requires reducing the learning rate so much
(typically $\eta$ smaller than $1/\abs{S}^2$) that the benefit of
second-order Newton convergence is lost. The same phenomenon is observed
for the parametric version of Theorem~\ref{thm:paramBN}.

This sampling issue can be avoided if using a factorized representation
$M=\transp{F}B$ as in Section~\ref{sec:FB}. Namely, there exists an
update of $F$ and $B$ that is compatible with sampling and that reproduces the
Bellman--Newton update (Section~\ref{sec:BNFB}). This decouples the sampling of states $s_1$ and
$s_2$, thus reducing variance and allowing for larger learning rates.
However, this also exacerbates another issue of the Bellman--Newton
update, namely, the existence of non-full-rank fixed points and the
preservation of the kernel and image of $M$. The representation
$M=\transp{F}B$ is usually not full-rank, and the Bellman--Newton update
of Section~\ref{sec:BNFB} preserves the kernels of $F$ and $B$. As a
consequence (at least for uniform $\rho$), this algorithm computes the
inverse of $\Id-\gamma P$ in the subspace spanned by the initializations
of $F$ and $B$, but no features are learned. Currently, we have found no
fully satisfactory second-order update beyond the full-matrix update
\eqref{eq:onlineM}.

\section{Learning Value Functions and Policies via Successor States}
\label{sec:V}
\label{sec:Vdetails}

There are many possible ways to use a model $M$ of the successor state
operator in policy and
and value function learning. Choices include:
\begin{itemize}
\item Using policy gradient versus using $Q$-learning
(greedy or Boltzmann policies, DDPG...).

If the reward is a known goal state, we may directly use the optimal
goal-dependent $Q$ function of Section~\ref{sec:optQ}.

For $Q$-learning with other types of rewards, the successor state
operator can be defined on the Markov process over state-action pairs (as
explained in Section~\ref{sec:notation}). The $Q$ function can be
computed from this ``successor state-action operator'' in the same ways as
the $V$ function from the successor state operator.  Thus, all methods
described below to
learn $V$ can be extended to $Q$, and we do not discuss this option
further here.

\item Using the goal-dependent value function
as in Section~\ref{sec:goals} (this leads to a
goal-dependent policy for every goal state, simultaneously for all single-state
reward functions), versus using the successor state operator of a
single policy as in Section~\ref{sec:td-alg} (this works for dense rewards but with a single policy).

\item Using the successor state operator directly in the policy gradient
formula without a value function model, versus learning a model of $V$ from
successor states, then using this model normally in policy gradient.

\item If learning a model of $V$ from successor states, there are several
options to do so. First, the FB representation of $M$ directly
yields a $V$ function. Second, the $V$ function may be learned from $M$ in a
supervised way based on $V=MR$. Third,
$M$ may be used only as one component of the value function
($V=MR+v_\phi$ with $v_\phi$ learned via TD), or as an
initialization. This is presumably better if $M$ is approximate. Fourth, $V$ may be learned via TD ``preconditioned by
$M$'',  based on the formula \eqref{eq:EonlineV} for the true change of $V$ when
new transitions are observed (Theorem~\ref{thm:onlineM_exp}). \todo{provide links for each?}
\end{itemize}

We now describe these options in greater detail.  They have different
bias-variance trade-offs, and the best option may differ based form case
to case.

We recall the general form of the policy gradient estimator for a
parametric policy $\pi$
\begin{equation}
\label{eq:policygradient}
\del \pi \deq \E_{s\sim \rho} \E_{a\sim \pi(a|s)} \left[(\partial \ln
\pi(a|s)) \,\E_{s'\sim P(\d s'|s,a)} \left[r_{s,a}+\gamma
V(s')-b(s)\right]\right]
\end{equation}
where $\partial \ln
\pi(a|s)$ is the derivative with respect to the policy parameters of the
log-probability to select action $a$, where $r_{s,a}$ is the immediate reward
received after action $a$, and where $b$ is an arbitrary baseline
function which reduces variance of the estimator (typically $b(s)=V(s)$
so that $r_{s,a}+\gamma
V(s')-b(s)$ is centered, but we will see other choices below).

\paragraph{Learning goal-dependent policies.} The
simplest case is for learning policies to reach arbitrary target
states, using the goal-dependent value function $v_\theta(s,g)$ of
Section~\ref{sec:goaldependentV}. Here $g$ represents a variable goal, such as 
a target state, or a desired value for some function of states
(Section~\ref{sec:goaldependentV}).

This works with a goal-dependent policy $\pi(a|s,g)$ depending on goal
$g$, and
leads to the policy gradient
update
\begin{equation}
\label{eq:goalpolicygradient}
\del \pi=\E_{(s,g)\sim \rho_{SG},\,a\sim \pi(a|s,g),s'|P(\d s'|s,a)}
(\partial \ln \pi(a|s,g))(\gamma v_\theta(s',g)-b(s,g))
\end{equation}
where $v_\theta$ is the goal-dependent value function model from
Section~\ref{sec:goaldependentV},
where $b$ is an arbitrary baseline function (such as
$b(s,g)=v_\theta(s,g)$), and
where $\rho_{SG}$ is the empirical distribution of state-goal pairs in
the trajectories
in the dataset (typically obtained by choosing a goal and following the associated
policy for some time).

A few comments on this formula: First, with goal states, the
reward $r_{s,a}$ in \eqref{eq:policygradient} is a Dirac mass, but it
depends only on the previous state, not on $a$ or $s'$; so by choosing
the baseline $b$ to include this Dirac, this term disappears in
\eqref{eq:goalpolicygradient}.

Second, in the formalism of Section~\ref{sec:goaldependentV}, the value
function is formally a measure over goals, $V_\theta(s,\d
g)=v_\theta(s,\d g)\rho_G(\d g)$. Thus, the policy gradient update
\eqref{eq:policygradient} is goal-dependent and is itself
a measure over goals $g$. This measure can be integrated over all goals $g$; for
each $g$ we may choose the distribution $s\sim \rho_{SG}(\d s|g)$ of states given this
goal. This is how we obtain the policy update
\eqref{eq:goalpolicygradient} from \eqref{eq:policygradient}. In the computation, the measures cancel out between $\rho_{SG}(\d s|g)$ and the $\rho_G(\d
g)$ appearing in $V_\theta$: this results in just $v_\theta$ in
\eqref{eq:goalpolicygradient}, and in the sampling of a pair $(s,g)$
from $\rho_{SG(s,g)}$.

\paragraph{Learning $V$ from $M$.} Another option is to learn the value
function $V$ using $M$, then just use the value function via ordinary
policy gradient. We now consider the case of a single
(non-goal-dependent) policy to be learned, with an arbitrary reward
function.
There are several options again. 

\begin{itemize}
\item The FB representation of
Section~\ref{sec:FB} directly provides a representation of the value
function as
\begin{equation}
V(s)\approx \transp{F(s)}B(R),\qquad B(R)\deq \E_{s\sim \rho} [r_s\,B(s)]
\end{equation}
where $B(R)$ is a ``representation of the reward'', which can be sampled
by weighting the representation $B(s)$ of states by their 
reward. Thus $B(R)$ can be estimated online. Then the value of $V$ can be
plugged directly in the policy gradient formula \eqref{eq:policygradient}.

Since the FB representation will focus on low frequencies (long-range)
features, it might be useful to used a ``mixed'' model for $V$, with
$\transp{F(s)}B(R)$ as one component, and another component learned via
ordinary TD; see \eqref{eq:FBcomponent} below.

\item Another case is if the reward is located at a single known target state
$g$. Then $V(s)=\tilde m(s,g)$ and the policy gradient
\eqref{eq:policygradient} is equal to
\begin{equation}
\del \pi=\E_{s\sim \rho} \E_{a\sim \pi(a|s)} \left[(\partial \ln
\pi(a|s)) \,\E_{s'\sim P(\d s'|s,a)} \left[
\gamma \tildem(s',g)-b(s)\right]\right]
\end{equation}
(once more, the reward term $r_{s,a}$ does not depend on $a$ in that case
and can be absorbed in the baseline $b$). This assumes the model
$\tildem$ is used for $M$; the model $m$ does not seem to lead to a
usable formula in this case.

This is
useful for sparse rewards: contrary to TD methods, $M$ and $V$
may be learned without ever seeing the reward, provided the target state is
known. (By ``known'', we mean we know the features or input representation of 
the target state, as provided to the neural networks that learn $M$ and
$V$.) 
This also extends to linear
combinations of a finite number of rewards at known states.

\item For general (dense) rewards and without the FB representation, the
simplest option is to learn a model of $V$ based on $V=MR$. This becomes a
supervised learning problem. No matrix product is necessary:
we can perform a stochastic gradient descent of
$\norm{V-MR}^2_{L^2(\rho)}$ with respect to the parameters of $V$, just by
sampling states, either with discrete or continuous states.

With $V$  parameterized as $V_\phi$,
and with $M$ parameterized by the model $\tildem_\theta$, we have
\begin{equation}
\label{eq:paramVMRtilde}
-\partial_\phi \norm{V_\phi-MR}^2_{L^2(\rho)}=2\E_{s\sim \rho,\,s_1\sim
\rho} \left[
\partial_\phi V_\phi(s_1)
(r_s\,\tildem(s_1,s)-V_\phi(s_1))
\right]
\end{equation}
where $r_s$ is the reward obtained when visiting state $s$. As for other
algorithms presented here, this requires sampling one or several
additional states $s_1$ in addition to the state $s$ currently visited.

With $M$ parameterized by the model $m_\theta$ instead,
we have
\begin{multline}
\label{eq:paramVMR}
-\partial_\phi \norm{V_\phi-MR}^2_{L^2(\rho)}=
\\
2\E_{s\sim \rho,\,s_1\sim
\rho} \left[
\,r_s\, \partial_\phi V_\phi(s)+(r_s\,m(s_1,s)-V_\phi(s_1))\,\partial_\phi V_\phi(s_1)
\right].
\end{multline}

\item Learning $V$ via $V=MR$ assumes that the model of $M$ is reasonably
accurate: any error on $M$ shows up on $V$. Another option is to just use $MR$ as a component in the model
of $V$, or as an
initialization to $V$. For instance, $V$ may be parameterized as 
\begin{equation}
V\deq V_{\phi_1}+V_{\phi_2}
\end{equation}
where $V_{\phi_1}$ is trained to match $MR$ using
\eqref{eq:paramVMRtilde}, and $V_{\phi_2}$ is learned via ordinary TD.

In the FB representation this would yield
\begin{equation}
\label{eq:FBcomponent}
V(s)= \transp{F(s)}B(R)+V_{\phi_2}(s)
\end{equation}
where $B(R)$ is estimated online as above, and $\phi_2$ is estimated by
ordinary TD.

This makes particular sense for the FB representation:
in Appendix~\ref{app:FB} we prove that the fb-FB algorithm minimizes a
loss producing a
truncated SVD of $M$, thus focussing on large eigenvalues of $M$
(large eigenvalues of $P$, long-range dependencies in the environment). Thus $\transp{F(s)}B(R)$
will focus on large eigenvalues of $P$. The training of $F$ and $B$ is
reward-independent (``unsupervised'' reinforcement learning). Thus,
ordinary TD on $V_{\phi_2}$ may be useful to catch
short-range (high-frequency) behavior in the reward.

\item Another option is to directly use samples from $MR$ instead of $V$
in the policy gradient update. This emphasizes $M$ as a ``credit
assignment'' for past actions. 

Abbreviate $sas'\sim \rho \pi P$ for the sampling of a state $s\sim
\rho$, action $a\sim \pi(a|s)$, and next state $s'\sim P(\d s'|s,a)$.
Starting with the policy gradient
\eqref{eq:policygradient} with baseline $b=V$, substituting $V(s')=\E_{s_1\sim \rho}
\tildem(s',s_1)r_{s_1}$, and renaming variables so that all rewards are
taken at the same point, we find
\begin{equation}
\del \pi=\E_{
\begin{subarray}{c}
sas'\sim \rho \pi P 
\\
s_1 a_1 s'_1 \sim \rho\pi P 
\end{subarray}
}
\left[
r_{s,a} \left(
\partial \ln \pi(a|s)+ (\gamma m_{s'1 s}-m_{s_1s})\partial \ln \pi(a_1|s_1)
\right)
\right]
\end{equation}
where two independent transitions must be sampled from the dataset.
In this expression, the model $m$ serves as a credit assignment to
increase the likelihood of those actions $a_1$ at other (past) states
that are estimated to lead to a reward $r_{s,a}$ at the current state
$s$. This is compatible with the view of $M$ as a model of eligibility
traces (Appendix~\ref{sec:traces}).

However, this is probably a high-bias, high-variance option, requiring a
good model of $M$.

\item Finally, $M$ may be used as a preconditioner for TD on $V$. Indeed,
by Theorem~\ref{thm:onlineM}, the ``true'' change of the value function
upon observing a new transition $s\to s'$ with reward $r_s$ is
\begin{alignat}{2}
\label{eq:onlineV_2}
\del V_{s_1} =\tfrac{1}{n_s}(r_s+\gamma \hat V_{s'}-\hat V_s) \,\hat
M_{s_1s}+o(1/n_s)
& \qquad  \forall s_1
\end{alignat}
namely, the Bellman gap $r_s+\gamma \hat V_{s'}-\hat V_s$ is sent back to
every ``predecessor state'' $s_1$ with coefficient $\hat M_{s_1s}$. (See
Appendix~\ref{sec:traces} for $M$ as an expected eligibility trace.)

The resulting parametric update is obtained as follows.

\begin{prop}[ (TD preconditioned by $M$ for the value function)]
\label{prop:VfromM}
Let $V_\phi$ be a smooth parametric model of the value function.
Define an update of $V$ by setting
$V^\tar\deq V_{\phi_t}+\del V$ with $\phi_t$ the parameter at
step $t$, and $\del V$ given by \eqref{eq:onlineV_2}, and taking the
gradient of the loss $J^V(\phi)\deq\frac{1}{2}
\norm{V_\phi-V^\tar}^2_{L^2(\rho)}$. Assume $\hat M$ is equal to the
model
\eqref{eq:Mmodel} using $m_\theta$. Then this gradient is
\begin{multline}
\label{eq:Vparam}
-\partial_\phi J^V(\phi)_{|\phi=\phi_t}=
{\textstyle \frac1t} \left(
r_s+\gamma V_{\phi_t}(s')-V_{\phi_t}(s)
\right)
\left(
\partial_\phi V_{\phi_t}(s)
\right.
\\\left.
+
\,\E_{ s_1\sim \rho} [m_{\theta}(s_1,s) \,\partial_\phi
V_{\phi_t}(s_1)]
\right)
+o(1/t)
\end{multline}
where $t$ is the total number of observations. For the model
\eqref{eq:Mmodeltilde} using $\tilde m_\theta$, this gradient is
\begin{multline}
\label{eq:Vparamtilde}
-\partial_\phi J^V(\phi)_{|\phi=\phi_t}=
{\textstyle \frac1t}
\left(
r_s+\gamma V_{\phi_t}(s')-V_{\phi_t}(s)
\right)
\E_{ s_1\sim \rho}\,
\tildem_{\theta}(s_1,s) \,\partial_\phi
V_{\phi_t}(s_1)
\\
+\,o(1/t)
\end{multline}
\end{prop}

The learning rate $1/t$ just results from the direct empirical averages
used to estimate the process in Section~\ref{sec:SSIPE}, and may be
replaced with any learning rate.

This involves sampling an additional state $s_1\sim \rho$ and applying a
TD update at that point, with weight depending on $M$. In the model of
$M$ using $m_\theta$, this appears as a correction to ordinary TD; in the model
of $M$ using $\tildem_\theta$, everything is included in $\tildem$.

Notably, even if the model of $M$ is wrong, , the true value function is
still a fixed point of \eqref{eq:Vparam} and \eqref{eq:Vparamtilde} in
expectation over $s'$ and $r_s$; it is the only fixed point provided
$\hat M$ is invertible and $\rho>0$. This is a theoretical advantage over
all other estimates of $V$ described above.  However, the sampling of
$s_1$ adds variance, and any negative eigenvalues in the estimate of $M$
will produce divergence.

\end{itemize}

\section{Small Learning Rates and the Continuous-Time Analysis}

This section is a more informal discussion about intuitions coming from a
continuous-time analysis when the learning rate is small. We will not
present formal statements. For simplicity we restrict ourselves to the
tabular,
finite state case so that all objects are always well-defined without
smoothness conditions, but
in principle the analysis extends to any
state space.

We also assume that states are sampled uniformly ($\rho$ is uniform) so
that the expected updates correspond to the Bellman operators.
Introducing non-uniform $\rho$ does not fundamentally change the results
about the forward and backward Bellman operators (indeed, the eigenvalues
of the matrix $\diagrho(\Id-\gamma P)$ have positive real part, just like
those of $\Id-P$, for any positive $\rho$).

For the Bellman--Newton operator, full non-asymptotic convergence rates
were provided in Theorem~\ref{thm:convergence}. Here, we provide a more
intuitive asymptotic analysis that clarifies how the error decreases
faster than with TD.

\subsection{Continuous-Time Analysis of the Forward and Backward Bellman
Operators}

The forward Bellman operator on $M$ with learning rate $\eta>0$ is
\begin{equation}
M\gets (1-\eta)M+\eta(\Id+\gamma PM).
\end{equation}
When $\eta$ is small, after $n$ iterations, the value of $M$ approximates
the value at time $t=n\eta$ of the solution of the matrix ordinary differential
equation
\begin{equation}
\frac{\d M_t}{\d t}=\Id+\gamma P M_t - M_t=\Id-\Delta M_t
\end{equation}
where $\Delta=\Id-\gamma P$ is the Laplacian associated with the Markov
process. The
solution to this equation is
\begin{equation}
M_t=\Delta^{-1}+e^{-t\Delta} (M_0-M)
\end{equation}
where $\Delta^{-1}$ is the true successor state matrix, $M_0$ is the initial value, and $e^{-t \Delta}$ is the
exponential of the matrix $t\Delta$.

Likewise, the backward Bellman operator on $M$ with learning rate
$\eta>0$ is
\begin{equation}
M\gets (1-\eta)M+\eta(\Id+\gamma MP).
\end{equation}
When $\eta$ is small, after $n$ iterations, the value of $M$ approximates
the value at time $t=n\eta$ of the solution of the ordinary differential
equation
\begin{equation}
\frac{\d M_t}{\d t}=\Id+\gamma M_tP - M_t=\Id-M_t\Delta 
\end{equation}
process. The
solution to this equation is
\begin{equation}
M_t=\Delta^{-1}+(M_0-M)e^{-t\Delta} 
\end{equation}
where $M_0$ is the initial value.
Letting $E_t$ be the error at time $t$:
\begin{equation}
E_t\deq M_t-\Delta^{-1}
\end{equation}
then the errors at time $t$ are
$E_t=e^{-t\Delta}E_0$ and $E_t=E_0 e^{-t\Delta}$ for the forward and backward operators,
respectively.

Thus the forward and backward equations converge at the same rate. Indeed,
assume for simplicity that $\Delta$ is diagonalizable, with eigenvalues
$\lambda_i$. \footnote{This
occurs for a dense subset of stochastic matrices $P$. If not, the
analysis is more technical, with polynomials in front of the exponentials
of the eigenvalues, but the conclusions are similar.} 
 By the spectral properties of stochastic
matrices, the eigenvalues of $\Delta$ have positive real part:
$\Re\lambda_i \geq 1-\gamma$. (The largest eigenvalue of $\Delta$ is
$1-\gamma$, with multiplicity $1$ if $P$ is irreducible.)
This implies that the errors tend to $0$.

For a more precise analysis,
let $u_i$ and $v_i$
be respectively the right and left eigenvectors of $\Delta$, associated
with eigenvalues $\lambda_i$.

Since the $u_i$'s and the $v_i$'s form bases,
one can decompose the initial error $E_0$ as $E_0=\sum_{i,j} \alpha_{ij}
u_i \transp{v_j}$. Then one checks that the error at time $t$ for the
continuous-time forward Bellman operator is
\begin{equation}
E_t=\sum_{i,j} e^{-t\lambda_i} \alpha_{ij} u_i \transp{v_j}
\end{equation}
for the forward operator, and
\begin{equation}
E_t=\sum_{i,j} e^{-t\lambda_j} \alpha_{ij} u_i \transp{v_j}
\end{equation}
for the backward operator.

The eigenvalues are the same for the forward and backward operator. Each
eigenvalue has multiplicity $n$ (the number of states) over the state of
matrices $M$, corresponding to all choices of $j$ for a given $i$ or
conversely. Notably, the smallest eigenvalue of $\Delta$ is $1-\gamma$,
corresponding to the direct of slowest convergence. This eigenvalue has
multiplicity $n$ when acting on $M$.

\subsection{Mixing Forward and Backward TD Improves Convergence}
\label{sec:mixingfbtd}

Interestingly, if one mixes the forward and backward operators, then this
eigenvalue analysis changes. The smallest eigenvalue is still the same, but its
multiplicity decreases considerably, from $n$ to $1$. Indeed, assume that
we perform alternatively one step of the forward and backward Bellman
operators, each with learning rate $\eta$. When $\eta$ is small, the
dynamics tends to that of the continuous-time ordinary differential
equation
\begin{equation}
\frac{\d M_t}{\d t}=
\frac12\left(\Id+\gamma P M_t-M_t\right)
+\frac12\left(\Id+\gamma M_tP-M_t\right)
=\Id-\frac12 (\Delta M_t+M_t \Delta)
\end{equation}
whose solution is
\begin{equation}
M_t=\Delta^{-1}+e^{-t\Delta/2}(M_0-M)e^{-t\Delta/2}
\end{equation}
where $\Delta^{-1}$ is the true successor state matrix. Thus, the
error $E_t\deq M_t-\Delta^{-1}$ satisfies $E_t=e^{-t\Delta/2}E_0 e^{-t\Delta/2}$.

But now, with the same eigenvector decomposition as above, we find
\begin{equation}
E_t=\sum_{i,j} e^{-t(\lambda_i+\lambda_j)/2} u_i\transp{v_j}.
\end{equation}
In particular, the error in the direction $u_i\transp{v_j}$ decreases
fast if \emph{at least} one of $\lambda_i$ or $\lambda_j$ has large real
part. Notably, the slowest convergence now occurs ony if \emph{both} $i$
and $j$ correspond to the smallest eigenvalue $1-\gamma$: this smallest
eigenvalue now has multiplicity $1$.

Thus, mixing the forward and backward Bellman operator does produce a
positive effect on convergence speed, bringing the multiplicity of the
worst eivengalue from $n$ (the number of states) to $1$, and generally
picking the best eigenvalue in each direction of the error.

\subsection{Continuous-Time Analysis of the Bellman--Newton Operator}
\label{sec:BNconttime}

Remember the Bellman--Newton operator $M\mapsto (1+\eta) M-\eta
M(\Id-\gamma P)M$ (Definition~\ref{def:BN}) with learning rate $\eta$.
When $\eta$ is small, after $n$ iterations of this operator, the value of
$M$ approximates the value at time $t=n\eta$ of the solution of the
matrix ordinary differential equation
\begin{equation}
\frac{\d M_t}{\d t}=M_t-M_t \Delta M_t
\end{equation}
where $\Delta=\Id-\gamma P$ as above. Obviously $M=\Delta^{-1}$ is a
fixed point. However, as with the Bellman--Newton operator, there are
other fixed points, such as $M=0$: since the differential equation
preserves the kernel and image of $M_t$, there is a (unique) fixed point
for every choice of kernel and image, amounting to computing the inverse
of $\Delta$ in the associated subspaces. Still, $\Delta^{-1}$ is the only full-rank
fixed point.

The accelerated asymptotic convergence of the Bellman--Newton operator
compared to TD on $M$ becomes clear on this continuous-time version.
Define the error
\begin{equation}
E_t\deq \Id-M_t\Delta
\end{equation}
(beware this differs from the definition of $E_t$ in the sections above).
It evolves according to
\begin{equation}
\label{eq:BNerror}
\frac{\d E_t}{\d t}=-E_t+E_t^2.
\end{equation}
This is generally convergent except for some initializations
(more on this below).

When the error is small, the dynamics is
$E'_t=-E_t+O(E_t^2)\approx -E_t$. The same holds for the error
$M_t-\Delta^{-1}=-E_t\Delta^{-1}$. So, in the small error regime, the
error $E_t$ decreases at a constant exponential rate, \emph{independently of
the Markov process}. This contrasts with the forward Bellman equation, whose
convergence depends on the eigenvalues of $\Id-\gamma P$, and which will
converge slowly if $P$ has eigenvalues close to $1$.

In this sense, the continuous-time Bellman--Newton dynamics is to the
Bellman operator what continuous-time second-order gradient descent is to
continuous-time gradient descent: it removes dependencies on the
eigenvalues for convergence close to the solution.

Global initialization and convergence outside of the small-error regime
is best understood by introducing a fictitious value of $P$ associated
with $M_t$. Since $M_t$ converges to $(\Id-\gamma P)^{-1}$, let us
introduce $P_t$ such that $M_t=(\Id-\gamma P_t)^{-1}$, namely, $\gamma
P_t\deq \Id-M_t^{-1}$, assuming $M_t$ is invertible. On $P_t$, the
evolution equation of $M_t$ becomes
\begin{equation}
\frac{\d P_t}{\d t}=-P_t+P
\end{equation}
which is affine, with solution $P_t=P+e^{-t}(P_0-P)$. Thus, the solution
for $M_t$ is
\begin{equation}
M_t=(\Id-\gamma P+\gamma e^{-t}(P_0-P))^{-1}.
\end{equation}

Namely, on the variable $P$, the solution just follows a straight line
from $P_0$ to $P$ at a fixed exponential decay rate. $P_t$ always converges;
however, $M_t$ may be undefined if $\Id-\gamma P_t$ is not invertible for
some $t$. This depends on the initialization $P_0$ (therefore, on $M_0$).

For instance, if $P_0$ is equal to any (sub)stochastic matrix, then $P_t$
is (sub)stochastic as well, and $\Id-\gamma P_t$ is always invertible, so
that $M_t$ converges. This happens for instance: if $P_0=0$, namely,
$M_0=\Id$; or if $M_0$ is initialized to the successor matrix of
\emph{any} Markov process.

More possible initializations appear if considering the dynamics of
$E_t$. Assume $E_t$ is diagonalizable (this is the case for random
initializations). Then from \eqref{eq:BNerror}, the eigenvectors of $E_t$
stay the same over time, and each associated eigenvalue $\lambda$ evolves according
to $\lambda'=-\lambda+\lambda^2$. As long as $\lambda\neq 0$, this is equivalent to
$(\lambda^{-1})'=\lambda^{-1}-1$. So each eigenvalue $\lambda^{-1}$
reaches $-\infty$, so that each eigenvalue $\lambda$ reaches $0$. The
exception is when $\lambda^{-1}=0$ at some point, in which case $\lambda$
diverges. Since $(\lambda^{-1})'=\lambda^{-1}-1$, this happens if and
only if $\lambda^{-1}$ is initially equal to some positive real value in
the complex plane. So there is a half-line of eigenvalues of $E_0$ in the
complex plane which will lead to divergence.
\footnote{This does not show that a pure random initialization converges
with probability $1$: indeed, a random
\emph{real} matrix will typically have some real eigenvalues, which will
lie on the wrong half-line with some positive probability.}



\bibliographystyle{plainnat}
\bibliography{biblio}

\vfill 

\pagebreak

\appendix

\section{Further Variants and Properties of TD for Successor States}
\label{sec:basictd}

\subsection{Using a Target Network}

In parametric TD, it is possible to
get closer to an exact application of the Bellman operator, by performing
several gradient steps to bring the model $M_\theta$ closer to the
Bellman operator $\Id+\gamma PM_{\theta^\tar}$ for a fixed previous value
of the parameter $\theta^\tar$, and only update $\theta^\tar\gets \theta$
once in a while. The formulas are as follows.

\begin{thm}[ (Parametric TD for $M$ with a target network)]
\label{thm:paramtd_target}
Keep the setting of Theorem~\ref{thm:paramtd}, but set the target
$M^\tar$ to $M^\tar\deq \Id+\gamma PM_{\theta^\tar}$ for some value $\theta^\tar$ of
the parameter. Then the gradient step to bring $M_\theta$ closer to
$M^\tar$ is
\begin{multline}
\label{eq:paramtd_target}
-\partial_\theta J(\theta)=
\E_{ s\sim \rho, \,s'\sim P(s,\d s'), \,s_2\sim\rho }
\left[
\gamma \,\partial_\theta m_{\theta}(s,s')
\right.
\\+\left.
\partial_\theta m_{\theta}(s,s_2)\,
(\gamma m_{\theta^\tar}(s',s_2)-m_{\theta}(s,s_2))
\right]
\end{multline}
for the model \eqref{eq:Mmodel} using $m_\theta$, and
\begin{multline}
\label{eq:paramtdtilde_target}
-\partial_\theta J(\theta)=
\E_{ s\sim \rho, \,s'\sim P(s,\d s'), \,s_2\sim\rho }
\left[\,
\partial_\theta \tildem_{\theta}(s,s)
\right.
\\+\left.
\partial_\theta \tildem_{\theta}(s,s_2)\,
(\gamma \tildem_{\theta^\tar}(s',s_2)-\tildem_{\theta}(s,s_2))
\right]
\end{multline}
for the model \eqref{eq:Mmodeltilde} using $\tildem_\theta$.
\end{thm}

\subsection{TD on $M$ with Multi-Step Returns}
\label{sec:horizonk}

A multistep, horizon-$h$ version of TD on $M$ can be defined by iterating
the 
Bellman equation, which yields $M=\Id+\gamma
P+\cdots+\gamma^{h-1}P^{h-1}+\gamma^h P^hM$. This requires being able to observe
$h$ consecutive transitions from the process. The
corresponding parametric update is as follows.

\begin{thm}[ (Multi-step TD for successor states with
function approximation)]
\label{thm:paramtd_multistep}
Maintain a parametric model of $M$
as in
Section~\ref{sec:M-func-approx}
via
$M_{\theta_t}(s_1,\d s_2)=\delta_{s_1}(\d s_2)+m_{\theta_t}(s_1,s_2)\rho(\d s_2)$, with $\theta_t$ the value
of the parameter at step $t$, and with $m_\theta$ some smooth family of
functions over pairs of states.

For $h\geq 1$,
define a target update of $M$ via the horizon-$h$ Bellman equation, $M^\tar\deq
\Id+\gamma P+\cdots+\gamma^{h-1}P^{h-1}+\gamma^h P^hM_{\theta_t}$.
Define the loss between $M$ and $M^\tar$ via $J(\theta)\deq \frac12
\norm{M_\theta-M^\tar}_\rho^2$ using the norm \eqref{eq:norm}.

Then the
gradient step on $\theta$ to reduce this loss is
\begin{multline}
\label{eq:paramtd_multistep}
-\partial_\theta J(\theta)_{|\theta=\theta_t}=
\E_{s_0\sim \rho, \, s_1\sim P(s_0,\d s_1), \,\ldots, \, s_h\sim P(s_{h-1},\d
s_h), \,s_\tar\sim\rho }
\\
\left[
\gamma\,\partial_\theta m_{\theta_t}(s_0,s_1)
+\gamma^2 \,\partial_\theta m_{\theta_t}(s_0,s_2)
+\cdots
+\gamma^h \,\partial_\theta m_{\theta_t}(s_0,s_h)
\right.
\\
\left.
+\,
\partial_\theta m_{\theta_t}(s_0,s_\tar)\,
(\gamma^h \,m_{\theta_t}(s_h,s_\tar)-m_{\theta_t}(s_0,s_\tar))
\right].
\end{multline}
For the model \eqref{eq:Mmodeltilde} using $\tildem_\theta$, this update
is
\begin{multline}
\label{eq:paramtd_multistep_mtilde}
-\partial_\theta J(\theta)_{|\theta=\theta_t}=
\E_{s_0\sim \rho, \, s_1\sim P(s_0,\d s_1), \,\ldots, \, s_h\sim P(s_{h-1},\d
s_h), \,s_\tar\sim\rho }
\\
\left[
\partial_\theta \tildem_{\theta_t}(s_0,s_0)
+\gamma \,\partial_\theta \tildem_{\theta_t}(s_0,s_1)
+\cdots
+\gamma^{h-1} \,\partial_\theta \tildem_{\theta_t}(s_0,s_{h-1})
\right.
\\
\left.
+\,
\partial_\theta \tildem_{\theta_t}(s_0,s_\tar)\,
(\gamma^h \,\tildem_{\theta_t}(s_h,s_\tar)-\tildem_{\theta_t}(s_0,s_\tar))
\right].
\end{multline}
\end{thm}

\subsection{Tabular TD on $MR$ Is Tabular TD on $V$}
\label{sec:tdistd}

In the tabular case,
if the reward is deterministic, learning $V$ via ordinary TD is
equivalent to learning $V$ via the matrix product $V=MR$ with $M$ learned
via tabular TD, as follows.

\begin{thm}
\label{thm:tdistd}
Consider a Markov reward process with deterministic reward $R$.
Initialize an estimate $\hat V$ of $V$ to $0$ and an estimate $\hat M$ of
$M$ to $0$. Each time a transition $s\to s'$ with reward $r_s=R_s$ is
observed, update $\hat V$ via ordinary TD and $\hat M$ via TD for
successor states, with learning rate $\eta$, namely
\begin{align}
\hat V_s &\gets \hat V_s + \eta \left(r_s + \gamma \hat V_{s'}-\hat
V_s\right),
\\
\hat M_{ss_2}& \gets \hat M_{ss_2}+\eta \left(\1_{s=s_2}+\gamma \hat
M_{s's_2}-\hat M_{ss_2}\right) \qquad \forall s_2.
\end{align}
Then at every time step, $\hat V=\hat M R$.
\end{thm}

\begin{proof}
By induction on the time step. This is true at time $0$ thanks to the
initialization. If $\hat V=\hat MR$ at one time step, then the update of
$\hat M R$ at the next time step is
\begin{align}
(\hat MR)_s &= \sum_{s_2} {\hat M_{ss_2}} R_{s_2}
\\&\gets \sum_{s_2} \left({\hat M_{ss_2}} R_{s_2} +\eta\left(
\1_{s=s_2} +\gamma \hat M_{s's_2}-\hat M_{ss_2}\right)R_{s_2}\right)
\\&= (\hat M R)_s + \eta \left(
R_s + \gamma (\hat M R)_{s'} - (\hat M R)_s
\right)
\end{align}
which is the same update as $\hat V_s$. The values at the other states
are not updated.
Therefore, if $\hat V=\hat MR$ before the update, this still holds after
the update.
\end{proof}

\subsection{The Parametric Update for Backward TD}
\label{sec:backwardTD_param}

We now state the
analogue of  Theorem~\ref{thm:paramtd} for backward TD; this provides
the associated parametric update.

\begin{thm}[ (Backward TD for successor states with
function approximation)]
\label{thm:paramtd_right}
Maintain a parametric model of $M$
as in
Section~\ref{sec:M-func-approx}
via
$M_{\theta_t}(s_1,\d s_2)=\delta_{s_1}(\d s_2)+m_{\theta_t}(s_1,s_2)\rho(\d s_2)$, with $\theta_t$ the value
of the parameter at step $t$, and with $m_\theta$ some smooth family of
functions over pairs of states.

Define a target update of $M$ via the Bellman equation, $M^\tar\deq
\Id+\gamma M_{\theta_t}P$.
Define the loss between $M$ and $M^\tar$ via $J(\theta)\deq \frac12
\norm{M_\theta-M^\tar}_\rho^2$ using the norm \eqref{eq:norm}.

Then the
gradient step on $\theta$ to reduce this loss is
\begin{multline}
\label{eq:paramtd_right}
-\partial_\theta J(\theta)_{|\theta=\theta_t}=
\E_{ s\sim \rho, \,s'\sim P(s,\d s'), \,s_1\sim\rho }
\left[
\gamma \,\partial_\theta m_{\theta_t}(s,s')
\right.
\\+\left.
m_{\theta_t}(s_1,s)\,
(\gamma \,\partial_\theta m_{\theta_t}(s_1,s')-\partial_\theta m_{\theta_t}(s_1,s))
\right].
\end{multline}
For the model variant in Eq.~\ref{eq:Mmodeltilde}, $M_{\theta_t}(s_1,\d
s_2)=\tildem_{\theta_t}(s_1,s_2)\rho(\d s_2)$, the gradient step on
$\theta$ to reduce the loss $J(\theta)$ is
\begin{multline}
\label{eq:paramtd_righttilde}
-\partial_\theta J(\theta)_{|\theta=\theta_t}=
\E_{ s\sim \rho, \,s'\sim P(s,\d s'), \,s_1\sim\rho }
\left[
\,\partial_\theta \tildem_{\theta_t}(s,s)
\right.
\\+\left.
\tildem_{\theta_t}(s_1,s)\,
(\gamma \,\partial_\theta \tildem_{\theta_t}(s_1,s')-\partial_\theta
\tildem_{\theta_t}(s_1,s))
\right].
\end{multline}
\end{thm}

\subsection{Having Targets on Features of the State}
\label{sec:targetfeatures}

Learning $M$ is particularly suitable when the reward is located at a
single known goal state $g$: then, the value function $V(s)$ is proportional to
$\tildem(s,g)$. For how to exploit $M$ with dense rewards, we refer to
Section~\ref{sec:V}.

Another scenario is to have a target value for \emph{some}
features of the state, not necessarily the whole state itself: namely,
the reward is nonzero when some known feature $\phi(s)$ of state $s$ is
equal to some known goal $g$. In that case, it is convenient to learn a
smaller object than $M$, from which the value function can be read
directly. This is also useful if the reward is known to depend only on
$\phi(s)$.

\newcommand{\Mphi}{M^\phi}
\newcommand{\mphi}{m^\phi}

\begin{defi}
Let $\phi\from S\to \R^k$ be any measurable map. The \emph{successor
feature operator} $\Mphi$ is defined as follows: for each state $s_1$,
$\Mphi(s_1,\d g)$ is a measure on $\R^k$ equal to the pushforward of
$M(s_1,\d s_2)$ by the map $s_2\mapsto g=\phi(s_2)$.
\end{defi}

This operator is different from successor representations: here we
keep track of the whole future \emph{distribution} of values of $\phi$,
not just the expected future value of $\phi$.

$\Mphi$ can be used to compute the value function of any reward that
depends only on $\phi(s)$.

\begin{prop}
\label{prop:targetrewards}
Assume that the reward function at state $s$ is equal to $R(\phi(s))$,
namely, it depends only on $\phi$.
Let $\tau$ be any probability distribution on features in $\R^k$. Assume
that $\Mphi$ is parameterized as $\Mphi(s,\d g)=\mphi(s,g)\tau(\d g)$.
Then the value function of a state $s$ for this reward is
\begin{equation}
V(s)=\E_{g\sim \tau} [\mphi(s,g)R(g)].
\end{equation}
In particular, if the reward is nonzero exactly when the feature
$\phi(s)$ is equal to some target value $g$, then the value function is
proportional to $\mphi(s,g)$.
\end{prop}

This is useful only if an algorithm to learn $\mphi$ is available.
Forward TD can be defined on $\Mphi$, based on the following Bellman
equation.

\begin{prop}
\label{prop:targetbellman}
$\Mphi$ satisfies the Bellman equation $\Mphi(s,\d g)=\delta_{\phi(s)}(\d
g)+\gamma \E_{s'\sim P(s,\d s')} \Mphi(s',\d g)$.
\end{prop}

\begin{thm}
\label{thm:targetfeatures}
Let $\tau$ be any probability distribution on features in $\R^k$. Assume
that $\Mphi$ is parameterized as $\Mphi_\theta(s,\d g)=\mphi_\theta(s,g)\tau(\d
g)$ for some parametric family of functions $\mphi_\theta(s,g)$ with
parameter $\theta$. 

Let $\theta_0$ be some value of the parameter, and
define a target operator $M^\tar$ by the Bellman equation: $M^\tar\deq
\delta_{\phi(s)}(\d
g)+\gamma \E_{s'\sim P(s,\d s')} \Mphi_{\theta_0}(s',\d g)$.
Define the loss between $\Mphi$ and $M^\tar$ via $J(\theta)\deq \E_{s\sim
\rho,\,g\sim \tau} ((\Mphi_\theta(s,\d g)-M^\tar(s,\d g))/\tau(\d g))^2$.

Then the gradient step to bring $\Mphi$ closer to $M^\tar$ in this norm
is
\begin{multline}
-\partial_\theta J(\theta)=\E_{s\sim \rho,\,s'\sim
P(s,\d s'),\,g\sim \tau} \left[
\,\partial_\theta \mphi_{\theta}(s,\phi(s))
\right.
\\
\left.
+\,\partial_\theta \mphi_\theta(s,g)\left(
\gamma \mphi_{\theta_0}(s',g)-\mphi_\theta(s,g)
\right)
\right].
\end{multline}
\end{thm}

\todo{frankly the parameterization using different measures makes much
more sense here. I'd rather learn the density $\mphi$ wrt to the Lebesgue
measure and have a Gaussian corrective factor for the first term, rather
than learn the density $\mphi$ wrt to a Gaussian, which diverges quickly.
The other option is just to take $\tau$ as the pushforward by $\phi$,
namely, take an $s_2\sim \rho$ and set $g=\phi(s_2)$. Advantage=the
latter involves no user choice}

\todo{zeta?? works, I think, with $\delta_{\phi(s)}(\d g)$ as the base
term, but no interest for $V$ function}

Once more, the term $\partial_\theta \mphi_{\theta}(s,\phi(s))$ makes every
transition informative: when visiting state $s$, we increase the
probability to reach the goal $\phi(s)$.

\subsection{Taking $\gamma$ Close to $1$: Relative TD}
\label{sec:relative}

\newcommand{\rel}{\mathrm{rel}}
\newcommand{\rhorel}{\rho_\rel}

For $\gamma$ close to $1$, it is known that the value function behaves
like a large constant plus an informative signal,
$V(s)=\frac{c}{1-\gamma}+V^\rel(s)$. A similar phenomenon occurs with
$M$. The large constant affects learning in practice, especially for
Bellman--Newton which has terms scaling like $M^2$.

$V^\rel$ can be learned
directly via \emph{relative TD}, adapted from \emph{relative value
iteration}
\citep[\S5.3.1]{bertsekasvol2_ed4}, \citep[\S6.6]{puterman2014markov}, just by removing the value of
$V$ at a reference state from the Bellman equation. Namely, with reference state
$s_\rel$, the relative TD update upon observing a transition $s\to s'$
with reward $r_s$ is
\begin{equation}
\delta V^\rel_s = r_s+\gamma V^\rel_{s'} -V^\rel_s-\gamma
V^\rel_{s_\rel}.
\end{equation}
This makes it possible to use a $\gamma$ very close to $1$, or even $\gamma=1$ if the
Markov process is ergodic or ``unichain''\todo{REF}.

Relative TD can be transposed to $M$ directly. The relative Bellman
equation above rewrites as $V^\rel=R+\gamma
(P-\1\transp{\1_{s_\rel}})V^\rel$.
Therefore, the solution is given by 
\begin{equation}
V^{\rel}=(\Id-\gamma P +\gamma
\1\transp{\1_{s_\rel}})^{-1} R.
\end{equation}
Thus we can set $M^\rel\deq (\Id-\gamma P +\gamma
\1\transp{\1_{s_\rel}})^{-1}$.

More generally, working with a distribution of reference states rather
than a single reference state, we will set
\begin{equation}
M^\rel\deq (\Id-\gamma P +\gamma
\1\transp{\rhorel})^{-1}
\end{equation}
where $\rhorel$ is the probability vector for reference states. When
$\gamma=1$ and $\rhorel=\rho$ is the invariant distribution of the
Markov process, this is exactly the fundamental matrix of the Markov
process \citep{kemenysnell1960}.

The effect of relative TD is just to replace the operator $P$ with
$P-\1\transp{\rhorel}\,$ everywhere. In practice, in the various formulas,
for every term involving the second state $s'$ of a transition $s\to s'$,
a corresponding term is added with $s_\rel$ instead of $s'$ and with the
opposite sign. Thus, the update \eqref{eq:paramtd} for parametric TD for $M$ becomes
\begin{multline}
\label{eq:paramtd_rel}
\E_{ s\sim \rho, \,s'\sim P(s,\d s'), \,s_2\sim\rho ,\, s_\rel\sim
\rhorel}
\left[
\gamma \,\partial_\theta m_{\theta_t}(s,s')-\gamma\,\partial_\theta
m_{\theta_t}(s,s_\rel)
\right.
\\+\left.
\partial_\theta m_{\theta_t}(s,s_2)\,
(\gamma m_{\theta_t}(s',s_2)-\gamma m_{\theta_t}(s_\rel,s_2)-m_{\theta_t}(s,s_2))
\right].
\end{multline}
The update for parametric backward TD becomes
\begin{multline}
\label{eq:parambackwardtd_rel}
\E_{ s\sim \rho, \,s'\sim P(s,\d s'), \,s_1\sim\rho ,\, s_\rel\sim \rhorel}
\left[
 \gamma\, \partial_{\theta} m_{\theta}(s,s')-\gamma \,\partial_{\theta}
 m_{\theta}(s,s_\rel)
\right.
\\+\left.
	m_{\theta}(s_1,s) \left(\gamma\,\partial_\theta m_\theta(s_1,s')
	-\gamma\,\partial_\theta m_\theta(s_1,s_\rel)
			- \partial_\theta m_{\theta}(s_1, s) \right)\right].
\end{multline}
The parametric update \eqref{eq:Vparam} of $V$ via $M$ becomes
\begin{multline}
\label{eq:Vparam_rel}
\E_{ s\sim \rho, \,s'\sim P(s,\d s'), \,s_1\sim\rho ,\, s_\rel\sim \rhorel}
\left[
\left(
r_s+\gamma V_{\phi_t}(s')-\gamma V_{\phi_t}(s_\rel)-V_{\phi_t}(s)
\right)
\right.\\\times \left.
\left(
\partial_\phi V_{\phi_t}(s)+
m_{\theta_t}(s_1,s) \,\partial_\phi
V_{\phi_t}(s_1)
\right)\right].
\end{multline}
Finally, the parametric Bellman--Newton update \eqref{eq:Mparam} for $M$ becomes
\begin{multline}
\label{eq:Mparam_rel}
\E_{ s\sim \rho, \,s'\sim P(s,\d s'),\,s_1\sim
\rho,\,s_2\sim\rho,\,s_\rel\sim\rhorel}
\left[
\gamma\,\partial_\theta m_{\theta_t}(s,s')
-\gamma\,\partial_\theta m_{\theta_t}(s,s_\rel)
\right.\\+\left.
\gamma \, m_{\theta_t}(s_1,s)\, \partial_\theta m_{\theta_t}(s_1,s')
-
\gamma \, m_{\theta_t}(s_1,s)\, \partial_\theta m_{\theta_t}(s_1,s_\rel)
\right.
\\+\left.
(\gamma m_{\theta_t}(s',s_2)-\gamma m_{\theta_t}(s_\rel,s_2)-m_{\theta_t}(s,s_2))
\left(\partial_\theta m_{\theta_t}(s,s_2)+
m_{\theta_t}(s_1,s)
\,\partial_\theta m_{\theta_t}(s_1,s_2)\right)
\right].
\end{multline}

\section{Proofs for Sections~\ref{sec:defM}, \ref{sec:td-alg},
\ref{sec:goals}, \ref{sec:BN}, \ref{sec:V}, and Appendix~\ref{sec:basictd}}
\label{sec:proofs}

In this text we consider two parametric models of $M$, 
\eqref{eq:Mmodeltilde} and
\eqref{eq:Mmodel},
given by $\tildem_\theta$ and $m_\theta$ 
respectively. In most proofs, we only cover the more complex model $m_\theta$;
the proofs with $\tildem_\theta$ are similar but simpler.

\subsection{Proofs for Sections~\ref{sec:defM} and \ref{sec:td-alg}: TD
for $M$}


\paragraph{Proof of Theorem~\ref{thm:Mdef}.} By the definition of $M$ in
\eqref{eq:M-def}, for any measurable set $A\subset \mathcal{S}$, for any
$s\in \mathcal{S}$, $M(s,A)$ is defined as
\begin{equation}
M(s,A)=\sum_{n\geq 0} \gamma^n P^n(s,A).
\end{equation}
Since each $P^n(s,\cdot)$ is a probability distribution, $P^n(s,A)\leq 1$ so that
this sum of non-negative terms is bounded by $\frac{1}{1-\gamma}$, and
therefore the sum converges. 
$M(s,\cdot)$ is a positive measure as a convergent sum of positive measures
($\sigma$-additivity for $M(s,\cdot)$ follows from the dominated
convergence theorem). Its
total mass is $M(s,\mathcal{S})=\sum_{n\geq 0}\gamma^n
P(s,\mathcal{S})=\sum_{n\geq 0}\gamma^n=\frac{1}{1-\gamma}$.

As a
positive measure with finite mass, 
$M(s,\cdot)$ acts on bounded measurable functions, just like $P$, via $(Mf)(s)=
\int f(s')M(s,\d s')$.
Since $M$ has mass $\frac{1}{1-\gamma}$ for any $s$, this integral is bounded by
$\frac{1}{1-\gamma} \sup f$, so that
$\sup Mf\leq \frac{1}{1-\gamma}\sup f$ for any function $f\in \BS$.
Thus, $M$ is well-defined as an operator from $\BS$ to $\BS$.

As an operator, one has $\gamma P M=\gamma P \sum_{n\geq 0} \gamma^n
P^n=\sum_{n\geq 1} \gamma^n P^n$. Therefore, $(\Id-\gamma P)M=M-\gamma
PM=\sum_{n\geq 0} \gamma^n
P^n-\sum_{n\geq 1} \gamma^n P^n= \gamma^0 P^0=\Id$ (the sums converge
absolutely by
the same boundedness argument as before, thus justifying the
infinite sum manipulations). This proves that $M$ is a right inverse of
$\Id-\gamma P$ as operators. The computation is identical for the left
inverse; therefore, $M$ and $\Id-\gamma P$ are inverses as operators on
$\BS$.

Finally, let $R$ be any (bounded, measurable) reward function. Since
$(\Id-\gamma P)M=\Id$, one has $(\Id-\gamma P)MR=R$ namely $MR=R+\gamma
PMR$. This proves that $V=MR$ satisfies the Bellman equation
$V=R+\gamma PV$, and so $MR$ is the value function of the Markov reward
process.

\paragraph{Proof of Theorems~\ref{thm:left-bellman}
and~\ref{thm:right-bellman}.} An operator $M'$ satisfies the left Bellman
equation $M'=\Id+\gamma PM'$ if and only if $M'-\gamma PM'=\Id$, or
$(\Id-\gamma P)M'=\Id$, namely, $M'$ is a right inverse of $\Id-\gamma
P$. By Theorem~\ref{thm:Mdef}, $\Id-\gamma P$ is invertible and its
inverse is $M$. Therefore, the only right inverse of $\Id-\gamma P$ is
$M$.

The proof is identical for the backward Bellman equation, with left inverses
instead of right inverses.

%
%
\paragraph{Proof of Propositions~\ref{prop:contract} and
\ref{prop:contract2}.}
By definition of the operator $P$, for any function $f$ we have
$\norm{Pf}_\infty=\sup_s \int f(s') P(s,\d s')\leq \sup_{s'}
f(s')=\norm{f}_\infty$, so that $P$ is $1$-contracting. Therefore, for any bounded
operator $M$ and function $f$, one has $\norm{PMf}_\infty\leq \norm{Mf}_\infty \leq
\norm{M}_{\mathrm{op}}\norm{f}_\infty$, so that
$\norm{PM}_\mathrm{op}\leq \norm{M}_\mathrm{op}$ for any $M$. Therefore, given two
operators $M$ and $M'$, one has $\norm{(\Id+\gamma PM)-(\Id+\gamma
PM')}_\mathrm{op}=\gamma \norm{P(M-M')}_\mathrm{op}\leq \gamma
\norm{M-M'}_\mathrm{op}$.

For the backward Bellman operator, $M\mapsto \Id+\gamma M P$, the proof
is similar, using that for any bounded
operator $M$ and function $f$, one has $\norm{MPf}_\infty\leq
\norm{M}_{\mathrm{op}}\norm{Pf}_\infty
\leq
\norm{M}_{\mathrm{op}}\norm{f}_\infty$, so that
$\norm{MP}_\mathrm{op}\leq \norm{M}_\mathrm{op}$ for any $M$.

\paragraph{Proof of Theorem~\ref{thm:paramtd}.}
%
In this proof, we freely go back and forth between $M$ or $M^\tar$ as measure-valued
functions, and $M$ or $M^\tar$ as operators on
bounded functions. Notably, the operator $\Id$ corresponds to the measure
$\delta_{s_1}(\d s_2)$.

We start with the statement for the first model, $M_{\theta_t}(s_1,\d
s_2)=\delta_{s_1}(\d s_2)+m_{\theta_t}(s_1,s_2)\rho(\d s_2)$.

By definition of $M^\tar=\Id+\gamma PM_{\theta_t}$, and by definition of the action of the operator $P$, we have 
\begin{align}
M^\tar(s,\d s_2)&=\delta_{s}(\d s_2)+\gamma \int_{s'}P(s,\d
s')M_{\theta_t}(s',\d s_2)
\\&=\delta_{s}(\d s_2)+\gamma \int_{s'}P(s,\d
s')\delta_{s'}(\d s_2)+\gamma \int_{s'}P(s,\d
s') m_{\theta_t}(s',s_2)\rho(\d s_2)
\\&=
\delta_{s}(\d s_2)+\gamma P(s,\d
s_2)+\gamma \,\E_{s'\sim P(s,\d
s')}[ m_{\theta_t}(s',s_2)\rho(\d s_2)]
\end{align}
by the definition of the Dirac measure $\delta_{s'}$. Therefore,
\begin{multline}
M^\tar(s,\d s_2)-M_\theta(s,\d s_2)=M^\tar(s,\d
s_2)-\delta_{s}(\d s_2)-m_\theta(s,s_2)\rho(\d s_2)
\\
\label{eq:Mgap}
=\gamma P(s,\d
s_2)+\gamma\, \E_{s'\sim P(s,\d
s')}[ m_{\theta_t}(s',s_2)\rho(\d s_2)]
-m_{\theta}(s,s_2)\rho(\d s_2)
\end{multline}

By definition of $J(\theta)$ and of the norm
$\norm{\cdot}_{\rho}$, we have
\begin{equation}
\label{eq:j}
J(\theta)=\frac12 \iint j_\theta(s,s_2)^2\,\rho(\d s)\rho(\d s_2)
\end{equation}
where $j_\theta(s,s_2)\deq (M^\tar(s,\d s_2)-M_\theta(s,\d
s_2))/\rho(\d s_2)$ (assuming this density exists). \footnote{This proof
involves $P(s,\d s_2)/\rho(\d s_2)$, but this quantity only appears as
$(P(s,\d s_2)/\rho(\d s_2))\rho(\d s_2)$ in the final result
\eqref{eq:dJ_final}. Therefore,
the argument extends by continuity to the case when $P(s,\cdot)$ is not
absolutely continuous with respect to $\rho$: in that case the norm
$J(\theta)$
is infinite but its gradient $\partial_\theta J(\theta)$ is still
well-defined by continuity.}
Consequently,
\begin{equation}
\label{eq:dj}
\partial_\theta J(\theta)=\iint j_\theta(s,s_2)\,\partial_\theta
j_\theta(s,s_2) \rho(\d s)\rho(\d s_2)
\end{equation}
assuming $j_\theta$ is smooth enough so that the derivative makes sense and
commutes with the integral. From the definition of $j_\theta$ and from
\eqref{eq:Mgap} we have
\begin{equation}
j_\theta(s,s_2)=\gamma \frac{P(s,\d s_2)}{\rho(\d s_2)}
+\gamma \,\E_{s'\sim P(s,\d
s')}[ m_{\theta_t}(s',s_2)]
-m_\theta(s,s_2)
\end{equation}
and
\begin{equation}
\partial_\theta j_\theta(s,s_2)=-\partial_\theta m_\theta(s,s_2)
\end{equation}
(and consequently, $j_\theta$ is smooth if $m_\theta$ is smooth).
Therefore,
\begin{equation}
\label{eq:dJ_final}
-\partial_\theta J(\theta)=\iint \partial_\theta
m_\theta(s,s_2)\left(
\gamma \frac{P(s,\d s_2)}{\rho(\d s_2)}
+\gamma \,\E_{s'\sim P(s,\d
s')}[ m_{\theta_t}(s',s_2)]
-m_\theta(s,s_2)
\right)\rho(\d s) \rho (\d s_2)
\end{equation}
The first term $\iint \partial_\theta
m_\theta(s,s_2)\gamma \frac{P(s,\d s_2)}{\rho(\d s_2)}\rho(\d s) \rho (\d
s_2)$ rewrites as $\gamma \iint \partial_\theta m_\theta(s,s_2) P(s,\d s_2)\rho(\d s)$
namely $\gamma \E_{s\sim \rho} \E_{s_2\sim P(s,\d s_2)} \partial_\theta
m_\theta(s,s_2)$.
Renaming $s_2$ to $s'$ in this term ends the proof.

Let us now turn to the model $M_{\theta_t}(s_1,\d s_2)=
\tildem_{\theta_t}(s_1,s_2)\rho(\d s_2)$.
Here, there is a hidden mathematical subtlety with continuous states.
Indeed, in that case, $M_{\theta_t}$ is absolutely continuous with
respect to $\rho$, while $M^\tar$ is not, due to the $\Id$ term, as
discussed in Section~\ref{sec:M-func-approx}. (With the other model, the
$\Id$ terms cancel between $M_{\theta_t}$ and $M^\tar$.) This makes the
norm $J(\theta)=\frac12\norm{M_{\theta}-M^\tar}^2_\rho$ infinite (see its
definition in \eqref{eq:norm}). However, the \emph{gradient} of this norm
is actually still well-defined. There are at least two ways to handle this
rigorously, which lead to the same result: either do the computation in the finite case and observe
that the resulting gradient still makes sense in the continuous case
(which can be obtained by a limiting argument), or
observe that the loss
$J(\theta)$ is equal to $\frac12\norm{M_{\theta}}^2_\rho
-\langle M_{\theta},M^\tar\rangle_\rho
+\frac12\norm{M^\tar}^2_\rho
$
and has the same minima and the same gradients as the loss
$J'(\theta)=\frac12\norm{M_{\theta}}^2_\rho-\langle
M_{\theta},M^\tar\rangle_\rho$ for a given $M^\tar$.
Namely, $J$ and $J'$ differ by a constant in the finite case, and by an
``infinite constant'' in the continuous case. We will 
work with the loss $J'$, which is finite even in the continuous case.

Here $\langle M_1,M_2\rangle_\rho=\int_{s,s_2}
\frac{M_1(s,\d s_2)}{\rho(\d s_2)} \frac{M_2(s,\d s_2)}{\rho(\d
s_2)}\,\rho(\d s)\rho(\d s_2)$ is the dot product associated with the norm
\eqref{eq:norm}. Since the integrand can be rewritten as $\frac{M_1(s,\d
s_2)}{\rho(\d s_2)} \rho(\d s) M_2(s,\d s_2)$, it is well-defined as soon as
at least one of $M_1$ or $M_2$ is absolutely continuous with respect to
$\rho$. Namely,
\begin{equation}
\label{eq:dotprod}
\langle M_1,M_2\rangle_\rho=\int_{s,s_2} \frac{M_1(s,\d
s_2)}{\rho(\d s_2)} \rho(\d s) \,M_2(s,\d s_2).
\end{equation}

Let us compute
$J'(\theta)=\frac12\norm{M_{\theta}}^2_\rho-\langle
M_{\theta},M^\tar\rangle_\rho$.
By definition of $M^\tar=\Id+\gamma PM_{\theta_t}$, and by definition of the action of the operator $P$, we have 
\begin{align}
M^\tar(s,\d s_2)&=\delta_{s}(\d s_2)+\gamma \int_{s'}P(s,\d
s')M_{\theta_t}(s',\d s_2)
\\&=
\delta_{s}(\d s_2)+
\gamma \,\E_{s'\sim P(s,\d
s')}[ \tildem_{\theta_t}(s',s_2)\rho(\d s_2)]
\label{eq:Mtartilde}
\end{align}
by definition of the model $M_{\theta_t}(s_1,\d s_2)=\tildem_{\theta_t}(s_1,s_2)\rho(\d s_2)$. Therefore, by~\eqref{eq:dotprod},
\begin{multline}
\langle M_\theta,M^\tar\rangle_\rho = 
\int_{s,s_2} \tildem_{\theta}(s,s_2)\, \rho(\d s)\, M^\tar(s,\d s_2)
\\=\int_s \tildem_{\theta}(s,s)\, \rho(\d s)+\gamma \int_{s,\,s',\,s_2} 
\tildem_{\theta}(s,s_2)\, \tildem_{\theta_t}(s',s_2) \,\rho(\d s)\,P(s,\d
s')\,\rho(\d s_2)
\end{multline}
thanks to \eqref{eq:Mtartilde}. Next, since $M_\theta(s,\d s_2)=\tildem_\theta(s,s_2)\rho(\d s_2)$, the definition of the norm \eqref{eq:norm}
yields
\begin{equation}
\frac12 \norm{M_{\theta}}^2_\rho=\frac12 \int_{s,s_2} \tildem_\theta(s,s_2)^2\,\rho(\d s)\,\rho(\d s_2).
\end{equation}
Collecting, and rewriting the integrals as expectations, we find
\begin{multline}
J'(\theta)=\E_{s\sim \rho,\,s_2\sim \rho} 
\left[\frac12 \tildem_\theta(s,s_2)^2-\tildem_\theta(s,s)\right]
\\
-\,\gamma \E_{s\sim
\rho,\,s'\sim P(s,\d s'),\,s_2\sim \rho} [\tildem_\theta(s,s_2)\,\tildem_{\theta_t}(s',s_2)]
\end{multline}
hence
\begin{multline}
\partial_\theta J'(\theta)=\E_{s\sim \rho,\,s_2\sim \rho} 
\left[\partial \tildem_\theta(s,s_2) \,\tildem_\theta(s,s_2)-\partial \tildem_\theta(s,s)\right]
\\
-\,\gamma \E_{s\sim
\rho,\,s'\sim P(s,\d s'),\,s_2\sim \rho} [\partial \tildem_\theta(s,s_2)\,\tildem_{\theta_t}(s',s_2)]
\end{multline}
which is the expression given in Theorem~\ref{thm:paramtd} for
$\theta=\theta_t$. This ends the proof.


\subsection{Proofs for Appendix~\ref{sec:basictd}: Further properties of
TD for $M$}

\paragraph{Proof of Theorem~\ref{thm:paramtd_target}.} The proof is
identical to that of Theorem~\ref{thm:paramtd}, but with $\theta^\tar$
instead of $\theta^t$ and no substitution $\theta=\theta_t$ in the last
step.

\paragraph{Proof of Theorem~\ref{thm:paramtd_multistep}.}
Exactly as in Theorem~\ref{thm:paramtd}, setting
$j_\theta(s,s')\deq (M^\tar(s,\d s')-M_\theta(s,\d
s'))/\rho(\d s')$, we have
\begin{align}
\partial_\theta J(\theta)&=\iint j_\theta(s,s')\,\partial_\theta
j_\theta(s,s') \rho(\d s)\rho(\d s')
\\&=\iint \partial_\theta
j_\theta(s,s') \rho(\d s) (j_\theta(s,s')\rho(\d s'))
\\&=\iint \partial_\theta
j_\theta(s,s') \rho(\d s) (M^\tar(s,\d s')-M_\theta(s,\d
s'))
\label{eq:dJ_multistep}
\end{align}
and since $M^\tar$ depends on ${\theta_t}$ but not on $\theta$,
\begin{equation}
\partial_\theta j_\theta(s,s')=-\partial_\theta
\left(\frac{M_{\theta}(s,\d s')}{\rho(\d s')}\right)=-\partial_\theta
m_\theta(s,s').
\label{eq:dj_multistep}
\end{equation}

From the definition of $M^\tar$ we have
\begin{align}
M^\tar(s,\d s')&=\delta_{s}(\d s')+\sum_{i=1}^{h-1} \gamma^i P^i(s,\d s')+
\gamma^h (P^h M_{\theta_t})(s,\d s')
\end{align}
and since $M_{\theta_t}(s,\d s')=\delta_{s}(\d
s')+m_{\theta_t}(s,s')\rho(\d s')$ 
we have $(P^hM_{\theta_t})(s,\d s')=P^h(s,\d s')+\int P^h(s,\d
s'')m_{\theta_t}(s'',s')\rho(\d s')$ so 
the above rewrites as
\begin{equation}
M^\tar(s,\d s')=\delta_{s}(\d s')+\sum_{i=1}^{h-1} \gamma^i P^i(s,\d
s')+
\gamma^h P^h(s,\d s') +\gamma^h \int
P^h(s,\d s'')m_{\theta_t}(s'',s')\rho(\d s')
\end{equation}
and so
\begin{multline}
M^\tar(s,\d s')-M_{\theta_t}(s,\d s')=
\\
-m_{\theta_t}(s,s')\rho(\d s')+\sum_{i=1}^{h} \gamma^i P^i(s,\d
s')+\gamma^h \int
P^h(s,\d s'')m_{\theta_t}(s'',s')\rho(\d s').
\end{multline}

Let us plug this into \eqref{eq:dJ_multistep} for $\theta=\theta_t$, and study each contribution
in turn. The term $-m_{\theta_t}(s,s')\rho(\d s')$ produces a
contribution
\begin{equation}
-\iint\partial_\theta j_{\theta_t}(s,s')\rho(\d s)
m_{\theta_t}(s,s')\rho(\d s')
=\E_{s\sim \rho,\,s'\sim \rho}\,m_{\theta_t}(s,s')\,\partial_\theta
m_{\theta_t}(s,s')
\end{equation}
by \eqref{eq:dj_multistep}. Each term $\gamma^i P^i$ produces a
contribution
\begin{equation}
\gamma^i \iint \partial_\theta j_\theta(s,s')\rho(\d s) P^i(s,\d s')
\end{equation}
which by definition of $P^i$, can be rewritten as
\begin{equation}
\gamma^i\, \E_{s_0\sim \rho,\,s_1\sim P(s_0,\d s_1),\,\ldots,\, s_i\sim
P(s_{i-1},\d s_i)} 
\,\partial_\theta j_\theta(s_0,s_i).
\end{equation}
For the same reason, the term $\gamma^h P^h m_{\theta_t}$ produces a
contribution
\begin{equation}
\gamma^h\, \E_{s_0\sim \rho,\,s_1\sim P(s_0,\d s_1),\, \ldots,\, s_h\sim
P(s_{i-h},\d s_h),\, s'\sim \rho} [
m_{\theta_t}(s_h,s') \,\partial_\theta j_\theta(s_0,s')].
\end{equation}
Collecting all terms and using \eqref{eq:dj_multistep} to replace
$\partial_\theta j$ with $-\partial_\theta m$ leads to the expression in
the theorem.

For the case of the model \eqref{eq:Mmodeltilde} using $\tildem_\theta$,
proceed as for Theorem~\ref{thm:paramtd} and use the loss
$J'(\theta)=\frac12\norm{M_{\theta}}^2_\rho-\langle
M_{\theta},M^\tar\rangle_\rho$, which has the same minima as the loss $J$
but makes sense in a more general setting. In this case we have
\begin{equation}
M^\tar(s,\d s')=\delta_{s}(\d s')+\sum_{i=1}^{h-1} \gamma^i P^i(s,\d
s')
+\gamma^h \int
P^h(s,\d s'')\tildem_{\theta_t}(s'',s')\rho(\d s')
\end{equation}
The dot product $\langle
M_{\theta},M^\tar\rangle_\rho$ is given by
\eqref{eq:dotprod}. Expand the value of $M^\tar$ into \eqref{eq:dotprod},
and proceed as above.

\paragraph{Proof of Theorem~\ref{thm:paramtd_right}.}
As in the proof of Theorems~\ref{thm:paramtd}
and~\ref{thm:paramtd_multistep}, set
$j_\theta(s,s')\deq (M^\tar(s,\d s')-M_\theta(s,\d
s'))/\rho(\d s')$. Then
\begin{align}
\partial_\theta J(\theta)&=\iint j_\theta(s,s')\,\partial_\theta
j_\theta(s,s') \rho(\d s)\rho(\d s')
\\&=\iint \partial_\theta
j_\theta(s,s') \rho(\d s) (j_\theta(s,s')\rho(\d s'))
\\&=\iint \partial_\theta
j_\theta(s,s') \rho(\d s) (M^\tar(s,\d s')-M_\theta(s,\d
s'))
\label{eq:dJ_right}
\end{align}
and since $M^\tar$ depends on ${\theta_t}$ but not on $\theta$,
\begin{equation}
\partial_\theta j_\theta(s,s')=-\partial_\theta
\left(\frac{M_{\theta}(s,\d s')}{\rho(\d s')}\right)=-\partial_\theta
m_\theta(s,s').
\label{eq:dj_right}
\end{equation}

From the definition of $M^\tar$ and the composition of operators, we have
\begin{align}
M^\tar(s,\d s')&=\delta_{s}(\d s')
+ \gamma \int M_{\theta_t}(s,\d s'')P(s'',\d s')
\\&=\delta_s(\d s')+\gamma P(s,\d s')+\gamma \int
m_{\theta_t}(s,s'')\rho(\d s'') P(s'',\d s')
\end{align}
thanks to the parameterization $M_{\theta_t}(s,\d s'')=\delta_s(\d
s'')+m_{\theta_t}(s,s'')\rho(\d s'')$. Thus
\begin{equation}
M^\tar(s,\d s')-M_{\theta_t}(s,\d s')=\gamma P(s,\d s')+\gamma
\E_{s''\sim \rho} [m_{\theta_t} (s,s'') P(s'',\d
s')]-m_{\theta_t}(s,s')\rho(\d s').
\end{equation}
and plugging this into \eqref{eq:dJ_right} at $\theta=\theta_t$,
substituting $-\partial_\theta m_\theta$ for $\partial_\theta j$ as per
\eqref{eq:dj_right}, and rewriting the integrals as expectations under
$\rho$ and $P$, we find
\begin{multline}
-\partial_\theta J(\theta)_{|\theta=\theta_t}=
\gamma \E_{s\sim \rho,\,s'\sim P(s,\d s')}\, \partial_\theta
m_{\theta_t}(s,s')
\\\mathbin{+}
\gamma \E_{s\sim \rho,\,s''\sim \rho, \,s'\sim P(s'',\d s')}
[m_{\theta_t}(s,s'')\,\partial_\theta m_{\theta_t}(s,s')
]
-\E_{s\sim \rho,\,s'\sim \rho} [m_{\theta_t}(s,s')\,\partial_\theta
m_{\theta_t}(s,s')]
\end{multline}
which yields the expression in the theorem after renaming variables. The
proof for $\tildem$ is similar, using the loss $J'$ instead of $J$ as in
the proof of Theorem~\ref{thm:paramtd}.

\paragraph{Proof of Propositions~\ref{prop:targetrewards} and
\ref{prop:targetbellman}.} The pushforward by $\phi$ of a
measure $\mu$ is the unique measure $\mu^\phi$ such that, for any
function $f$, one has $\int f(g) \mu^\phi(\d g)=\int f(\phi(s))\mu(\d
s)$.

For Proposition~\ref{prop:targetrewards}, assume that the reward function
at a state $s$ is equal to $R(\phi(s))$. By definition of the successor
state operator $M$, the corresponding value function satisfies
$V(s)=\int_{s'} R(\phi(s'))M(s,\d s')$. By definition of the pushforward
measure, the latter is equal to $\int_g R(g)M^\phi(s,\d g)$. If
$M^\phi(s,\d g)$ is equal to $m^\phi(s,g)\tau(\d g)$ for some probability
distribution $\tau$, this rewrites as $\E_{g\sim \tau} m^\phi(s,g)R(g)$.
This proves Proposition~\ref{prop:targetrewards}.

For Proposition~\ref{prop:targetbellman}, just start with the Bellman
equation for $M$: $M(s,\d s_2)=\delta_s(\d s_2)+\gamma \E_{s'\sim P(\d
s'|s)}M(s',\d s_2)$. Then take the pushforward by $\phi$ on both sides,
using that the pushforward of measures is linear. Finally, use that the
pushforward of the Dirac mass at $s$ is the Dirac mass at $\phi(s)$. This
provides the Bellman equation for $M^\phi$.

\paragraph{Proof of Theorem~\ref{thm:targetfeatures}.} The proof is
entirely analogous to the proof of Theorem~\ref{thm:paramtd} for the
model $\tildem$.

\subsection{Proofs for Section~\ref{sec:goals}: Goal-Dependent Methods}

\paragraph{Proof of Theorem~\ref{thm:paramq}.} The proof is very similar
to that of Theorem~\ref{thm:paramtd} and is omitted.
Theorem~\ref{thm:paramq} can also be
obtained as a particular case of Theorem~\ref{thm:goalTD} applied to the
state-action process.

\paragraph{Proof of Theorem~\ref{thm:goalTD}.} 
Proceed similarly to Theorem~\ref{thm:paramtd}.
Define a norm on $V(s,\d
g)$
similarly to
\eqref{eq:norm}, 
as the $L^2$ norm of its density with respect to $\rho_G$:
\begin{equation}
\norm{V(s,\d g)}^2_{\rho_{SG},\rho}\deq \E_{(s,g)\sim \rho_{SG}}
\left [ \frac{V(s,\d g)}{\rho_G(\d g)}^2 \right].
\end{equation}

Let $v_\theta(s,g)$ be any smooth parametric model,
and set $V_\theta(s,\d g)\deq v_\theta(s,g)\rho_G(\d g)$.

Let $\theta_0$ be some
value of the parameter $\theta$, and define a target update $V^\tar$ via
the Bellman equation \eqref{eq:goaldependentV}:
\begin{equation}
V^{\tar}(s,\d g)\deq \alpha(s,g)\,\delta_{\phi(s)}(\d g)+\gamma \E_{s'\sim P(\d
s'|s,g)} V_{\theta_0}(s',\d g).
\end{equation}
For any parameter $\theta$, define the loss
\begin{equation}
J(\theta)\deq \frac12 \norm{V_\theta-V^{\tar}}^2_{\rho_{SG},\rho}.
\end{equation}

Then, as in Theorem~\ref{thm:paramtd} one finds
\begin{multline}
-\partial_\theta J(\theta)
=\iint_{s,g} \rho_{SG}(\d s,\d g) \,
\frac{\partial_\theta V_\theta(s,\d g)}{\rho_G(\d g)}\,\frac{V^\tar(s,\d
g)-V_\theta(s,\d g)}{\rho_G(\d g)}
\\
=\iint_{s,g} \rho_{SG}(\d s,\d g)\,\partial_\theta v_\theta(s,g)
\left(
\frac{\alpha(s,g)\,\delta_{\phi(s)}(\d g)}{\rho_G(\d g)} \right.
\\+ \left. \gamma \E_{s'\sim P(\d s'|s,g)} v_{\theta_0}(s',\d g)-v_\theta(s,\d g)
\right)
\end{multline}

The second part of this equation matches the Bellman gap part of the TD
update in the statement of the theorem, with $\theta_0=\theta$.  (This
also provides the TD update with an arbitrary target network defined by
$\theta_0$.)

For the first part  with the Dirac
term, remember that $\alpha(s,g)=\rho_S(\d g)\rho_G(\d g)/\rho_{SG}(\d
s,\d g)$. Thus,
\begin{multline}
\iint_{s,g} \rho_{SG}(\d s,\d g) \,\partial_\theta
v_\theta(s,g)\,\frac{\alpha(s,g)\,\delta_{\phi(s)}(\d g)}{\rho_G(\d g)}
\\=\int_s \rho_S(\d s) \int_g \partial_\theta
v_\theta(s,g)\,\delta_{\phi(s)}(\d g)
=\int_s \rho_S(\d s) \, \partial_\theta v_\theta(s,\phi(s))
= \E_{s\sim \rho_S}\,\partial_\theta v_\theta(s,\phi(s))
\end{multline}
as needed. This proves that the TD update is as announced in the
statement.

Obviously, $\alpha=1$ when $s$ and $g$ are independent.

For the statement about $\phi=\Id$, note that the Bellman equation only
depends on the value of $\alpha$ on pairs $(s,g)$ such that $\phi(s)=g$.
Therefore, if the statement holds for some function $\alpha$, then it
also holds for any other function $\alpha'$ such that
$\alpha'(s,g)=\alpha(s,g)$ when $\phi(s)=g$, because this will define the
same $V^\tar$. With $\phi=\Id$, this means
that the statement holds for any other function $\alpha'$ with
$\alpha'(g,g)=\alpha(g,g)$. Define $\alpha'(s,g)\deq
\alpha(g,g)$.
Then $\alpha'(g,g)=\alpha(g,g)$, and $\alpha'$ only depends on $g$.
This completes the proof.

\newcommand{\Qspace}{\mathcal{Q}}

\paragraph{Proof of Theorem~\ref{thm:optQ}.} Assume the action space $A$
is countable. Let $\Qspace$ be the set of
measurable functions from $S\times A$ to the set of measures on $S$. 

For $Q_1$ and $Q_2$ in $\Qspace$, we write $Q_1\leq Q_2$ if
$Q_1(s,a,X)\leq Q_2(s,a,X)$ for any state-action $(s,a)$ and measurable
set $X\subset S$.
The
Bellman operator of Definition~\ref{def:optQ} acts on $\Qspace$ and is
obviously monotonous: if $Q_1\leq Q_2$ then $TQ_1\leq TQ_2$. 

Since the zero measure $\zerom\in \Qspace$ is the smallest measure, we have
$T\zerom\geq \zerom$. Since $T$ is monotonous, by induction we have
$T^{t+1}\zerom\geq T^t\zerom$ for any $t\geq 0$. Thus, the
$(T^t\zerom)_{t\geq 0}$ form an increasing sequence of measures.
Therefore, for every state-action $(s,a)$ and measurable set $X$, the
sequence $(T^t\zerom)(s,a,X)$ is increasing, and thus converges to a
limit. We denote this limit by $Q^\ast(s,a,X)$. We have to prove that
$Q^\ast\in \mathcal{Q}$, namely, that for each $(s,a)$, $Q^\ast(s,a,\cdot)$ is
a measure. The only non-trivial point is $\sigma$-additivity.

Denote
$Q_t\deq T^t\zerom$.
If $(X_i)$
is a countable collection of disjoint measurable sets, we have
\begin{multline}
Q^\ast(s,a,\cup_i X_i)=\lim_{t\to\infty} Q_t(s,a,\cup_i X_i)=\lim_{t\to\infty} \sum_i
Q_t(s,a,X_i)\\=\sum_i \lim_{t\to\infty} Q_t(s,a,X_i)=\sum_i Q^\ast(s,a,X_i) 
\end{multline}
where the limit commutes with the sum thanks to the monotone convergence
theorem, using that $Q_t$ is non-decreasing. Therefore, $Q^\ast$ is a
measure.

Let us prove that $TQ^\ast=Q^\ast$. We have
\begin{equation}
TQ^\ast(s,a,\cdot)=\delta_s+\gamma \E_{s'\sim P(s'|s,a)}\sup_{a'}
Q^\ast(s',a',\cdot)
\end{equation}
by definition. For any $s'$, denote $\tilde
Q_t(s',\cdot)\deq\sup_{a'} Q_t(s',a',\cdot)$ where the supremum is as
measures over $S$. Since $Q_t$ is non-decreasing, so is $\tilde Q_t$.

For any state $s'$, we have
\begin{equation}
\sup_{a'}
Q^\ast(s',a',\cdot)=\sup_{a'} \sup_t Q_t(s',a',\cdot)=
\sup_t \sup_{a'} Q_t(s',a',\cdot)=\sup_t \tilde Q_t(s',\cdot)
\end{equation}
since supremums commute. Now, since $\tilde Q_t$ is non-decreasing,
thanks to the monotone convergence theorem, the supremum commutes with
integration over $s'\sim P(s'|s,a)$ (which does not depend on $t$), namely,
\begin{multline}
\E_{s'\sim P(s'|s,a)}\sup_{a'}
Q^\ast(s',a',\cdot)=
\E_{s'\sim P(s'|s,a)} \sup_t \tilde Q_t(s',\cdot)
\\=
\sup_t \E_{s'\sim P(s'|s,a)}\tilde Q_t(s',\cdot)
=
\sup_t \E_{s'\sim P(s'|s,a)} \sup_{a'} Q_t(s',a',\cdot)
\end{multline}
and so $TQ^\ast=\sup_t TQ_t$. Now, since $Q^t=T^t\zerom$, we have
$TQ^t=T^{t+1}\zerom$, so that $\sup_{t\geq 0} TQ^t=\sup_{t\geq
1}T^t\zerom=Q^\ast$. So $Q^\ast$ is a fixed point of $T$.

Let us prove that $Q^\ast$ is the smallest such fixed point. Let $Q'$
such that $TQ'=Q'$. Since $\zerom\leq Q'$ and $T$ is monotonous, we have
$T\zerom \leq TQ'=Q'$. By induction, $T^t\zerom\leq Q'$ for any $t\geq
0$. Therefore, $\sup_t T^t\zerom \leq Q'$, i.e., $Q^\ast\leq Q'$.

The statement for finite state spaces reduces to the 
classical uniqueness property of
the usual $Q$ function, separately for each goal state.

\subsection{Examples of MDPs with Infinite Mass for $Q^\ast$}
\label{sec:exqopt}

Here are two simple examples of MDPs with finite action space, for which
the mass of the goal-dependent $Q$-function $Q^\ast(s,a,s_2)$ is infinite.
The first has discrete states, the second, continuous ones.

Take for $S$ an infinite rooted dyadic tree, namely, $S=\{\emptyset,
0,1,00,01,\ldots\}$ the set of binary strings of finite length $k\geq 0$.
Consider the two actions ``add a $0$ at the end'' and ``add a $1$ at the
end''. Then, for every state $s$, $Q^\ast(s,a,\cdot)$ is a measure that
gives mass $\gamma^k$ to all states $s_2$ that are extensions of $s$ by a
length-$k$ string that starts with $a$. Thus, its mass is $1+\sum_{k\geq
1} \gamma^k 2^{k-1}$. This is infinite as soon as $\gamma\geq 1/2$. This
extends to any number of actions by considering higher-degree trees.

A similar example with continuous states is obtained as follows. Let
$S=[0;1)\times [0;1)$.
Let
$C=\{\emptyset,
0,1,00,01,\ldots\}$ the dyadic tree above. For each string $w\in X$,
consider the set $B_w\subset S$ defined as follows: $B_w$ is made of
those points $(x,y)\in S$ such that the binary expansion of $x$ starts
with $w$, and $y\in [1-1/2^k;1-1/2^{k+1})$ where $k$ is the length of
$w$. Graphically, this creates a tree-like partition of the square $S$,
where the empty string corresponds to the bottom half, the strings $w=0$
and $w=1$ correspond to two sets on the left and right above the bottom
hald, etc. Define the following MDP with two actions $0$ and $1$: with
action $0$, every
state $s\in B_w$ goes to a uniform random state in $B_{w0}$, and with
action $1$, every state $s\in B_w$ goes to a uniform random state in
$B_{w1}$. The goal-dependent $Q$-function $Q^\ast$ is similar to the
dyadic tree above, but is continuous. Its mass is infinite for the same
reasons.

\subsection{Proofs for Sections~\ref{sec:BN} and~\ref{sec:V}:
Second-Order Methods}

\paragraph{Proof of Theorem~\ref{thm:onlineM}.} 
Define $\hat M\deq (\Id-\gamma \hat P)^{-1}$ where $\hat P$ is updated by
\eqref{eq:Pupdate}.
The update
\eqref{eq:Pupdate} can be rewritten as $\hat P \gets \hat
P+(1/n_s) \1_s (\transp{\1_{s'}}-\transp{\1_s}\hat P)$. This is a
rank-one update of $\hat P$. The update of $\Id-\gamma \hat P$ is $-\gamma$
times the update of $\hat P$, and is still rank-one: it is equal to
$u\transp{v}$ with $u\deq -(\gamma/n_s)\1_s$ and $\transp{v}\deq
(\transp{\1_{s'}}-\transp{\1_s}\hat P)$. The
Sherman--Morrison formula gives the update
of the inverse of a matrix after a rank-one update. By this formula, the update of
$\hat M=(\Id-\gamma \hat P)^{-1}$ is
\begin{align}
\hat M\gets & \hat M-\frac{\hat Mu\transp{v}\hat M}{1+\transp{v}\hat Mu}
=
\hat M +\frac{1}{n_s} \frac{\hat M\1_s
(\gamma \transp{\1_{s'}}-\gamma \transp{\1_s}\hat P)\hat M}
{1-\frac{1}{n_s}(\gamma \transp{\1_{s'}}-\gamma \transp{\1_s}\hat P)\hat M \1_s
}
\end{align}
Now, since $\hat M=(\Id-\gamma \hat P)^{-1}$, we have $\gamma \hat P\hat
M=\hat M-\Id$. Therefore, the terms $(\gamma \transp{\1_{s'}}-\gamma
\transp{\1_s}\hat P)\hat M$ are equal to $\gamma \transp{\1_{s'}}\,\hat
M-\transp{\1_s}\hat M+\transp{\1_s}$, and the update is
\begin{align}
\hat M\gets & 
\hat M +\frac{1}{n_s} \frac{\hat M\1_s(
\gamma \transp{\1_{s'}}\,\hat M-\transp{\1_s}\hat
M+\transp{\1_s})
}
{1-\frac{1}{n_s}(
\gamma \transp{\1_{s'}}\,\hat
M-\transp{\1_s}\hat M+\transp{\1_s}
)\1_s
}
\\&=
\label{eq:trueonlineM}
\hat M +\frac{1}{n_s} \frac{\hat M\1_s(
\gamma \transp{\1_{s'}}\,\hat M-\transp{\1_s}\hat
M+\transp{\1_s})
}
{1-\frac{1}{n_s}(
\gamma \hat M_{s' s}-\hat M_{ss}+1)
}
\end{align}
which is the exact update of $\hat M$.
This provides the update \eqref{eq:onlineM}.

The value function $\hat V$ of the estimated process is $(Id-\gamma \hat
P)^{-1}\hat R=\hat M\hat R$.
When $\hat M\gets \hat M+\del M$ and $\hat R\gets \hat R+\del R$ one
has $\hat V\gets \hat V+\hat M \del R +\del M \hat R+ \del M \del R$.
From \eqref{eq:Pupdate} we have
$\del R=\frac{1}{n_s}(r_s-\hat R_s)\1_s=\frac{1}{n_s}(r_s\1_s-\1_s
\transp{\1_s}\hat R)$.
Plugging in the value of $\del M$ from \eqref{eq:trueonlineM}, keeping
only first-order terms in $1/n_s$,
and using
$\transp{\1_s}\hat M\hat R=\transp{\1_s}\,\hat V=\hat V_s$, provides the update of $\hat V$ in
\eqref{eq:onlineV}.

\paragraph{Proof of Theorem~\ref{thm:onlineM_exp}.} First, note that the
expectation in the statement is averaged over the next step, but
conditional to all quantities $\hat M$, $\hat V$, etc., computed in the previous
steps. In this proof, we will just write $\E$ for short.

Since the
the denominator in \eqref{eq:trueonlineM} is $1+o(1/n_s)$, the update
\eqref{eq:trueonlineM} of $\hat M$ is
$\hat M\gets \hat M+\del M$ with
\begin{equation}
\del M=\frac{1}{n_s} \hat M\1_s(
\gamma \transp{\1_{s'}}\,\hat M-\transp{\1_s}\hat
M+\transp{\1_s})
+o(1/n_s).
\end{equation}
We want to compute the expectation of this update when $s$ is sampled
from $\rho$ and $s'$ from $P_{ss'}$. 
This yields
\begin{align}
\E [\del M]&=\sum_{s,s'} \rho_s P_{ss'} \frac{1}{n_s}\hat M\1_s(
\gamma \transp{\1_{s'}}\,\hat M-\transp{\1_s}\hat
M+\transp{\1_s})
+o(1/n_s)
\\&=\frac1t \sum_{s,s'} P_{ss'} \hat M\1_s(
\gamma \transp{\1_{s'}}\,\hat M-\transp{\1_s}\hat
M+\transp{\1_s})
+o(1/t)
\end{align}
where the last equality holds because $n_s=t \rho_s +o(t)$ by the law of
large numbers (since $s$ is sampled from $\rho$). Now, we have
$
\sum_{s,s'} P_{ss'}\1_s\transp{\1_{s'}}=P$ and
$\sum_{s,s'} P_{ss'} \1_s\transp{\1_{s}}=\sum_s \1_s\transp{\1_{s}}=\Id$.
Thus,
\begin{align}
\E [\del M]&=
\frac1t \hat M(
\gamma P\hat M-\hat
M+\Id)
+o(1/t)
\end{align}
as needed.

To compute the update of $\hat V=\hat M \hat R$, let us first compute the
update of $\hat R$. By \eqref{eq:Pupdate}, the latter is $\hat R\gets
\hat R+\del R$ with
\begin{equation}
\del R=\frac1{n_s}(r_s-\hat R_s)\1_s=\frac{1}{t\rho_s} (r_s-\hat
R_s)\1_s+o(1/t).
\end{equation}
Now, the update of $\hat V=\hat M \hat R$ is $\del V=\del M \hat R+ \hat
M \del R +\del M \del R$. The last term $\del M \del R$ is $O(1/t^2)$, so
we can drop it. We find
\begin{align}
\E[\del V]&=\E[\del M \hat R]+ \E[\hat
M \del R]+o(1/t)
\\&=\E[\del M]\hat R+\hat M \E[\del R]+o(1/t)
\end{align}
since the expectations are averaged over the next step but conditional on the previous steps, which
comprises the previous values of $\hat R$ and $\hat M$. Next,
\begin{align}
\E[\del M]\hat R
&=
\frac1t \hat M(
\gamma P\hat M-\hat
M+\Id)\hat R
+o(1/t)
\\&=\frac1t \hat M(\gamma P \hat V-\hat V+\hat R)+o(1/t)
\end{align}
since $\hat V=\hat M \hat R$. Next,
\begin{align}
\hat M \E[\del R]
&=\hat M \sum_s \rho_s \frac{1}{t\rho_s} (\E[r_s]-\hat R_s)\1_s+o(1/t)
\\&=\frac1t \hat M (R-\hat R)+o(1/t)
\end{align}
since $\sum_s \E[r_s]\1_s=R$ and $\sum_s \hat R_s \1_s=\hat R$. Summing,
we find $\E[\del V]=\frac1t \hat M(\gamma P\hat V-\hat V+\hat R+R-\hat R)+o(1/t)$ as
needed.

\paragraph{Proof of Theorem~\ref{thm:paramBN}.} First, note that we
expressed this theorem for a single transition $s\to s'$, while we
expressed the similar theorem for TD using the Bellman operator
$M^\tar=\Id+\gamma PM$, which is
the sum of the single-transition update for all values of $s$.

This is because the single-transition update is more informative in this
case, especially given the exact update of $M$ in
Theorem~\ref{thm:onlineM}. (At first, we worked at the operator level,
and found a parametric expression which was the same in expectation
over transitions $s\to s'$, but did not correspond to a single-transition
update, had a larger variance, and performed much worse in practice.)

However, the single-transition updates \eqref{eq:Pupdate} and \eqref{eq:onlineM} only make sense
in a discrete-state setting. Thus, to best preserve the information from
observing a single transition, 
we state and derive
Theorem~\ref{thm:paramBN} in a discrete-state setting. The
resulting parametric update makes sense for continuous
states.

(In Appendix~\ref{sec:formalsmproof} we rigorously derive this same update for
continuous states, in expectation over $s \to s'$, as we did for TD. The
analogue of the Bellman operator $\del M=\Id +\gamma PM-M$ for implicit process
updating is the Newton--Bellman operator $\del
M=M-M(\Id-\gamma P)M$
of Definition~\ref{def:BN}.)

Thus, let us work in a discrete setting, using matrix notation.
Let us consider
 $\del M$ and $\del V$ given by \eqref{eq:onlineM}--\eqref{eq:onlineV}.
For simplicity we omit the $o(1/n_s)$ terms in
\eqref{eq:onlineM}--\eqref{eq:onlineV}: they
are absorbed in the $o(1/t)$ of the final statement of the theorem, because
$n_s\sim \rho(s) t$ by the law of large numbers. (Indeed, $\rho(s)$ is
\emph{defined} as
the probability to sample a transition from $s$ in our data model.)

With $M^\tar\deq M_{\theta_t} +\del M$ and from the
definition \eqref{eq:norm} of $\norm{\cdot}_\rho$, we obtain
\begin{align}
\norm{M_\theta-M^\tar}_\rho^2&=
\E_{s_1\sim \rho,\,s_2\sim\rho}\,
\frac{(M_\theta-M^\tar)_{s_1s_2}^2}{\rho(s_2)^2}
\\&=\sum_{s_1,s_2} \frac{\rho(s_1)}{\rho(s_2)}
(M_\theta-M^\tar)_{s_1s_2}^2
\end{align}
so that the gradient step on the loss is
\begin{align}
-\partial_\theta J(\theta)&=
-\partial_\theta \left({\textstyle\frac12}
\norm{M_\theta-M^\tar}_\rho^2\right)
\\&=
\sum_{s_1,s_2} \frac{\rho(s_1)}{\rho(s_2)}
\,(M^\tar-M_\theta)_{s_1s_2}
\,\partial_\theta (M_{\theta})_{s_1s_2}
\end{align}
and we compute the gradient step at $\theta=\theta_t$; since
$M^\tar-M_{\theta_t}=\del M$, we get
\begin{equation}
-\partial_\theta J(\theta)_{\theta=\theta_t}=
\sum_{s_1,s_2} \frac{\rho(s_1)}{\rho(s_2)} \,(\del M)_{s_1s_2}\,
\partial_\theta (M_{\theta})_{s_1s_2}.
\end{equation}
Now, remember that the parameterization is
$(M_\theta)_{s_1s_2}=\1_{s_1=s_2}+m_\theta(s_1,s_2)\rho(s_2)$.
We obtain
\begin{equation}
\partial_\theta (M_{\theta})_{s_1s_2}=\partial_\theta
m_\theta(s_1,s_2)\rho(s_2)
\end{equation}
and from the expression \eqref{eq:onlineM} for $\del M$, up to
$O(1/n_s^2)$ terms,
we have
\begin{equation}
-\partial_\theta J(\theta)_{\theta=\theta_t}=
\sum_{s_1,s_2}\frac{\rho(s_1)}{\rho(s_2)}
\frac{1}{n_s}(M_{\theta_t})_{s_1s}\left(\1_{s_2=s}+\gamma
(M_{\theta_t})_{s's_2}-(M_{\theta_t})_{ss_2}\right)\,\partial_\theta
m_{\theta_t}(s_1,s_2)\rho(s_2)
\end{equation}
so that $\rho(s_2)$ cancels out. Now
let us expand
$(M_{\theta_t})_{s_1s_2}=\1_{s_1=s_2}+m_{\theta_t}(s_1,s_2)\rho(s_2)$
into this
expression. We have
\begin{equation}
\1_{s_2=s}+\gamma
(M_{\theta_t})_{s's_2}-(M_{\theta_t})_{ss_2}=\1_{s'=s_2}+\gamma
m_{\theta_t}(s',s_2)\rho(s_2)-m_{\theta_t}(s,s_2)\rho(s_2)
\end{equation}
and after tediously collecting all terms, we arrive at
\begin{multline}
\label{eq:tedious}
-\partial_\theta J(\theta)_{\theta=\theta_t}=
\sum_{s_1,s_2} \frac{\rho(s_1)}{n_s} \1_{s_1=s}\,\1_{s'=s_2}
\, \partial_\theta m_{\theta_t}(s_1,s_2)
\\\mathbin{+}
\sum_{s_1,s_2} \frac{\rho(s_1)}{n_s} m_{\theta_t}(s_1,s)
\rho(s)\,\1_{s'=s_2}
\, \partial_\theta m_{\theta_t}(s_1,s_2)
\\\mathbin{+}
\sum_{s_1,s_2} \frac{\rho(s_1)}{n_s} \1_{s_1=s}
\left(\gamma m_{\theta_t}(s',s_2)-m_{\theta_t}(s,s_2)\right)\rho(s_2)
\, \partial_\theta m_{\theta_t}(s_1,s_2)
\\\mathbin{+}
\sum_{s_1,s_2} \frac{\rho(s_1)}{n_s} m_{\theta_t}(s_1,s)\,\rho(s)
\left(\gamma m_{\theta_t}(s',s_2)-m_{\theta_t}(s,s_2)\right)\rho(s_2)
\, \partial_\theta m_{\theta_t}(s_1,s_2).
\end{multline}
The first term rewrites
\begin{equation}
\sum_{s_1,s_2} \frac{\rho(s_1)}{n_s} \1_{s_1=s}\,\1_{s'=s_2}
\, \partial_\theta m_{\theta_t}(s_1,s_2)
=\frac{\rho(s)}{n_s} \partial_\theta m_{\theta_t}(s,s').
\end{equation}
The second one rewrites
\begin{equation}
\sum_{s_1,s_2} \frac{\rho(s_1)}{n_s} m_{\theta_t}(s_1,s)
\rho(s)\,\1_{s'=s_2}
\, \partial_\theta m_{\theta_t}(s_1,s_2)
=\frac{\rho(s)}{n_s}\, \E_{s_1\sim \rho} [
m_{\theta_t}(s_1,s)\,\partial_\theta m_{\theta_t}(s_1,s')].
\end{equation}
Similarly, the third term in \eqref{eq:tedious} rewrites as
\begin{equation}
\frac{\rho(s)}{n_s} \,\E_{s_2\sim \rho} \left(\gamma
m_{\theta_t}(s',s_2)-m_{\theta_t}(s,s_2)\right)\partial_\theta
m_{\theta_t}(s,s_2)
\end{equation}
and the fourth as
\begin{equation}
\frac{\rho(s)}{n_s} \,\E_{s_1\sim\rho,\,s_2\sim \rho} \left(\gamma
m_{\theta_t}(s',s_2)-m_{\theta_t}(s,s_2)\right)\partial_\theta
m_{\theta_t}(s_1,s_2).
\end{equation}
Now, by definition of $\rho$ in our data model, $\rho(s)$ is
the frequency with which a
transition starting at $s$ is sampled. Therefore, by the law of large
numbers, $n_s\sim t\rho(s)$ when $t\to \infty$. Therefore,
\begin{equation}
\frac{\rho(s)}{n_s}=1/t+o(1/t)
\end{equation}
when $t\to\infty$. (This is the advantage of defining all norms with
respect to $\rho$; anyway, in a general setting, $\rho$ is the only
available measure on $\mathcal{S}$ to define norms with.)

Thus, when $t\to \infty$,
\begin{multline}
-\partial_\theta J(\theta)_{\theta=\theta_t}=
\frac1t
\left(
\partial_\theta m_{\theta_t}(s,s')
+
\E_{s_1\sim \rho} [
m_{\theta_t}(s_1,s)\,\partial_\theta m_{\theta_t}(s_1,s')]
\right.\\\left.
+
\E_{s_2\sim \rho} \left(\gamma
m_{\theta_t}(s',s_2)-m_{\theta_t}(s,s_2)\right)\partial_\theta
m_{\theta_t}(s,s_2)
\right.\\\left.
+\E_{s_1\sim\rho,\,s_2\sim \rho} \left(\gamma
m_{\theta_t}(s',s_2)-m_{\theta_t}(s,s_2)\right)\partial_\theta
m_{\theta_t}(s_1,s_2)
\right)
+o(1/t)
\end{multline}
which is the expression \eqref{eq:Mparam} given in Theorem~\ref{thm:paramBN}.

This ends the proof of Theorem~\ref{thm:paramBN} for discrete states,
which is the only setting in which the single-transition update $\del M$
makes sense. Yet the expressions obtained also make sense for continuous
states.
Appendix~\ref{sec:formalsmproof}
contains a rigorous derivation for continuous states, in
expectation over $s\to s'$, as we did for parametric TD.

%
\paragraph{Proof of Proposition~\ref{prop:VfromM}.}
Given a parametric model $V_\phi$
with parameter $\phi$, at each step $t$,
define a target $V^\tar\deq
V_{\phi_t}+\del V$ with $\del V$ given by \eqref{eq:onlineV}. As for
Theorem~\ref{thm:paramBN}, the update \eqref{eq:onlineV} is defined via
a single transition $s\to s'$ and only makes sense in a discrete space,
as does $V^\tar$.
So we work with a parametric model on a discrete space and observe that
the resulting update is well-defined in continuous spaces. (As with
Theorem~\ref{thm:paramBN}, continuous spaces can be treated rigorously by
considering the expectation over $s\to s'$, see
Appendix~\ref{sec:formalsmproof}.)

The
loss function on $\phi$ is $J^V(\phi)\deq \frac12
\norm{V_\phi-V^\tar}^2_{L^2(\rho)}$. Then
\begin{align}
-\partial_\phi J^V(\phi)&=\sum_{s_1} \rho(s_1)\,
(V_{\phi_t}(s_1)+\del V_{s_1}-V_\phi(s_1))\,\partial_\phi V_\phi(s_1)
\end{align}
and so
\begin{equation}
-\partial_\phi J^V(\phi)_{|\phi=\phi_t}=\sum_{s_1} \rho(s_1) \del V_{s_1}
\,\partial_\phi V_{\phi_t}(s_1).
\end{equation}
Plugging in the expression \eqref{eq:onlineV} for $\del V$ (with $\hat
V_s=V_{\phi_t}(s)$ and $\hat M_{s_1s_2}=M_{\theta_t}(s_1,s_2)$) yields, again omitting the $o(1/n_s)$ terms,
\begin{equation}
-\partial_\phi J^V(\phi)_{|\phi=\phi_t}=
(r_s+\gamma V_{\phi_t}(s')-V_{\phi_t}(s))
\sum_{s_1} \frac{\rho(s_1)}{n_s}
M_{\theta_t}(s_1,s)\,\partial_\phi V_{\phi_t}(s_1)
\end{equation}
and plugging in the parametric model
$M_{\theta_t}(s_1,s)=\1_{s_1=s}+m_{\theta_t}(s_1,s)\rho(s)$ yields
\begin{align}
-\partial_\phi J^V(\phi)_{|\phi=\phi_t}&=(r_s+\gamma
V_{\phi_t}(s')-V_{\phi_t}(s))
\sum_{s_1}
\frac{\rho(s_1)}{n_s}(\1_{s_1=s}+m_{\theta_t}(s_1,s)\rho(s))\,\partial_\phi
V_{\phi_t}(s_1)
\\&=
(r_s+\gamma
V_{\phi_t}(s')-V_{\phi_t}(s))
\left(
\frac{\rho(s)}{n_s} \,\partial_\phi
V_{\phi_t}(s)+\frac{\rho(s)}{n_s} \E_{s_1\sim \rho} [
m_{\theta_t}(s_1,s)\partial_\phi
V_{\phi_t}(s_1)
]
\right)
\end{align}
and as above, $\rho(s)/n_s=1/t+o(1/t)$ when $t\to \infty$, so this yields
the expression \eqref{eq:Vparam} in the theorem.

\paragraph{Proof of Theorem~\ref{thm:convergence}.}

This convergence analysis is partially inspired by \citep{Pananjady2019ValueFE}. The main differences are the data model and the metrics computed.

We assume that $\rho$ is an invariant probability measure of $P$, and
that the reward is bounded by $R_{\max}$ with probability $1$. We define
the empirical distribution of states $\hat \rho_{t}$ as:
$(\hat\rho_{t})_{s} = \frac{n_{s}}{t}$, with $n_{s}$ the number of visits
to state $s$ up to time $t$. We also consider $\hat P_{t}$ and $\hat
R_{t}$ as defined in \eqref{eq:Pupdate}. 

The initialization of $\hat P$ and $\hat R$ does not matter, as it is erased the first time a
state is visited. To fix ideas, we initialize $\hat P$ and $\hat R$ to
$0$; this helps if $\rho=0$ for some states. (In particular, $\hat P$ may
be substochastic: $0\leq \sum_j \hat P_{ij}\leq 1$ for all $i$.)

We define $\widehat{\rho P}_{t}$ as the empirical distribution of
transitions: $(\widehat{\rho P}_{t})_{s_{1}s_{2}} \deq
\frac{n_{s_{1}s_{2}}}{t}$ where $n_{s_{1}s_{2}}$ is the number of
observations of a transition $(s_{1}, s_{2})$ up to time $t$. We have $(\hat
P_{t})_{s_{1}s_{2}} = \frac{n_{s_{1}s_{2}}}{n_{s1}}$ if $n_{s1} > 0$, or $0$
if $n_{s1} = 0$. Hence $\left(\widehat{\rho P}_{t}\right)_{s_{1}s_{2}} =
\hat\rho_{s_{1}}(\hat P_{t})_{s_{1}s_2}$.

The proof strategy is to bound the errors $\|\hat M -
M\|_{\rho,\mathrm{TV}}$ and $\|\hat V - V\|_{\rho}$ 
by errors on $\widehat{\rho P}$ and $\hat R$. The error on
$\widehat{\rho P}$ can then be controlled by concentration inequalities
on empirical distributions, and the error on $\hat R$ can be bounded via
the Hoeffding inequality.

The successor state operator estimate $\hat M$ is $(\Id - \gamma \hat
P)^{-1}$. By the Bellman equation for $M$ and $\hat M$,
\begin{align}
  \hat M - M &=   \gamma \hat P \hat M - \gamma P M
  \\ &= \gamma P (\hat M-M)  + \gamma (\hat P - P) \hat M
\end{align}
and therefore
\begin{equation}
    (\Id - \gamma P)(\hat M - M) =   \gamma (\hat P - P) \hat M
\end{equation}
and thus
\begin{equation}
    \hat M - M =   \gamma M (\hat P - P) \hat M
  \end{equation}
  by definition of $M$.
  
  Therefore, 
  \begin{align}
    \|\hat M - M\|_{\rho,\mathrm{TV}} &= \gamma \|M (\hat P - P) \hat M\|_{\rho,\mathrm{TV}}
    \\ &= \frac{\gamma}{2} \sum_{i,j}\rho_{i} \abs{\sum_{k,l} M_{ik}(\hat
    P-P)_{kl}\hat M_{lj}}
         \\ &\leq \frac{\gamma}{2} \sum_{i,j,k,l}\rho_{i}
	 M_{ik}|\hat P-P|_{kl}\hat M_{lj}.
  \end{align}



We know that $(1-\gamma)M$ is a stochastic matrix, and $\rho$ is an
invariant probability measure. Therefore, $\sum_{i}\rho_{i}M_{ik} =
\frac{1}{1-\gamma}\rho_{k}$. Moreover, if $\hat P$ is sub-stochastic,
$\sum_{j}\hat M_{lj} \leq \frac{1}{1-\gamma}$ (with equality if $P$ is
stochastic). Therefore,
\begin{equation}
  \|\hat M - M\|_{\rho,\mathrm{TV}} \leq
  \frac{\gamma}{(1-\gamma)^{2}}\|\hat P-P\|_{\rho,\mathrm{TV}}.
\end{equation}
We define $(\rho P)$ as the matrix $\mathrm{Diag}(\rho)P$. We now bound
the error $\|\hat P-P\|_{\rho,\mathrm{TV}}$ by the error $\|\widehat{\rho
P} - (\rho P)\|_{\mathrm{TV}}$, in order to use standard concentration
inequalities on empirical distributions:
\begin{align}
  \|\hat P-P\|_{\rho,\mathrm{TV}} &= \frac12\|\mathrm{Diag}(\rho) \hat P-(\rho P)\|_{1}
  \\ &\leq \|\widehat{\rho P}-(\rho P)\|_{\mathrm{TV}}  + \frac12\|\mathrm{Diag}(\hat\rho - \rho) \hat P\|_{1}
  \\ &\leq \|\widehat{\rho P}-\rho P\|_{\mathrm{TV}}  + \|\hat\rho - \rho\|_{\mathrm{TV}}
  \\ &\leq \|\widehat{\rho P}-\rho P\|_{\mathrm{TV}}  + \frac12\sum_{i}|\sum_{j}\hat\rho_{i}\hat P_{ij} - \rho_{i} P_{ij}|
  \\ &\leq \|\widehat{\rho P}-\rho P\|_{\mathrm{TV}}  + \frac12\sum_{i,j}|\hat\rho_{i}\hat P_{ij} - \rho_{i} P_{ij}|
       \\ &\leq 2\|\widehat{\rho P}-\rho P\|_{\mathrm{TV}}  
\end{align}
Therefore,
\begin{equation}
  \label{eq:m-bounded-p}
  \|\hat M - M\|_{\rho,\mathrm{TV}} \leq
  \frac{2\gamma}{(1-\gamma)^{2}}\|\widehat{\rho P}-\rho
  P\|_{\mathrm{TV}}.
\end{equation}
We now consider the error on $\hat V$.
We have:
\begin{align}
  \|\hat V - V\|_{\rho} &= \|\hat M\hat R - MR\|_{\rho}
  \\ &\leq \|(\hat M - M)\hat R\|_{\rho} + \| M(\hat R - R)\|_{\rho}
  \\ &\leq  2R_{\max}\|\hat M - M\|_{\rho, \mathrm{TV}} + \frac{1}{1-\gamma}\|\hat R - R\|_{\rho}
  \\ &\leq \frac{4 R_{\max}}{(1-\gamma)^{2}}\|\widehat{\rho P}-\rho P\|_{\mathrm{TV}}  + \frac{1}{1-\gamma}\|\widehat{R} - R\|_{\hat \rho}
       \label{eq:v-bounded-p-r}
\end{align}

We now bound $\|\widehat{\rho P}-\rho P\|_{\mathrm{TV}}$ and $\|\widehat{R} - R\|_{\hat \rho}$.

We can bound the error $\|\widehat{R} - R\|_{\hat \rho}$ with the
Hoeffding inequality. $\widehat{R}_{s}$ is the average of $n_{s}$
independent samples of expectation $R_{s}$. Since the reward is bounded
by $R_{max}$ with probability $1$, we can use Hoeffding's inequality. For any $s$ with $n_{s} > 0$, we have:
\begin{align}
  \Prob(|\widehat{R} - R|_{s} > u) \leq 2\exp\left(-\frac{n_{s}u^{2}}{2R^{2}_{\max}}\right)
\end{align}
Hence, for any $s$ with $n_{s} > 0$, we have with probability $1-\frac{\delta}{S}$:
\begin{equation}
    |\widehat{R_{t}} - R|_{s} \leq R_{\max}\sqrt{\frac{2}{n_{s}}\log{\frac{2S}{\delta}}}
\end{equation}
and since $\hat\rho_{s}=n_s/t$, this rewrites as
\begin{equation}
    \hat\rho_{s}|\widehat{R_{t}} - R|_{s} \leq
    \frac{R_{\max}}{t}\sqrt{2n_{s}\log{\frac{2S}{\delta}}}.
  \end{equation}
  For states with $n_{s}=0$, $\hat\rho_{s}=0$ and the inequality still holds. If for all $s$, $\Prob(|\widehat{R_{t}} - R|_{s} \geq \epsilon_{s}) \leq \frac{\delta}{S}$, then
  \begin{align}
    \Prob(\|\hat R - R\|_{\hat\rho} \geq \sum_{s} \epsilon_{s}) &\leq \sum_{s}\Prob(\hat\rho_{s}|\widehat{R_{t}} - R|_{s} \geq \epsilon_{s})
    \\ &\leq \sum_{s}\frac{\delta}{S} = \delta
  \end{align}
  Thus, with probability $1-\delta$,
  \begin{align}
    \|\widehat{R} - R\|_{\hat \rho} &\leq \frac{R_{\max}}{t}\sqrt{2\log{\frac{2 S}{\delta}}} \sum_{s} \sqrt{n_{s}}
    \\ &\leq \frac{R_{\max}}{t}\sqrt{2\log{\frac{2 S}{\delta}}} \sqrt{\sum_{s} n_{s}}\sqrt{S}
    \\ &\leq \frac{R_{\max}}{\sqrt{t}}\sqrt{2S\log{\frac{2 S}{\delta}}}
         \label{eq:r-hoeffding}
  \end{align}
since $\sum_s n_s=t$.  

We now bound $\|\widehat{\rho P}-\rho P\|_{\mathrm{TV}}$. $\widehat{\rho
P}$ is the empirical distribution over all possible transitions. The set
of all possible transitions is of size $S^{2}$. However, if ${(\rho
P)}_{s_{1}s_{2}} = 0$, then with probability $1$,  $\widehat{\rho
P}_{s_{1}s_{2}} = 0$. Therefore, if $E$ is the number of edges of the MDP
($(s,s')$ is an edge if $P_{ss'}>0$), $\|\widehat{\rho P}-\rho
P\|_{\mathrm{TV}}$ can be bounded by an inequality on the total variation
error of the empirical measure on a set of size $E$. We use Theorem 2.2
from \citep{weissman2003inequalities}\footnote{We use the trivial bound $\phi(\pi) \geq 2$ with the notation of the original paper.}, and have with with probability $1-\delta$:
  \begin{equation}
    \label{eq:tv-error-emp-measure}
    \|\widehat{\rho P}_{t} - \rho P\|_{\mathrm{TV}} \leq \frac{1}{2\sqrt{t}}\sqrt{2 E \log\frac{2}{\delta}}
  \end{equation}

  By plugging equation~(\ref{eq:tv-error-emp-measure}) into
  (\ref{eq:m-bounded-p}), with probability $1-\delta$,
  \begin{align}
    \|\hat M - M\|_{\rho, \mathrm{TV}} &\leq \frac{\gamma}{(1-\gamma)^{2}\sqrt{t}}\sqrt{2 E \log\frac{2}{\delta}}
  \end{align}

  Finally, by plugging  (\ref{eq:r-hoeffding}) and
  (\ref{eq:tv-error-emp-measure}) into (\ref{eq:v-bounded-p-r}), with
  probability $1-\delta$, we obtain
\begin{align}
  \|\hat V - V\|_{\rho} &\leq \frac{2R_{\max}}{(1-\gamma)^{2}}\frac{1}{\sqrt{t}}\sqrt{2 E \log\frac{4}{\delta}}  + \frac{1}{1-\gamma}\frac{R_{\max}}{\sqrt{t}}\sqrt{2S\log{\frac{4 S}{\delta}}}
 \\ &\leq \frac{3 R_{\max}}{(1-\gamma)^{2}}\sqrt{\frac{2E}{t} \log\frac{4S}{\delta}} 
  \end{align}
which ends the proof.

\section{The Bellman--Newton Operator and Path Composition}
\label{sec:bellman-newton}

In Section~\ref{sec:paths}, we explained the link between learning
successor states and counting paths in a Markov process. Here, we
formalize that link, by studying how updating $M$ via the
Bellman equation (or the backward Bellman equation) updates the paths
represented in $M$. 
We will prove that after $t$ steps, the
estimate of $M$ via Bellman--Newton
exactly contains all paths up to length $2^t-1$ with their correct
probabilities in the Markov process, while forward and backward
TD exactly contain all paths up to length $t$.

Thus
for each algorithm (forward TD, backward TD, and Bellman--Newton), we consider the
exact (deterministic, non-sampled) update:
we set $M_{0} = \Id$ and then define
at step $t+1$ the update $M_{t+1}$ as the target update given by the
corresponding fixed point equation.
For forward TD, the operator update is defined as:
\begin{equation}
    M^{\text{TD}}_{t+1} = \Id + \gamma P M^{\text{TD}}_{t}.
\end{equation}
For backward TD, the operator update is defined as:
\begin{equation}
   M^{\text{BTD}}_{t+1} = \Id + \gamma M^{\text{BTD}}_{t} P.
 \end{equation}
The Bellman--Newton update (Definition~\ref{def:BN}) with learning rate
$1$ is
\begin{equation}
   \label{eq:BNupdate}
   M^{\text{BN}}_{t+1} = 2 M^{\text{BN}}_{t} -
   M^{\text{BN}}_{t} (\Id - \gamma P) M^{\text{BN}}_{t}.
\end{equation}
In expectation over the transition $s\to s'$, the
expected exact online update is $\del M=\frac1t (M-M^2+\gamma MPM)$
(Eq.~\ref{eq:EonlineM}). Here for the corresponding deterministic
operator we use a learning rate $1$ instead of $1/t$, which yields $\del
M=M-M(\Id-\gamma P)M$, hence the update \eqref{eq:BNupdate}. The successor state operator $M=(\Id-\gamma
P)^{-1}$ is a fixed point of this update. \footnote{It is not the only fixed point; for
instance, $M=0$ is another. But it is the only full-rank fixed point.}
This corresponds to the \emph{Newton method} $M\gets 
2M-MAM$ for inverting a matrix $A$
\citep{pan1991improved}.

 

\medskip
We now relate the forward TD, backward TD, and Bellman--Newton updates
to path composition. For each algorithm, we prove by induction that at
step $t$, there exists an integer $n_t$ such that $(M_{t})_{ss'}$ is
equal to the number of paths from $s$ to $s'$ with length  at most
$n_{t}$, weighted by their probability and discounted by their length,
namely,
\begin{equation}
  (M_{t})_{ss'} = \sum_{p \text{ path from } s \text{ to } s', |p| \leq
  n_t}\gamma^{|p|}\Prob(p) = \sum_{k=0}^{n_{t}}\gamma^{k}\sum_{s=s_0, ...,
  s_{k-1}, s_{k}=s'}P_{s_{0}s_{1}} \cdots P_{s_{n-1}s_{n_t}}
\end{equation}
where as in Section~\ref{sec:defin-prop}, $|p|$ denotes the length of a
path $p$ and $\Prob(p)=P_{s_{0}s_{1}}\cdots P_{s_{n-1}s_{n}}$ its probability
in the Markov process. Equivalently,
\begin{equation}
M_t=\sum_{0\leq k\leq n_t} \gamma^k P^k.
\end{equation}
The three algorithms will differ by the value of $n_t$.

For $t=0$, $M_{0} = \Id$, and the induction hypothesis is satisfied.

If the end point of a path $p_{1}$ corresponds to the starting point of a
path $p_{2}$, we denote $p_{1}\cdot p_{2}$ the concatenation of the two
paths.

For forward TD, we have $M^{\text{TD}}_{t+1}=\Id+\gamma
PM^{\text{TD}}_{t}$. By induction, if $M^{\text{TD}}_{t}=\sum_{0\leq
k\leq n_t^\text{TD}} \gamma^k P^k$, then we find
$M^{\text{TD}}_{t+1}=\Id+\gamma P \sum_{0\leq
k\leq n_t^\text{TD}}\gamma ^kP^k=\sum_{0\leq
k\leq n_t^\text{TD}+1}\gamma^k P^k$. Equivalently, looking at paths we
have
\begin{align}
  (M^{\text{TD}}_{t+1})_{ss'} &= \delta_{s=s'} + \gamma(PM^{\text{TD}}_{t})_{ss'}
  \\ &= \delta_{s=s'} + \gamma\sum_{s''}P_{ss''}\sum_{p \text{ path from
  } s'' \text{ to } s', \, |p| \leq n_{t}^{\text{TD}}}\gamma^{|p|}\Prob(p)
  \\ &= \delta_{s=s'} + \sum_{s''}\sum_{p \text{ path from } s'' \text{
  to } s',\, |p| \leq n_{t}^{\text{TD}}}\gamma^{|p|+1}\Prob((s,s'')\cdot p)
  \\ &= \delta_{s=s'} + \sum_{p \text{ path from } s \text{ to } s',\, 1 \leq |p| \leq n_{t}^{\text{TD}}}\gamma^{|p|}\Prob(p)
  \\ &= \sum_{p \text{ path from } s \text{ to } s',\, |p| \leq n_{t}^{\text{TD}}+1}\gamma^{|p|}\Prob(p)
\end{align}
Thus the induction hypothesis is satisfied with $n_{t+1}^{\text{TD}} =
n_{t}^\text{TD} + 1$. By induction, $n_{t}^\text{TD}=t$: at step $t$, $M^{\text{TD}}_{t}$  is the weighted sum of
paths of length at most $t$. $M^{\text{TD}}_{t+1}$ is obtained from
$M^{\text{TD}}_{t}$ by adding a transition to the left to every known
path (and re-adding the length-$0$ paths via the $\Id$ term).

Likewise, with backward TD we have
\begin{align}
  (M^{\text{BTD}}_{t+1})_{ss'} &= \delta_{s=s'} + \gamma(M^{\text{BTD}}_{t}P)_{ss'}
  \\ &= \delta_{s=s'} + \sum_{s''}\sum_{p \text{ path from } s'' \text{
  to } s',\, |p| \leq n_{t}^{\text{BTD}}}\gamma^{|p|+1}\Prob(p \cdot (s'',s'))
  \\ &= \sum_{p \text{ path from } s \text{ to } s',\, |p| \leq n_{t}^{\text{BTD}}+1}\gamma^{|p|}\Prob(p)
\end{align}
Contrary to forward TD, $M^{\text{BTD}}_{t+1}$ is obtained from
$M^{\text{BTD}}_{t}$ by adding a transition to the right to every known
path. This still leads to $n_t^\text{BTD}=t$. 

We now consider the Bellman--Newton operator update. We have
\begin{align}
   M^{\text{BN}}_{t+1} = 2 M^{\text{BN}}_{t} -
   M^{\text{BN}}_{t} (\Id - \gamma P) M^{\text{BN}}_{t}.
 \end{align}
Let us first compute $(\Id - \gamma P) M^{\text{BN}}_{t}$. By the
induction hypothesis and by the same reasoning as for forward TD, we have
 \begin{align}
   ((\Id - \gamma P) M^{\text{BN}}_{t})_{ss'}  &= M^{\text{BN}}_{t}  - \gamma P M^{\text{BN}}_{t}
   \\ &= \sum_{p \text{ path from } s \text{ to } s',\, |p| \leq
   n_{t}^{\text{BN}}}\gamma^{|p|}\Prob(p) - \sum_{p \text{ path from }
   s \text{ to } s', \, 1 \leq |p| \leq n_{t}^{\text{BN}} + 1}\gamma^{|p|}\Prob(p)
   \\ &= \delta_{s=s'} - \gamma ^{n_{t}^{\text{BN}} +
   1}\left(P^{n_{t}^{\text{BN}} + 1} 
   \right)_{ss'}.
\end{align}
%
Therefore,
\begin{align}
  M^{\text{BN}}_{t+1} &= 2 M^{\text{BN}}_{t} -  M^{\text{BN}}_{t} (\Id - \gamma P) M^{\text{BN}}_{t}
  \\ &= 2 M^{\text{BN}}_{t} -  M^{\text{BN}}_{t} (\Id - \gamma
  ^{n_{t}^{\text{BN}} + 1}P^{n_{t}^{\text{BN}} + 1} ) 
  \\ &= M^{\text{BN}}_{t} + \gamma ^{n_{t}^{\text{BN}} + 1}M^{\text{BN}}_{t}P^{n_{t}^{\text{BN}} + 1}
  \\ &= \sum_{p \text{ path from } s \text{ to } s',\, |p| \leq
  n_{t}^{\text{BN}}}\gamma^{|p|}\Prob(p) + \sum_{p \text{ path from }
  s \text{ to } s',\, n_{t}^{\text{BN}} + 1  \leq |p| \leq 2n_{t}^{\text{BN}} + 1}\gamma^{|p|}\Prob(p)
  \\ &= \sum_{p \text{ path from } s \text{ to } s',\, |p| \leq 2n_{t}^{\text{BN}}+1}\gamma^{|p|}\Prob(p) 
 \end{align}
 Therefore, $n_{t+1}^{\text{BN}} = 2n_{t}^{\text{BN}} + 1$.  At every
 step the Bellman--Newton operator update is doubling the maximal length of all known paths.
 
The efficiency of the operator Bellman--Newton update can also be explained from
properties of the Newton method. Indeed, the Bellman--Newton update in
 \eqref{eq:BNupdate} corresponds to the Newton update $M \leftarrow 2M -
 MAM$ for inverting the matrix $A$, applied to $A = \Id - \gamma P$
 \citep{pan1991improved}. With this method, the error $\Id-AM$ gets
 squared at each iteration: $\Id-AM\gets (\Id-AM)^2$
 \citep{pan1991improved}. Here at each step,
 if $M_t$ exactly contains all paths up to length $n_t$, then
 the error $\Id-AM_t$ contains all paths of length 
 $n_t+1$, namely, if $M_t=\sum_{k\leq n_t} \gamma^k P^k$ then
 $\Id-AM_t=\gamma^{n_t+1} P^{n_t+1}$. Thus squaring the error corresponds to doubling
 $n_t+1$.

\section{Successor States, Eligibility Traces, and the Backward Process}
\label{sec:traces}

\label{sec:backw-temp-diff}

\newcommand{\Pback}{P_{\mathrm{back}}}
\newcommand{\Mback}{M^{\mathrm{back}}}
\newcommand{\mback}{m^{\mathrm{back}}}

In this section, we relate the update equation \eqref{eq:Vparam} for the
value function using $M$, to the algorithm TD($\lambda$) and eligibility traces.
We also prove the
statement that backward TD is forward TD on the time-reversed process
(Theorem~\ref{thm:backwardTD}).

More precisely, we prove (Theorem~\ref{thm:elig}) that the expectation of
the TD$(1)$ update (expectation over the eligibility traces given the
current state) is the update \eqref{eq:Vparam} of the value
function using the successor state
operator.
Thus, updating $V$ via \eqref{eq:Vparam} using a learned model
$m_\theta$ of $M$ is equivalent to estimating the true $M$ via a model, while
eligibility traces are an unbiased Monte Carlo estimator of the
true $M$.
This suggests the possibility of using mixed estimates, such as
eligibility traces over a few past steps, and a model $m_\theta$ for the
older past.

Eligibility traces require access to an arbitrarily long trajectory
$(s_t)_{t\in \Z}$ (which, for convenience, we index with both positive
and negative integers, with $s_0$ the state at the current time). Thus,
contrary to the rest of this text, we assume that the Markov process is
ergodic and that the data are coming
from a stationary random trajectory of the process. In this
case, the sampling measure $\rho$ is the stationary distribution, and the
law of any sequence of consecutive observations $(s_t,\ldots,s_{t+n})$
from the trajectory is
$\rho(\d s_t) P(s_t,\d s_{t+1})\cdots P(s_{t+n-1},\d s_{t+n})$.

We also assume that for every $s$, $P(s,\d s')$ is absolutely continuous
with respect to $\rho(\d s')$. This is not necessary but leads to nicer
expressions. In that case, $M(s,\d s')=\delta_s(\d s')+m(s,s')\rho(\d
s')$ for some function $m$.

In the tabular setting, TD($\lambda$) maintains a vector $e_t$ over
states; $e_t$ is updated by
\begin{alignat}{2}
  e_{t}(\tilde s) &= \1_{s_t} + \gamma\lambda e_{t-1}(\tilde s) \qquad
  & \forall \tilde s
  \\ \del V(\tilde s) &= e_{t}(\tilde s)(r_{t} + \gamma V(s_{t+1}) - V(s_{t}))
  \qquad & \forall \tilde s.
\end{alignat}
This can be generalized to continuous environments and to a parametric
model $V_\phi$ of $V$, by formally defining $e$ as the discounted empirical measure of the past:
\begin{align}
  e_{t}(\d \tilde s) &\deq
  \sum_{n\geq 0}(\gamma\lambda)^{n}\delta_{s_{t-n}}(\d \tilde s)
  =
  \delta_{s_t}(\d \tilde s)+\gamma\lambda e_{t-1}(\d \tilde s)
\end{align}
corresponding to the parametric update of $V_\phi$ by
\begin{align}
  \del\phi &= (r_{t} + \gamma V_\phi(s_{t+1}) -
  V_\phi(s_{t}))\int_{\tilde s}\partial_{\phi} V_{\phi}(\tilde s)\,e_{t}(\d \tilde s)
  \\ &= (r_{t} + \gamma V_\phi(s_{t+1}) - V_\phi(s_{t}))\sum_{n\geq 0}
  (\gamma\lambda)^n \partial_{\phi} V_{\phi}(s_{t-n}).
  \label{eq:elig_traces_mc}
\end{align}

We have the following statement:
\begin{thm}
  \label{thm:elig}
  Let $\rho$ be the invariant measure of the Markov process, and
  $M_{\gamma\lambda} \deq (\Id - \gamma \lambda P)^{-1}$ the successor
  state operator with discount factor $\gamma\lambda$. Let
  $m_{\gamma\lambda}$ be the density of $(M_{\gamma\lambda}-\Id)$ with
  respect to $\rho$: $M_{\gamma\lambda}(\tilde s, \d s_{2}) =
  \delta_{\tilde s}(\d s_2)+m_{\gamma\lambda}(\tilde s, s_{2})\rho(\d
  s_{2})$.

  Then, the expected eligibility trace $e_{t}(\d \tilde s)$ knowing $s_{t}=s$ is:
  \begin{align}
    \E\left[e_{t}(\d \tilde s) |s_{t}=s\right] = \delta_{s}(\d \tilde s) + m_{\gamma\lambda}(\tilde s, s)\rho(\d \tilde s)=
    \frac{M_{\gamma\lambda}(\tilde s,\d s)\rho(\d \tilde s)}{\rho(\d s)}
  \end{align}
  Moreover, the expectation of the parametric TD($\lambda$) update
  \eqref{eq:elig_traces_mc} when a transition $(s, s')$ is observed is
  equal to the update \eqref{eq:Vparam} of $V$ using $m_{\gamma\lambda}$:
  \begin{align}
    \E[\del\phi|(s_{t}, s_{t+1})=(s,s')] = (r_{t} + \gamma
    V_{\phi}(s') - V_{\phi}(s))\left(\partial_{\phi}
    V_{\phi}(s) + \E_{\tilde s\sim \rho}
    \left[
    m_{\gamma\lambda}(\tilde s, s)
    \,\partial_{\phi} V_{\phi}(\tilde s)
    \right]\right)
  \end{align}
with $\rho$-probability $1$ over $s_t$.
\end{thm}

The proof of this theorem involves the time-reversal of the Markov
process; indeed, eligibility traces are a Monte Carlo estimate of the
discounted measure of \emph{predecessor} states.

Define the \emph{backward process} $\Pback(s',\d s)$ by reversing time: it is
the law of $s$ given $s'$ in a transition $s\to s'$. More precisely,
let $(s,s')$ be a random pair of states distributed according to $\rho(\d
s)P(s,\d s')$, and define $\Pback(s',\d s)$ to be the conditional
distribution of $s$ given $s'$ under this distribution. (This exists by
the general theory of conditional distributions
\citep[Thm.~8.1]{parthasarathy2005probability}, and is well-defined up
to a set of $\rho$-measure $0$.) Since $\rho$ is the invariant
measure of the process, the law of both $s$ and $s'$ is $\rho$, and one
has
\begin{equation}
\rho(\d s)P(s,\d s')=\rho(\d s')\Pback(s',\d s)
\end{equation}
by definition of conditional probabilities.

Then, given $s_t$, the distribution of $s_{t-n}$ follows the backward
process from $s_t$. Namely,
the law of any sequence of observations $(s_{t-n},\ldots,s_t)$
from the stationary distribution of the process satisfies
\begin{equation}
\rho(\d s_{t-n})P(s_{t-n},\d s_{t-n+1})\cdots P(s_{t-1},\d s_t)
=\rho(\d s_t) \Pback(s_t,\d s_{t-1})\cdots \Pback(s_{t-n+1},\d s_{t-n}).
\end{equation}

\begin{lem}
\label{lem:mback}
Let $m$ be the density of $M$, namely, $M(s,\d s')=\delta_s(\d
s')+m(s,s')\rho(\d s')$. (This exists under the assumption above that $P$ is
absolutely continuous with respect to $\rho$.)

Let 
$\Mback\deq (\Id-\gamma \Pback)^{-1}$ be the successor state
operator of the backward process, and let $\mback$ be the associated
density, $\Mback(s,\d s')=\delta_s(\d s')+\mback(s,s')\rho(\d s')$.
Then $\rho(\d s')\Mback(s',\d s)=\rho(\d s)M(s,\d s')$ and
\begin{equation}
\mback(s',s)=m(s,s')
\end{equation}
for $\rho$-almost every $(s,s')$.
\end{lem}

\begin{proof}
%
By induction from the definition of the backward process, we have
$\rho(\d s')\Pback^{n}(s', \d s) = \rho(\d s)P^{n}(s, \d s')$.
Then by definition of $\Mback$,
\begin{align}
  \rho(\d s')M^{\text{backward}}(s', \d s) &= \rho(\d s')\sum_{n\geq 0}\gamma^{n}P^{n}_{\text{backward}}(s', \d s) = \sum_{n\geq 0}\gamma^{n} \rho(\d s)P^{n}(s, \d s')
       \\ &= \rho(\d s) M(s, \d s')
\end{align}
Since $M(s, \d s') = \delta_{s}(\d s') + m(s,
s')\rho(\d s')$, and likewise for $\Mback$, this implies
\begin{align}
  \rho(\d s')\mback(s',s)\rho(\d s)=\rho(\d s) m(s,s')\rho(\d s')
\end{align}
as needed.
\end{proof}

\begin{proof}[Proof of Theorem~\ref{thm:elig}]
By definition of eligibility traces, one has $e_t(\d \tilde
s)=\sum_{n\geq 0} (\gamma\lambda)^n \delta_{s_{t-n}}(\d \tilde s)$.
Therefore,
the expectation of $e_{t}$ over the past of $s_{t}$ knowing $s_{t}$ is:
\begin{align}
  \E[e_{t}(\d \tilde s)|s_{t}=s] &= \E\left[ \sum_{n\geq 0}(\gamma\lambda)^{n}\delta_{s_{t-n}}(\d \tilde s) | s_{t} = s\right]
  \\  &= \sum_{n\geq 0} (\gamma\lambda)^{n}\Pback^{n}(s, \d \tilde s)
  \\  &= \Mback_{\gamma\lambda}(s, \d \tilde s)
\end{align}
where $\Mback_{\gamma\lambda}\deq (\Id-\gamma\lambda \Pback)^{-1}$ is the
successor state operator of the backward process. By
Lemma~\ref{lem:mback}, this is
\begin{align}
  \E[e_{t}(\d \tilde s)|s_{t}=s]  &= \delta_{s}(\d \tilde s) + m(\tilde s, s)\rho(\d \tilde s)
\end{align}
as needed.

Therefore, the expectation of the update~\eqref{eq:elig_traces_mc} of $V$ with TD($\lambda$) is:
  \begin{align}
    \E\left[\del\phi | s_{t}=s, s_{t+1}=s'\right] &= (r_{t} + \gamma
    V_\phi(s_{t+1}) -
V_\phi(s_{t}))\int_{\tilde s}\partial_{\phi} V_{\phi}(\tilde s)\,\E[e_{t}(\d \tilde s)| s_{t}=s, s_{t+1}=s']
    \\ &= (r_{t} + \gamma V_\phi(s_{t+1}) -
         V_\phi(s_{t}))\int_{\tilde s}\partial_{\phi} V_{\phi}(\tilde s)( \delta_{s}(\d \tilde s) + m(\tilde s, s)\rho(\d \tilde s))
         \\ &= (r_{t} + \gamma V_\phi(s') -
	 V_\phi(s))\left(
	 \partial_{\phi} V_{\phi}(s) + \E_{\tilde s\sim
	 \rho}\left[\partial_{\phi} V_{\phi}(\tilde
	 s)m_{\gamma\lambda}(\tilde s, s)\right]\right)
  \end{align}
  \end{proof}

Finally, the backward process provides a simple proof that backward TD is
forward TD on the backward process. Remember that the forward and
backward successor state operators are linked by $\rho(\d s_1)M(s_1,\d
s_2)=\rho(\d s_2) \Mback(s_2,\d s_1)$.

\begin{thm}[ (Backward TD is forward TD on the backward process)]
\label{thm:backwardTD}
Let $M$ and $\Mback$ be measure-valued functions such that $\Mback$ is the time-reverse of $M$,
namely $\rho(\d s_1)M(s_1,\d s_2)=\rho(\d s_2) \Mback(s_2,\d s_1)$.
Then the backward TD update
\begin{equation}
M\gets \Id+\gamma MP
\end{equation}
is equivalent ($\rho$-almost everywhere) to
\begin{equation}
\Mback\gets \Id+\gamma \Pback \Mback.
\end{equation}
\end{thm}

\begin{proof}
Let $D_\rho(\d s_1,\d s_2)$ be the diagonal measure with marginal $\rho$,
namely, $D_\rho(\d s_1,\d s_2)=\rho(\d s_1)\delta_{s_1}(\d s_2)=\rho(\d
s_2)\delta_{s_2}(\d s_1)$. Remember that the operator $\Id$ corresponds to the
process $\delta_{s_1}(\d s_2)$. By multiplying the backward TD update by
$\rho(\d s_1)$ one gets
\begin{align}
\rho(\d s_1)M(s_1,\d s_2)&\gets
D_\rho(\d s_1,\d s_2)+\gamma \rho(\d s_1) (MP)(s_1,\d s_2)
\\&=
D_\rho(\d s_1,\d s_2)+\gamma \int_{s'} \rho(\d s_1)M(s_1,\d s')P(s',\d
s_2)
\\&=D_\rho(\d s_1,\d s_2)+\gamma \int_{s'} \Mback(s',\d s_1)\rho(\d s')
P(s',\d s_2)
\\&=D_\rho(\d s_1,\d s_2)+\gamma \int_{s'} \Mback(s',\d s_1)\rho(\d
s_2)\Pback(s_2,\d s')
\\&=D_\rho(\d s_1,\d s_2)+\gamma \rho(\d s_2) (\Pback\Mback)(s_2,\d s_1)
\end{align}
and since $\rho(\d s_1)M(s_1,\d s_2)=\rho(\d s_2) \Mback(s_2,\d s_1)$,
this rewrites as
\begin{equation}
\rho(\d s_2) \Mback(s_2,\d s_1)\gets \rho(\d
s_2)\delta_{s_2}(\d s_1)+\gamma \rho(\d s_2) (\Pback\Mback)(s_2,\d s_1)
\end{equation}
namely ($\rho$-almost everywhere),
\begin{equation}
\Mback(s_2,\d s_1)\gets\delta_{s_2}(\d s_1)+\gamma (\Pback\Mback)(s_2,\d
s_1)
\end{equation}
which is forward TD on $\Mback$ for the time-reversed process.
\end{proof}

\section{Fixed Points for the FB Representation of $M$}
\label{app:FB}

\todo{move to appendix B?}

Here we state precisely, and prove, the fixed points properties for the
four variants of successor state learning in the FB representation
(Section~\ref{sec:FB}), in the tabular and in the overparameterized case.
The ``tabular'' case for $F$ and $B$ means that 
the state space
is finite and the values of $F(s)$ and $B(s)$ are stored explicitly for
every state $s$.

We fully describe the fixed points of
the four algorithms ff-FB, bb-FB, fb-FB, and bf-FB, which have quite
different properties.

We state these properties for the tabular case; by a simple argument
the fixed points are the same for \emph{overparameterized} $F$ and $B$.
\footnote{Namely,
parameterizations $F_{\thetaF}$ and $B_{\thetaB}$ such that any function
$F$ can be realized for some $\thetaF$, and moreover the map
$\partial_\thetaF F_{\thetaF}$ is surjective for any $\thetaF$, and
likewise for $B$. In short, any $F$ and $B$ can be realized, and any small \emph{change}
of $F$ or $B$ can be realized by a small change in $\thetaF$ and
$\thetaB$. \todo{give formal statement and proof?}}

In this section, we abuse notation by considering $F$ and $B$ both as
functions from the state space to $\R^r$ (as in Section~\ref{sec:FB}),
and as $r\times
\#S$-matrices. The model $M(s_1,\d s_2)=\transp{F(s_1)}B(s_2)\rho(\d
s_2)$ rewrites as $M=\transp{F}B\diag(\rho)$ or $\tildem=\transp{F}B$,
viewing everything as matrices with entries indexed by the state space.

We also assume that $\rho_s>0$ for every state $s$: every
state is sampled with nonzero probability.

By direct identification in
Proposition~\ref{prop:FBupdates}, in the tabular case we find the
following expressions for the updates of $F$ and $B$.

\begin{prop}[ (Tabular FB updates)]
\label{prop:tabularFB}
Assume the state space is finite and let
$F$ and $B$ be two $r\times
\#S$-matrices. Let the parameter $\theta_F$ of $F$ be the matrix $F$ itself
and likewise for $B$.

Abbreviate $\diagrho$ for the diagonal matrix
with entries $\rho_s$ for each state $s$.

Then the updates $\del \theta_F$ and $\del \theta_B$ of
Proposition~\ref{prop:FBupdates} for the FB representation of $M$ are equal to
\begin{equation}
\label{eq:tabularFB_f}
\del F=B\diagrho - \SigmaB F\transp{\Delta}\diagrho ,\qquad
\del B=F\diagrho-F \diagrho \Delta \transp{F}B\diagrho
\end{equation}
for forward TD on $F$ and $B$ respectively, and to
\begin{equation}
\label{eq:tabularFB_b}
\del F=B\diagrho-B\transp{(\diagrho\Delta)}\transp{B}F\diagrho,\qquad
\del B=F\diagrho-\SigmaF B\diagrho\Delta
\end{equation}
for backward TD on $F$ and $B$ respectively.
Here $\Delta$ is the matrix $\Id-\gamma P$, $\SigmaB= B\diagrho \transp{B}$, and $\SigmaF=
F\diagrho\transp{F}$.
\end{prop}

\begin{prop}[ (The fixed points of fb-FB approximate $M$ in $L^2(\rho)$ norm)]
\label{prop:fbFB}
The fixed points of the tabular fb-FB algorithm are the local extrema of
the error
\begin{equation}
\ell(F,B)\deq \E_{s_1\sim \rho,\,s_2\sim \rho} 
\left(\transp{F}(s_1)B(s_2)-\tildem(s_1,s_2)\right)^2
\end{equation}
where $\tildem(s_1,s_2)\deq M(s_1,\d s_2)/\rho(\d s_2)$ is the value
of $\tildem$ for the true successor state operator $M$. \footnote{This is the
Hilbert-Schmidt norm of the difference between $M$ and its approximation
$\transp{F}B\rho$, as operators on the space of functions over $S$
equipped with the $L^2(\rho)$ norm (Appendix~\ref{sec:svd}).
}

In that case, $\transp{F}B\diagrho$ is a truncated singular value decomposition
of the operator $M$ acting on the space of functions over $S$ equipped
with the $L^2(\rho)$ norm. 
\end{prop}

\todo{Clarify meaning of converse statements. In every case, what we prove is, if $X$ is
any operator satisfying the conditions, then there exists a stable point
$F$ and $B$ with $\transp{F}B\diagrho=X$. Not "any $F$ and $B$ with
$\transp{F}B\diagrho=X$ is a fixed point".}

\begin{prop}[ (Fixed points of ff-FB correspond to eigenspaces of
$M$)]
\label{prop:ffFB}
The set of approximations $\transp{F}B\diagrho$ of $M$ that appear
as a fixed point of the tabular ff-FB algorithm
is exactly the set of operators
such that there exists an $L^2(\rho)$-orthogonal decomposition
$\R^{\#S}=E\oplus E'$ of functions over the state space such that $E$ is
stable by $M$ (namely, $ME\subset E$), $E$ has dimension at most $r$, and $\transp{F}B\diagrho$ is equal to $M$ on $E$
and to $0$ on $E'$.
\end{prop}

\begin{prop}[ (Fixed points of bb-FB)]
\label{prop:bbFB}
The set of approximations $\transp{F}B\diagrho$ of $M$ that appear
as a fixed point of the tabular bb-FB algorithm
is exactly the set of operators such that
there exists an $L^2(\rho)$-orthogonal decomposition
$\R^{\#S}=E\oplus E'$ of functions over the state space such that
$E'$ is
stable by $M$ (namely, $ME'\subset E'$), $E$ has dimension at most $r$,
and $\transp{F}B\diagrho$ is the projection of $M$ onto $E$, namely,
$\transp{F}B\diagrho=\Pi_E M$ with $\Pi_E$ the $L^2(\rho)$-orthogonal
projector onto $E$.
\end{prop}

\begin{rem}[ (Fixed points of bb-FB correspond to eigen-probability
densities of
$M$)]
\label{rem:bbFB}
Stability of $E'$ by $M$ is equivalent to stability of $E$ by 
$\diagrho^{-1}\transp{M}\diagrho$. This corresponds to the Markov operator acting on probability
densities: if the state at time $t$ has probability distribution $f\rho$
for some vector $f$,
then the state at time $t+1$ has probability distribution
$(\diagrho^{-1}\transp{P}\diagrho f)\rho$.

Thus, in bb-FB, the space $E$ is a stable space of probability densities
for $P$ and $M$.
\end{rem}

In contrast, the bf-FB algorithm can stabilize on any subspace of
features. For instance, in rank $1$, set $\transp{F}$ to any vector, then
set $\transp{B}=\alpha \transp{F}$ where $\alpha=1/(F \diagrho
(\Id-\gamma P)\transp{F})$ (assuming this is nonzero).  In fact, fixed
points of bf-FB just compute a weak inverse of $\diagrho (\Id-\gamma P)$
in an arbitrary $r$-dimensional subspace.


\begin{prop}[ (Fixed points of bf-FB)]
\label{prop:bfFB}
The set of approximations $\transp{F}B\diagrho$ of $M$ that appear
as a fixed point of the tabular bf-FB algorithm
is exactly the set of operators such that
there exists a subspace $E$ of $L^2(\rho)$ of
dimension at most $r$ such
that $\transp{F}B\diagrho$ has image $E$ and kernel $E^\bot$, and 
$\transp{F}B\diagrho$ is the inverse of $\Pi(\Id-\gamma P)\Pi$ as
operators from $E$ to $E$, where $\Pi$ is the $L^2(\rho)$-orthogonal
projector on $E$.

Moreover, 
if $\rho$ is an invariant probability distribution of the Markov
process, then
every subspace $E$ of $L^2(\rho)$ of dimension at most $r$ provides such
a fixed point $\transp{F}B\diagrho$.
\end{prop}

\begin{proof}[Proof of Proposition~\ref{prop:fbFB}]
Viewing $\tildem$, $F$ and $B$ as matrices, the loss is
\begin{equation}
\ell(F,B)=\sum_{ij} \rho(i)\rho(j) \left(\sum_k
F_{ki}B_{kj}-\tildem_{ij}\right)^2
\end{equation}
so that
\begin{equation}
\frac{\partial \ell(F,B)}{\partial F_{ki}}=2\sum_{j} \rho(i)\rho(j)
B_{kj}\left(\sum_{k'}
F_{k'i}B_{k'j}-\tildem_{ij}\right)
\end{equation}
which is the $ki$ entry of the matrix
$2B\diagrho(\transp{B}F-\transp{\tildem})\diagrho$.

Now, $F$ is a local extremum of this loss if and only if this derivative
is $0$ for every $ki$, namely, if and only if the matrix
$B\diagrho(\transp{B}F-\transp{\tildem})\diagrho$ is $0$.  Now, by definition of
$\tildem$ we have $M=\tildem\diagrho$, namely,
$\tildem=\Delta^{-1}\diagrho^{-1}$. So
$B\diagrho(\transp{B}F-\transp{\tildem})\diagrho=0$ is equivalent to
$B\diagrho\transp{B}F\diagrho-B\transp{(\Delta^{-1})}\diagrho=0$. Since $\diagrho$ and
$\Delta$ are
invertible, by multiplying by $\diagrho^{-1}\transp{\Delta}\diagrho$ on the right, this is equivalent
to $B\diagrho\transp{B}F\transp{\Delta}\diagrho-B\diagrho=0$. This is
equivalent to $\del F=0$ in \eqref{eq:tabularFB_f}, namely, to $F$ being
a fixed point of forward TD.

A similar computation with $B$ proves that $\partial \ell(F,B)/\partial
B=0$ if and only if $\del B=0$ in \eqref{eq:tabularFB_b}, namely, if and
only if $B$ is a fixed point of \emph{backward} TD. Therefore, $F$ and $B$ are a
local optimum of $\ell$ if and only if they are a fixed point of the
fb-FB algorithm.

Let us turn to the statement about singular value decompositions.
Generally speaking, we know that the matrices of a given rank which are
local extrema of the matrix norm of the difference with $\tilde m$ are
truncated singular value decompositions of $\tilde m$; however, here
these matrices are parameterized as $\transp{F}B$, and a priori this
parameterization might change the local extrema, so we give a full proof.

By Lemma~\ref{lem:truncsvd}, the matrix $\transp{F}B\diagrho$ is a truncated SVD
of $M$ if and only if $\transp{F}B\diagrho$ and $M$
coincide on $(\Ker \transp{F}B\diagrho)^\bot$ and $M (\Ker
\transp{F}B\diagrho)\bot
\Img \transp{F}B\diagrho$. Here all orthogonality relations are defined with respect
to the $L^2(\rho)$ inner product, namely, $\langle
x,y\rangle=\transp{x}\diagrho y$.

If $F$ is a fixed point of \eqref{eq:tabularFB_f}, then
$0=B\diagrho-\SigmaB F\transp{\Delta}\diagrho$. Since $\diagrho$ is
invertible and since
$\SigmaB=B\diagrho\transp{B}$, this rewrites as
$B(\Id-\diagrho\transp{B}F\transp{\Delta})=0$. Taking transposes, 
this is $(\Id-\Delta \transp{F}B\diagrho)\transp{B}=0$. By
definition, $M$ is the inverse of $\Delta$;
multiplying by $M$, we find
$(M-\transp{F}B\diagrho)\transp{B}=0$. This implies that $M$
and $\transp{F}B\diagrho$ coincide on the image of $\transp{B}$.
A fortiori, they coincide on the image of $\transp{B}F\diagrho$, which is
included in the image of $\transp{B}$.

In general, for an operator $A$ on a Euclidean space, $\Img A=(\Ker
A^{\ast})^\bot$ with $A^\ast$ the adjoint of $A$. Here, with the
inner product from $L^2(\rho)$, the adjoint of $A$ is $\diagrho^{-1}\transp{A}\diagrho$
(Appendix~\ref{sec:svd}). So the
adjoint of $\transp{B}F\diagrho$ is $\transp{F}B\diagrho$. Therefore, 
$\Img \transp{B}F\diagrho$ is $(\Ker \transp{F}B\diagrho)^\bot$. So
$M$
and $\transp{F}B\diagrho$ coincide on $(\Ker \transp{F}B\diagrho)^\bot$.
This was the first point to be proved.

Next, if $B$ is a fixed point of \eqref{eq:tabularFB_b}, then
$0=F\diagrho-F\diagrho\transp{F}B\diagrho\Delta$. Multiplying on the
right by
$M=\Delta^{-1}$ this is equivalent to
$F\diagrho(M-\transp{F}B\diagrho)=0$. This states that the image of
$M-\transp{F}B\diagrho$ is $\rho$-orthogonal to the image of
$\transp{F}$. So for any $x$, $(M-\transp{F}B\diagrho)x\bot \Img
\transp{F}$.
Take $x\in \Ker \transp{F}B\diagrho$. Then $Mx\bot \Img \transp{F}$. Since
$\Img \transp{F}B\diagrho\subset \Img \transp{F}$, we have $Mx\bot
\Img\transp{F}B\diagrho$ as well. Therefore, the image of $\Ker \transp{F}B\diagrho$ by
$M$ is orthogonal to the image of $\transp{F}B\diagrho$. This was the
second point to be proved.
\end{proof}

\begin{proof}[Proof of Proposition~\ref{prop:ffFB}]
In this proof, we denote
\begin{equation}
f\deq F\diagrho^{1/2},\qquad b\deq B\diagrho^{1/2},\qquad D\deq
\diagrho^{1/2}\Delta \diagrho^{-1/2},
\end{equation}
using that $\rho$ is invertible. Then the fixed point equations $\del
F=0$ and $\del B=0$ for the
forward TD updates 
\eqref{eq:tabularFB_f}
rewrite as
\begin{equation}
\label{eq:tabularFB_f_reduced}
0=b-b\transp{b}f\transp{D},\qquad 0=f-fD\transp{f}b.
\end{equation}
This change of variables cancels the $\rho$ factors and maps
$L^2(\rho)$-orthogonality to usual orthogonality.

$(\Rightarrow)$. Assume that $\transp{F}B\diagrho$ is a fixed point of
ff-FB, so that the fixed point equations above are satisfied.

The first fixed point equation yields $D\transp{f}b\transp{b}=\transp{b}$.
Let $b'$ be
the Moore-Penrose pseudoinverse of $\transp{b}$ (equal to
$(b\transp{b})^{-1}b$ if invertible). By the general properties of the
Moore-Penrose pseudoinverse,
$\transp{b}b'$ is the orthogonal projector onto $\Img
\transp{b}$, and $b\transp{b}b'=b$. Thus, multiplying
$D\transp{f}b\transp{b}=\transp{b}$ by $b'$ on the right, we find
$D\transp{f}b=\Pi$ where $\Pi$ is the orthogonal projector onto $\Img
\transp{b}$. This rewrites as $\transp{f}b=D^{-1}\Pi$, so that
$\transp{f}b$ is equal to $D^{-1}$ on $\Img \Pi$ and to $0$ on its
orthogonal.

The second fixed point equation reads $fD\transp{f}b=f$. Since
$D\transp{f}b=\Pi$ this means that $f\Pi=f$, or $\transp{f}=\Pi
\transp{f}$.
Consequently, $\Img \transp{f}\subset \Img \Pi$, and a fortiori $\Img
\transp{f}b\subset \Img \Pi$. Thus, $\Img D^{-1}\Pi \subset \Img \Pi$,
namely, $\Img \Pi$ is stable by $D^{-1}$.

Note that $\Img \Pi=\Img \transp{b}$, so its dimension is at most the
rank of $b$ which is at most $r$.

Unwinding the change of variables with $\diagrho^{1/2}$, we see that
$\Pi_E\deq \diagrho^{-1/2} \Pi \diagrho^{1/2}$ is an
$L^2(\rho)$-orthogonal projector, whose image $E\deq \Img \Pi_E$ is
stable by $\Delta^{-1}$, and such that $\transp{F}B\rho$ is equal to
$\Delta^{-1}\Pi_E$. Thus $\transp{F}B\rho$ is equal to $\Delta^{-1}$ on
$E$ and to $0$ on its $L^2(\rho)$-orthogonal.

$(\Leftarrow)$.
Let $E$ be a stable subspace of $M$, of dimension at most $r$, such that
$\transp{F}B\rho$ is equal to $M$ on $E$ and to $0$ on the
$L^2(\rho)$-orthogonal $E'$ of $E$.

Let $\Pi_E$ be the $L^2(\rho)$-orthogonal projector onto $E$. 
Since $E$ is stable by $M$, we have $M\Pi_E=\Pi_EM\Pi_E$. Moreover,
the
condition that $\transp{F}B\rho$ is equal to $M$ on $E$ and to $0$ on
$E'$ is equivalent to saying that $\transp{F}B\rho=M \Pi_E$.

Define $H=\diagrho^{1/2}E$ and $H'=\diagrho^{1/2}E'$, so that $H$ and
$H'$ are orthogonal in the usual sense. Note that $H'$ is stable by
$\diagrho^{1/2}M\diagrho^{-1/2}=D^{-1}$. The property $M\Pi_E=\Pi_EM\Pi_E$
rewrites as $D^{-1}\Pi=\Pi D^{-1}\Pi$ with $\Pi$ the orthogonal
projector onto $H$.
Moreover, $\transp{F}B\rho=M
\Pi_E$ rewrites as $\transp{f}b=D^{-1}\Pi$.

Let $b$ be
any matrix such that $\Img \transp{b}=H$ (e.g., made of a basis of $H$
padded with $0$'s up to dimension $r$). Let $b'$ be its Moore--Penrose
pseudoinverse. Define $f\deq b'\transp{(D^{-1})}$. Then
$b\transp{b}f\transp{D}=b\transp{b}b'=b$ so that the first fixed point
equation $0=b-b\transp{b}f\transp{D}$ is satisfied.

Since $D^{-1}\Pi=\Pi D^{-1}\Pi$, we have $\Pi \transp{(D^{-1})}=\Pi
\transp{(D^{-1})}\Pi$, thus $b'\Pi \transp{(D^{-1})}=b'\Pi
\transp{(D^{-1})}\Pi$. As above, $\Pi=\transp{b}b'$. Therefore,
$b'\transp{b}b'\transp{(D^{-1})}=b'\transp{b}b'\transp{(D^{-1})}\Pi$.
Now,
the Moore--Penrose pseudoinverse of $\transp{b}$ satisfies $b'
\transp{b}b'=b'$. Thus $b'\transp{(D^{-1})}=b'\transp{(D^{-1})}\Pi$,
namely, $f=f\Pi$. Since $\transp{f}b=D^{-1}\Pi$ this rewrites as
$f=fD\transp{f}b$, namely, the second fixed point equation is satisfied.

This proves that $\transp{F}B\rho$ satisfies the fixed point equations.
Moreover, given $E$, many such fixed points exist: a fixed point can be built using any matrix $B$ which spans
$E$, then defining $F$ from $B$.
\end{proof}

\begin{proof}[Proof of Proposition~\ref{prop:bbFB}]
Denoting
\begin{equation}
f\deq F\diagrho^{1/2},\qquad b\deq B\diagrho^{1/2},\qquad D\deq
\transp{(\diagrho^{1/2}\Delta \diagrho^{-1/2})},
\end{equation}
the fixed point equations $\del
F=0$ and $\del B=0$ for the
backward TD updates
\eqref{eq:tabularFB_b}
rewrite as
\begin{equation}
0=f-f\transp{f}b\transp{D},\qquad 0=b-bD\transp{b}f.
\end{equation}

These are the same equations as \eqref{eq:tabularFB_f_reduced} with $f$
and $b$ swapped. Therefore, the same proof yields the following. Let
$\Pi$ be the orthogonal projector on $\Img \transp{f}$, we obtain that
$\transp{b}f$ is equal to $D^{-1}\Pi$, and that $\Img \Pi$ is stable by $D^{-1}$.

Equivalently, $\transp{f}b$ is equal to $\Pi \transp{(D^{-1})}$ and $\Ker
\Pi$ is stable by $\transp{(D^{-1})}$.

Set $\Pi_E\deq \diagrho^{-1/2}\Pi \diagrho^{1/2}$. Then $\Ker \Pi_E$ is
stable by $\diagrho^{-1/2}\transp{(D^{-1})}\diagrho^{1/2}$. Moreover,
$\Pi_E$ is an $L^2(\rho)$-orthogonal projector.

Here $\transp{(D^{-1})}=\diagrho^{1/2} M \diagrho^{-1/2}$. Therefore,
$\Ker \Pi_E$ is stable by $M$. Moreover, the relationship
$\transp{f}b=\Pi \transp{(D^{-1})}$ rewrites as $\transp{F}B\rho=\Pi_E
M$.
\end{proof}

\begin{lem}
\label{lem:dirichlet}
Let $\rho$ be an invariant probability distribution of $P$. Then for any
vector $f$,
\begin{align}
\transp{f}\diagrho (\Id-\gamma P)f = (1-\gamma)\E_{s\sim \rho} [f(s)^2]
+\frac{\gamma}{2} \E_{s\sim \rho,\,s'\sim P(s,\d s')} [(f(s)-f(s'))^2]
\end{align}
and in particular, this is positive for any nonzero $f$.
\end{lem}

\begin{proof}
The proof is left as an exercise. The second term is known as the
Dirichlet form of a Markov chain \citep{DSC_logsob}, and plays an important role in the convergence analysis of
TD in some situations \citep{ollivier2018approximate}.
\end{proof}

\begin{proof}[Proof of Proposition~\ref{prop:bfFB}]
Denoting again
\begin{equation}
f\deq F\diagrho^{1/2},\qquad b\deq B\diagrho^{1/2},\qquad D\deq
\diagrho^{1/2}\Delta \diagrho^{-1/2},
\end{equation}
then the fixed point equations $\del F=0$ and $\del B=0$ in
\eqref{eq:tabularFB_f}--\eqref{eq:tabularFB_b} for backward TD for $F$ and forward TD for
$B$ rewrite as
\begin{equation}
0=f-fD\transp{f}b,\qquad 0=b-b\transp{D}\transp{b}f.
\end{equation}
Moreover, if $\rho$ is an invariant probability
distribution of the Markov process, then
Lemma~\ref{lem:dirichlet} implies
\begin{equation}
\label{eq:dirichlet_reduced}
\transp{x}Dx>0
\end{equation}
for any nonzero vector $x$.

We will work on $f$, $b$, and $D$; the statements on $\transp{F}B\diagrho$ follow by unwinding the change of
variables.

$(\Leftarrow)$. Assume that $X$ is an operator with image $H$ and kernel
$H^\bot$, such that $X$ and $\Pi D\Pi$ are inverses as operators from $H$
to $H$, with $\Pi$ the orthogonal projector onto $H$.  Let $O$ be any
isometry from $H$ to $\R^r$.  Set $f=O\Pi$ and $b=OX$, so that
$\transp{f}b=\Pi X=X$.  Note that $\Img \transp{f}=\Img \transp{b}=H$.
Moreover, $f\Pi=f$, and $b\Pi=b$ because $X\Pi=X$.
So $\transp{f}b$ and $\Pi D\Pi$ are inverses as operators on $H$.
Therefore, for any $x,y\in H$, we have $\transp{x} (\Pi D \Pi)(
\transp{f}b)y=\transp{x}y$ and $\transp{x}(\transp{f}b)(\Pi D \Pi)
y=\transp{x}y$. Since $x$ and $y$ lie in $H$ and $\Img \transp{f}=H$, we can write $x=\transp{f}z$
and $y=\Pi z'$, with $z$ and $z'$ not necessarily in $H$. Then
$\transp{x} (\Pi D \Pi)( \transp{f}b)y=\transp{z} f\Pi D \Pi\transp{f}b
\Pi z'=\transp{x}y=\transp{z}f\Pi z'$ for any $z$ and $z'$ in the whole space. Since $f\Pi=f$ and $b\Pi=b$
this rewrites as $\transp{z}fD\transp{f}bz'=\transp{z}fz'$ for any $z$
and $z'$ in the whole space.  Therefore, $fD\transp{f}b=f$, namely, the first fixed point
equation is satisfied. The second fixed point equation
$\transp{b}=\transp{f}bD\transp{b}$ is similar, starting with
$\transp{x}(\transp{f}b)(\Pi D \Pi) y=\transp{x}y$ and substituting
$x=\Pi z$, $y=\transp{b}z'$.

$(\Rightarrow)$.
Assume that the two
fixed point equations are satisfied. Since $f=fD\transp{f}b$ we have
$\Ker b\subset \Ker f$. Using the other equation proves that $\Ker
f\subset \Ker b$, thus $f$ and $b$ have the same kernel. Therefore
$\transp{f}$ and $\transp{b}$ have the same image. Let $H$ be this image,
and let $\Pi$ be the orthogonal projector onto $H$.

The second fixed point equation is $\transp{b}=\transp{f}bD\transp{b}$.
Thus $H=\Img
\transp{b}=\Img \transp{f}bD\transp{b} \subset \Img \transp{f}b\subset
\Img \transp{f}=H$. Therefore the image of $\transp{f}b$ is $H$.
Likewise, the first equation $f=fD\transp{f}b$ implies that the kernel of
$\transp{f}b$ is $H^\bot$.

Let us prove that $\transp{f}b$ and $\Pi D\Pi$ are inverses as operators
from $H$ to $H$. This is equivalent to proving that for any $x,y\in H$,
we have
$\transp{x} (\Pi D \Pi)( \transp{f}b)y=\transp{x}y$
and
$\transp{x}(\transp{f}b)(\Pi D \Pi) y=\transp{x}y$.
Since $\Img \transp{f}=H$, we can write
$x=\transp{f} z$. Hence $\transp{x}
(\Pi D \Pi)( \transp{f}b)y=\transp{z} f\Pi D \Pi \transp{f}by$. Since
$\Img \transp{f}=H$ we have $\Pi \transp{f}=\transp{f}$ and $f\Pi=\Pi$,
so $\transp{z} f\Pi D \Pi \transp{f}by=\transp{z}fD\transp{f}by$, which
is $\transp{z}fy=\transp{x}y$ by the first fixed point equation.
Therefore, we have $\transp{x}
(\Pi D \Pi)( \transp{f}b)y=\transp{x}y$. For the other equality, since $\Img
\transp{b}=H$, we can write $y=\transp{b}z$. Hence
$\transp{x}(\transp{f}b)(\Pi D \Pi) y=
\transp{x}\transp{f}b\Pi D \Pi\transp{b}z$. Again,
$\Pi\transp{b}=\transp{b}$ and $b\Pi=b$, so $\transp{x}\transp{f}b\Pi
D \Pi\transp{b}z=\transp{x}\transp{f}b D \transp{b}z$. By the second
fixed point equation, this is $\transp{x}\transp{b}z=\transp{x}y$. This
proves the claim.

Finally, let us turn to the statement about realizing any subspace $E$ this
way.
Let $E$ be an arbitrary subspace
of $\R^{\#S}$,
of dimension $r$. Let $H\deq \diagrho^{1/2}E$. Let $\Pi$ be the
\emph{rectangular}
orthogonal projector onto $H$ (namely, its range is $H$
only; its transpose $\transp{\Pi}$ is the inclusion map from $H$ to
$\R^{\#S}$), and let $A$ be any invertible linear
map from $H$ to $\R^r$. Set $f\deq A \Pi$.

First, we claim that the square matrix $fD\transp{f}$ is invertible. 
Indeed, assume there exists a vector $x\in \R^r$ such that
$fD\transp{f}x=0$. Then $\transp{x}fD\transp{f}x=0$. By
\eqref{eq:dirichlet_reduced} this implies $\transp{f}x=0$, or
$\transp{\Pi}\transp{A}x=0$. Since $\transp{A}x\in H$ we have
$\transp{\Pi}\transp{A}x=\transp{A}x$, so $\transp{A}x=0$. But since $A$
is invertible this implies $x=0$. Therefore, $fD\transp{f}$ is
invertible.

Then we set $b\deq (fD\transp{f})^{-1}f$. Let us check that the fixed
point equations are satisfied. Obviously, $f=fD\transp{f}b$, so the first
fixed point equation holds. For the second one, we have
\begin{equation}
b\transp{D}\transp{b}f=(fD\transp{f})^{-1}f\transp{D}\transp{f}(f\transp{D}\transp{f})^{-1}f=(fD\transp{f})^{-1}f=b.
\end{equation}
Therefore, the second equation holds as well, so that $f$ and $b$ are a
fixed point of the bf-FB algorithm.
\end{proof}

\section{The FB Representation and Bellman--Newton}

\subsection{The FB Representation Coincides With Bellman--Newton for
Symmetric $P$}
\label{sec:FBandBN}

Here we prove that the tabular FB updates (all four versions) coincide
with the Bellman--Newton update in the small-learning-rate
(continuous-time) limit, on-policy, with suitable initializations, and
provided that the transition matrix $P$ of the Markov process is
symmetric.

On a finite space, if $P$ is symmetric then the uniform measure is an
invariant distribution of the process. Therefore, being on-policy means
that $\rho$ is uniform.

\begin{thm}[ (The FB update is Bellman--Newton for symmetric $P$)]
\label{thm:FBandBN}
Assume that the state space $S$ is finite, and that the transition matrix
$P$ is symmetric.

Let $\rho$ be the uniform distribution on $S$, which is
invariant under the Markov process. Let $\diagrho=\frac{1}{\#S}\Id$ be the
diagonal matrix with entries $\rho$.

Let $F_0$ and $B_0$ be two $r\times \#S$-matrices
Consider the continuous-time equation
\begin{equation}
\frac{\d F_t}{\d t}=\del F,\qquad \frac{\d B_t}{\d t}=\del B
\end{equation}
initialized at $F_0$ and $B_0$,
where $\del F$ and $\del B$ are the tabular FB updates
of Proposition~\ref{prop:tabularFB}, computed at $F_t$ and $B_t$. Any of
the four variants ff-FB, fb-FB, bf-FB, bb-FB of
Proposition~\ref{prop:tabularFB} may be used.

Assume that $F_0=B_0$. For the ff-FB, fb-FB, and bb-FB variants,
furthermore assume that $\Delta$ commutes with $\transp{F_0}B_0$ (e.g.,
initialize to $F_0=B_0=\Id$).

Let $M_t\deq \transp{F_t}B_t \diagrho$ be the resulting estimate of the
successor state matrix. Then $M_t$ evolves according to the
Bellman--Newton update
\begin{equation}
\frac{\d M_t}{\d t}=\eta M_t-\eta M_t(\Id-\gamma P)M_t
\end{equation}
with learning rate $\eta=2/\#S$.
\end{thm}

This bears some discussion with respect to feature learning. As discussed
elsewhere, the Bellman--Newton update does not learn features (the kernel
and image of $M_t$ are preserved), and neither does the bf-FB variant in
the case of uniform $\rho$. All other variants (ff-FB, bf-FB, bb-FB)
learn features by changing the kernel of $F$ or $B$, and have fixed
points corresponding to various eigenspaces of $\Delta$
(Appendix~\ref{app:FB}). Thus, how is it possible that these FB updates coincide
with Bellman--Newton? This is where the assumption
$[\Delta,\transp{F_0}B_0]=0$ comes in: this commutation occurs, broadly
speaking, if $\transp{F_0}B_0]$ is already aligned with the eigenspaces
of $\Delta$. In that case, the FB updates will coincide with
Bellman--Newton and enjoy its improved asymptotic convergence. If not,
they will avoid the shortcoming of Bellman--Newton and learn features,
stabilizing to such an alignment.

\begin{proof}
We abbreviate $F'_t$ for $\d F_t/\d t$ and likewise for all other
quantities.

According to Proposition~\ref{prop:tabularFB}, the forward-TD equations for
$F$ and $B$ are
\begin{equation}
F'_t=B_t\diagrho - B_t\diagrho \transp{B_t} F_t\transp{\Delta}\diagrho ,\qquad
B'_t=F_t\diagrho-F_t \diagrho \Delta \transp{F_t}B_t\diagrho
\end{equation}
and the backward-TD equations are
\begin{equation}
F'_t=B_t\diagrho-B_t\transp{(\diagrho\Delta)}\transp{B_t}F_t\diagrho,\qquad
B'_t=F_t\diagrho-F_t\diagrho\transp{F_t} B_t\diagrho\Delta
\end{equation}
Here we have $\diagrho=\frac{1}{\# S}\Id$. Moreover, since $P$ is
symmetric, we have $\Delta=\transp{\Delta}$.

Let us start with the bf-FB variant (backward-TD on $F$ and forward-TD on
$B$). In that case, the evolution equations are symmetric between $F$ and
$B$,
because $\Delta=\transp{\Delta}$. Therefore, if $F=B$ at startup then
$F=B$ at all times. Thus, we have $M_t=\transp{F_t}F_t\diagrho$. Since
$\diagrho$ is proportional to $\Id$, it commutes with all other terms.
Thus, using $F_t=B_t$ and $\Delta=\transp{\Delta}$, we find
\begin{align}
M'_t&=\transp{(F'_t)}F_t\diagrho+\transp{F_t}F'_t\diagrho
\\&=2\transp{F_t}F_t\diagrho^2-2\transp{F_t}F_t \Delta \transp{F_t} F_t \diagrho^3
\\&=2M_t\diagrho-2M_t\Delta M_t \diagrho
\\&=\frac{2}{\# S} (M_t-M_t \Delta M_t)
\end{align}
as $\diagrho=\frac{1}{\#S}\Id$.
This is the Bellman--Newton update.

In the other cases there is one more argument, after which the
computation is similar. At startup, by assumption we have $F=B$ and
$\Delta$ commutes with $\transp{F}B$. Assume that, at some particular
time $t$, we have $F_t=B_t$ and $\Delta$ commutes
with $\transp{F_t}B_t$. Then, since $\Delta=\transp{\Delta}$ and $\diagrho$ commutes with everything, all the
updates of $F$ and $B$ at that time $t$ amount to just
\begin{equation}
F'_t=F_t\diagrho-F_t\transp{F_t}F_t\Delta \diagrho^2.
\end{equation}
Therefore, at that time $t$, the derivative of the commutator between
$\Delta$ and $\transp{F_t}B_t$ is
\begin{align}
[\Delta,\transp{F_t}B_t]'&=[\Delta,(\transp{F_t}F_t)']
\\&=[\Delta, 2\transp{F_t}F_t\diagrho-\transp{F_t}F_t\transp{F_t}F_t\Delta
\diagrho^2-\Delta\transp{F_t}F_t\transp{F_t}F_t\diagrho^2]
\\&=0
\end{align}
since $\Delta$ commutes with $\transp{F_t}F_t$ and $\diagrho$ commutes
with everything.

Thus, if at some time $t$ one has $F_t=B_t$ and $\Delta$ commutes
with $\transp{F_t}B_t$, then $F_t$ and $B_t$ have the same derivative at
time $t$, and the derivative of the commutator of $\Delta$ and
$\transp{F_t}B_t$ is $0$. Therefore, if these conditions hold at startup,
they hold at all times $t$. In that case, the evolution equations become
identical to the bf-FB case, and the conclusion holds as above.
\end{proof}

\subsection{The BN-FB update}
\label{sec:BNFB}

Here we introduce Bellman--Newton FB (BN-FB), a variant of the FB updates that coincides with
Bellman--Newton in the tabular case for arbitrary $P$, not only symmetric
$P$. It is compatible with sampling, without the sampling issues of the
Bellman--Newton update, and admits a parametric version.

However, it
still keeps the main shortcoming of the Bellman--Newton update, namely,
that the kernel and image of the estimate of $M$ are fixed (no features
are learned).

In the tabular case, let $F$ and $B$ be two $r\times \#S$-matrices, and
define the updates
\begin{equation}
\del \transp{F}\deq \transp{F}-\transp{F}B\diagrho\Delta
\transp{F},\qquad
\del B=B-B\diagrho\Delta \transp{F}B
\end{equation}
where as usual $\diagrho$ is the diagonal matrix with entries $\rho$, and
$\Delta=\Id-\gamma P$.

The updates with learning rate $\eta$ lead to a Bellman--Newton udpate on
the model $M=\transp{F}B\diagrho$:
\begin{multline}
F\gets F+\eta \del F, \qquad B\gets B+\eta \del B 
\\\Rightarrow \qquad M \gets
(1+\eta)M-\eta M\Delta M +O(\eta^2)
\end{multline}
at first order in $\eta$. In particular, the continuous-time dynamics
will be an exact Bellman--Newton update.

The parametric version is obtained as before, by approximating these
ideal updates in $\rho$-norm, and by sampling $\Delta$ using
$\diagrho\Delta=\E_{s\sim \rho,\,s'\sim P(\d s'|s)}
(\1_s\transp{\1_s}-\gamma \1_s\transp{\1_{s'}})$. Letting $\thetaF$ and
$\thetaB$ be the parameters of $F$ and $B$ respectively, this yields
\begin{equation}
\E_{s_1 \sim \rho}\, (\partial_{\thetaF} \transp{F(s_1)})
(F(s_1)-\transp{D}F(s_1))
\end{equation}
for the update of the parameters of $F$, and
\begin{equation}
\E_{s_1\sim \rho}\,(\partial_\thetaB \transp{B(s_1)})
(B(s_1)-DB(s_1))
\end{equation}
for the parameters of $B$.
Here $D$ is an $r\times r$ matrix (even for continuous states) given by
\begin{equation}
D\deq \E_{s\sim
\rho,\,s'\sim P(s,\d s')}\,
B(s)\transp{(\gamma F(s')-F(s))}.
\end{equation}
It is possible to use a single sampled transition $s\to s'$ for $D$ (this
option
requires no storage of $D$ since the updates simplify), or to estimate
$D$ online using a moving average over a number of past transitions $s\to
s'$. This second option reduces variance but introduces some bias due to
old values in the moving average.

\section{Sampling Simplified States for $s_1$ and $s_2$}
\label{sec:sss}

This section addresses two potential shortcomings of the parametric TD
and Bellman--Newton
algorithms for $M$.

\begin{itemize}
\item The parametric updates for TD and for Bellman--Newton involve
sampling additional states $s_1$ and $s_2$ unrelated to the states $s\to
s'$ currently visited (and actually use \emph{every}
state $s_1$ and $s_2$ for the tabular case). A simple option is to
sample $s_1$ and $s_2$ at random among a dataset of past visited states.
But if actual
states and transitions are few, or complicated to come by, or if it is
inconvenient to store many states (pure online setting), sampling
additional states according to the data distribution $\rho$ may be a
limitation.

We show that $s_1$ and $s_2$ can be sampled according to any
``simple'' distribution $\rhosimple$. This could help learning by
making it possible to use many samples of $s_1$ and $s_2$ for each observed transition
$s\to s'$, thus bringing the
parametric updates closer to the tabular updates (which use every $s_1$
and $s_2$).

\item Defining $m_\theta$ as a density with respect to the unknown
distribution $\rho$ may pose numerical problems: In regions where $M$ is
not small but $\rho$ is small, this attributes a large value to
$m_\theta$, which may perturb learning.
\end{itemize}

Here, we show that using simplified states $s_1,s_2\sim \rhosimple$ in
the parametric updates, and defining $m_\theta$ with respect to an
arbitrary reference measure $\rhoref$ on $\mathcal{S}$, just amounts to
using different learning rates on different parts of the state, and
different norms to define the parametric updates. Thus, these simplified
algorithms still make sense; however, proper factors must be included,
given in \eqref{eq:paramtd_sss}--\eqref{eq:Mparam_sss} below.

We consider three different measures
on states: the main ``data'' measure $\rho(\d s)$ from which we obtain
observations $s\to s'$, and which we do not control; a ``simple'',
user-chosen probability measure $\rhosimple$
from which we can cheaply sample states, real or synthetic (such as
a uniform distribution, or a Gaussian with the same mean and variance as
real states, or real states with added Gaussian noise); and a user-chosen
reference measure $\rhoref$ used to parameterize $M$
via $M(s_1,\d s_2)=\delta_{s_1}(\d s_2)+m_\theta(s_1,s_2)\,\rhoref(\d
s_2)$. The measure $\rhoref$ is not necessarily of mass $1$, and may for instance be
the Lebesgue measure. 

We assume that the ratio $\rhosimple/\rhoref$ is known; this is
the case for instance if we take $\rhoref\deq \rhosimple$, or if $\rhoref$ is the
Lebesgue measure and $\rhosimple$ is Gaussian. The simplest case is to
use an arbitrary $\rhosimple$ and set
$\rhoref=\rhosimple$: in that case all expressions are the same as before, but
they correspond to different learning rates at different states
depending on $\rhosimple$ (since $\rhosimple$ controls how we sample
states), and to learning the density
$m_\theta$ of $M$ with respect to $\rhosimple$ not $\rho$.

The parametric TD update for $M$ becomes
\begin{multline}
\label{eq:paramtd_sss}
\E_{ s\sim \rho, \,s'\sim P(s,\d s'), \,s_2\sim\rhosimple }
\left[
\gamma \,\partial_\theta m_{\theta_t}(s,s')\, \frac{\rhosimple(\d
s')}{\rhoref(\d s')}
\right.
\\+\left.
\partial_\theta m_{\theta_t}(s,s_2)\,
(\gamma m_{\theta_t}(s',s_2)-m_{\theta_t}(s,s_2))
\right].
\end{multline}
\todo{backward TD!}
The parametric update \eqref{eq:Vparam} for $V$ becomes
\begin{equation}
\label{eq:Vparam_sss}
\left(
r_s+\gamma V_{\phi_t}(s')-V_{\phi_t}(s)
\right)
\left(
\partial_\phi V_{\phi_t}(s)
\,\frac{\rhosimple(\d s)}{\rhoref(\d s)}
+
\,\E_{ s_1\sim \rhosimple} [m_{\theta_t}(s_1,s) \,\partial_\phi
V_{\phi_t}(s_1)]
\right).
\end{equation}
The parametric Bellman--Newton update \eqref{eq:Mparam} for $M$ becomes
\begin{multline}
\label{eq:Mparam_sss}
\E_{ s_1\sim \rhosimple,\,s_2\sim\rhosimple}
\left[
\gamma\,\partial_\theta m_{\theta_t}(s,s')
\,\frac{\rhosimple(\d s)}{\rhoref(\d s)}
\frac{\rhosimple(\d s')}{\rhoref(\d s')}
\right.
\\\mathbin{+}
\gamma \, m_{\theta_t}(s_1,s)\,
\partial_\theta m_{\theta_t}(s_1,s')
\,\frac{\rhosimple(\d s')}{\rhoref(\d s')}
\\+\left.
(\gamma m_{\theta_t}(s',s_2)-m_{\theta_t}(s,s_2))
\left(\partial_\theta m_{\theta_t}(s,s_2)
\,\frac{\rhosimple(\d s)}{\rhoref(\d s)}
+
m_{\theta_t}(s_1,s)
\,\partial_\theta m_{\theta_t}(s_1,s_2)\right)
\right].
\end{multline}

We now justify each of these updates in turn, by deriving them in the
same way as above, but using different norms and learning rates.

On the other hand, for various reasons this does not work for backward TD 
(even if $\rho$ is the invariant distribution from the Markov
process). Reversing time changes the parameterization of $M$: instead of
$\Id+m(s_1,s_2)\rhoref(\d s_2)$ with a user-chosen factor on $s_2$, one
gets a user-chosen factor on $s_1$.

Given three measures $\rho_1$, $\rho_2$, and $\rhoref$ (not necessarily of mass $1$), and two measure-valued
functions $M_1(s,\d s')$ and $M_2(s,\d s')$ on $\mathcal{S}$, we define
the norm
\begin{equation}
\label{eq:norm2}
\norm{M_1-M_2}^2_{\rho_1,\rho_2,\rhoref} \deq \iint
(m_1(s,s')-m_2(s,s'))^2\,\rho_1(\d s)\,\rho_2(\d s')
\end{equation}
where $m_1(s,s')\deq M_1(s,\d s')/\rhoref(\d s')$ is the density of $M_1$
with respect to $\rhoref$ (if it exists; if not, the norm is infinite),
and likewise for $M_2$. This generalizes the norm \eqref{eq:norm}.

\begin{thm}[ (TD for successor states with
function approximation and simple sample states)]
\label{thm:paramtd_sss}
Maintain a parametric model of $M$
via
$M_{\theta_t}(s_1,\d s_2)=\delta_{s_1}(\d
s_2)+m_{\theta_t}(s_1,s_2)\rhoref(\d s_2)$, with $\theta_t$ the value
of the parameter at step $t$, and with $m_\theta$ some smooth family of
functions over pairs of states.

Define a target update of $M$ via the Bellman equation, $M^\tar\deq
\Id+\gamma P M_{\theta_t}$.
Define the loss between $M$ and $M^\tar$ via $J(\theta)\deq \frac12
\norm{M_\theta-M^\tar}^2_{\rho,
\rhosimple,\rhoref}$ using the norm \eqref{eq:norm2}.

Then the update \eqref{eq:paramtd_sss} is equal to the
gradient step $-\partial_\theta J(\theta)_{|\theta=\theta_t}$.
\end{thm}

For the updates of $V$ and $M$, we will assume that we learn the implicit
Markov process $\hat P$ and $\hat R$ with state-dependent learning rates
inversely proportional to $\rhoref$. (The standard case has $1/n_s$ for
the learning rates; since $n_s\sim t \rho_s$, this produces
learning rates inversely proportional to $\rho$.)

Namely, let $\eta_t\to 0$
be an overall learning rate schedule. Upon observing a transition $s\to s'$ with
reward $r_s$, learn $\hat P$ and $\hat R$ via
\begin{align}
\label{eq:Pupdate_sss}
\hat P_{ss_2} &\gets \hat P_{ss_2}+\frac{\eta_t}{\rhoref(s)}(\1_{s_2=s'}-\hat
P_{ss_2})\qquad
\forall s_2
\\
\label{eq:Rupdate_sss}
\hat R_s &\gets \hat R_s +\frac{\eta_t}{\rhoref(s)} (r_s-\hat R_s).
\end{align}
Thus, different states learn at different speeds, but this still converges
to the true values when $t\to \infty$.

\begin{thm}[ (Value function update via successor states with simple
sample states)]
\label{thm:VfromM_sss}
Consider the empirical estimates $\hat P$ and $\hat R$ of a finite Markov
reward process.
%
Let $s\to s'$ be the transition in the Markov process observed at step
$t$, with reward $r_s$. Let $\del V$ be the update of the value function
$(\Id-\gamma \hat P)^{-1}\hat R$
of the estimated process when $\hat P$ and $\hat R$ are updated by
\eqref{eq:Pupdate_sss}--\eqref{eq:Rupdate_sss}.

Given a
parametric model $V_\phi$ of the value function,
define a target update of $V$ via $V^\tar\deq V_{\phi_t}+\del V$ with $\phi_t$ the parameter at
step $t$. Define the loss between $V$ and $V^\tar$ via
$J^V(\phi)\deq\frac{1}{2}
\norm{V_\phi-V^\tar}^2_{L^2(\rhosimple)}$. Assume $\hat M=(\Id-\gamma
\hat P)^{-1}$ is
given by \eqref{eq:Mmodel}.

Then, up to $O(\eta_t^2)$, the gradient step $-\partial_\phi
J^V(\phi)_{\phi=\phi_t}$ is
$\eta_t$ times
\eqref{eq:Vparam_sss}.
\end{thm}

\begin{thm}[ (Successor states via online inversion, with
function approximation and simple sample states)]
\label{thm:paramBN_sss}
Maintain a parametric model of $M$ via
$M_{\theta_t}(s_1,\d s_2)=\delta_{s_1}(\d
s_2)+m_{\theta_t}(s_1,s_2)\rhoref(\d s_2)$,
with $\theta_t$ the value
of the parameter at step $t$, and with $m_\theta$ some smooth family of
functions over pairs of states.

Let $s\to s'$ be the transition in the Markov process observed at step
$t$, with reward $r_s$.
Let $\del M$ be the update of $(\Id-\hat
P)^{-1}$ corresponding to the update \eqref{eq:Pupdate_sss} of $\hat P$.

Define a target update of $M$ by $M^\tar\deq M_{\theta_t}+\del M$.
Define the loss between $M$ and $M^\tar$ via $J(\theta)\deq \frac12
\norm{M_\theta-M^\tar}_{\rhosimple,\rhosimple,\rhoref}$ using the norm \eqref{eq:norm2}.

Then, up to $O(\eta_t^2)$ the
gradient step $-\partial_\theta J(\theta)_{|\theta=\theta_t}$ is
$\eta_t$ times \eqref{eq:Mparam_sss}.
\end{thm}

The proofs of these theorems are identical to their analogues with a
single measure $\rho$, up to tracking where
$\rhosimple$ and $\rhoref$ appear instead of $\rho$ at suitable places;
they are omitted.

\section{Formal Approach to Theorem~\ref{thm:paramBN} for Continuous
Environments}
\label{sec:formalsmproof}

Contrary to TD on $M$, for Theorem~\ref{thm:paramBN}, we have defined the update
for a single transition $s\to s'$. The resulting parametric update makes
sense in any state space. But strictly speaking, this restricts
the statement of
Theorem~\ref{thm:paramBN} to discrete spaces, since it is defined via the
tabular update \eqref{eq:onlineM} which is defined only in discrete
spaces.

For TD on $M$ in general spaces (Theorem~\ref{thm:paramtd}), we directly defined the Bellman operator
on any space; the Bellman operator does not depend on a single transition
$s\to s'$, but it updates all states $s$ at once.

It is possible to proceed
analogously for Theorem~\ref{thm:paramBN}: this provides a rigorous proof
of Theorem~\ref{thm:paramBN} for general state spaces, in expectation
over the transition $s\to s'$.

We first have to define the analogue of the Bellman operator. We do this
by going back to discrete states and averaging the updates $\del M$ and
$\del V$ over transitions $s\to s'$. Averaging
\eqref{eq:onlineM} and \eqref{eq:onlineV} yields (omitting again the
$o(1/n_s)=o(1/t)$ terms)
\begin{equation}
\E_{s\sim\rho,\,s'\sim P_{ss'}} \del M_{s_1s_2}=
\sum_s \sum_{s'} \frac{\rho_s}{n_s} P_{ss'} \hat
M_{s_1s}(\1_{s_2=s}+\gamma \hat M_{s's_2}-M_{ss_2})
\end{equation}
and
\begin{equation}
\E_{s\sim\rho,\,s'\sim P_{ss'}} \del V_{s_1}= \sum_s \sum_{s'}
\frac{\rho_s}{n_s} P_{ss'} (r_s+\gamma \hat V_{s'}-\hat V_s) \hat M_{s_1
s}.
\end{equation}
Once more, since $n_s\sim t\rho_s$ when $s\to \infty$, we have
$\frac{\rho_s}{n_s}=1/t+o(1/t)$. Thus, up to $o(1/t)$ terms, the above
rewrite as
\begin{align}
\E_{s\sim\rho,\,s'\sim P_{ss'}} \del M_{s_1s_2}&=\tfrac1t
\sum_s \sum_{s'} P_{ss'} \hat
M_{s_1s}(\1_{s_2=s}+\gamma \hat M_{s's_2}-M_{ss_2})
\\
\E_{s\sim\rho,\,s'\sim P_{ss'}} \del V_{s_1}&= \tfrac1t \sum_s \sum_{s'}
P_{ss'} (r_s+\gamma \hat V_{s'}-\hat V_s) \hat M_{s_1
s}.
\end{align}
Since $\sum_{s'} P_{ss'}=1$ and $\sum_{s,s'} \hat
M_{s_1s}P_{ss'}\hat M_{s's_2}=(\hat MP\hat M)_{s_1s_2}$ and likewise
$\sum_s \hat
M_{s_1s}\hat M_{ss_2} =(\hat M^2)_{s_1s_2}$, the
update of $M$ rewrites as
\begin{equation}
\label{eq:BNproof}
\E_{s\sim\rho,\,s'\sim P_{ss'}}
\del M=\tfrac1t (\hat M +\gamma \hat MP\hat M-\hat M^2).
\end{equation}
Likewise, for the update of $\hat V$, since $\E r_s=R_s$ and $\sum_{s'}
P_{ss'}(\gamma \hat V_{s'}-\hat V_s)=(\gamma P\hat V-\hat V)_s$, we have
\begin{equation}
\E_{s\sim\rho,\,s'\sim P_{ss'}}\del V=\tfrac1t \hat M(R+\gamma P\hat
V-\hat V).
\end{equation}

Thus, in the continuous case, we can define target updates at step $t$ by
\begin{equation}
\label{eq:smtar}
M^\tar\deq M_{\theta_t} + \frac1t (M_{\theta_t}-M_{\theta_t}(\Id-\gamma
P)M_{\theta_t})
\end{equation}
(well-defined for continuous states as an operator on functions)
and
\begin{equation}
V^\tar\deq V_{\phi_t}+\frac1t M_{\theta_t} (R+\gamma P
V_{\phi_t}-V_{\phi_t})
\end{equation}
and define, as before, the losses
\begin{equation}
J(\theta)\deq \frac12 \norm{M_\theta-M^\tar}^2_{\rho}
\end{equation}
and
\begin{equation}
J^V(\phi)\deq \frac12 \norm{V_\phi-V^\tar}^2_{L^2(\rho)}.
\end{equation}

From now on we only work with $M$, as the computation for $V$ is similar
but simpler.

As in the proof of Theorem~\ref{thm:paramtd}, 
by definition of $J(\theta)$ and of the norm
$\norm{\cdot}_{\rho}$, we have
\begin{equation}
J(\theta)=\frac12 \iint j_\theta(s_1,s_2)^2\,\rho(\d s_1)\rho(\d s_2)
\end{equation}
and
\begin{equation}
\label{eq:dj2}
\partial_\theta J(\theta)=\iint j_\theta(s_1,s_2)\,\partial_\theta
j_\theta(s_1,s_2) \rho(\d s_1)\rho(\d s_2)
\end{equation}
where 
\begin{equation}
\label{eq:defj2}
j_\theta(s_1,s_2)\deq (M^\tar(s_1,\d s_2)-M_\theta(s_1,\d
s_2))/\rho(\d s_2)
\end{equation}

Since $M^\tar$ does not depend on $\theta$ (it depends on $\theta_t$, but
we optimize with respect to $\theta$ for fixed $M^\tar$), we have
\begin{equation}
\label{eq:djism2}
\partial_\theta j_\theta(s_1,s_2)=\partial_\theta
\left(-\frac{M_\theta(s_1,\d s_2)}{\rho(\d s_2)}\right)
=-\partial_\theta m_\theta(s_1,s_2)
\end{equation}
in the parameteriation $M_{\theta}(s_1,\d
s_2)=\Id+m_{\theta}(s_1,s_2)\rho(\d s_2)$.

Consequently, by \eqref{eq:dj2}, \eqref{eq:defj2} and \eqref{eq:djism2}, at $\theta=\theta_t$ we have
\begin{align}
-\partial_\theta J(\theta)_{|\theta=\theta_t}&=
\iint \left(
\frac{M^\tar(s_1,\d s_2)-M_{\theta_t}(s_1,\d s_2)}{\rho(\d s_2)}
\right)
\partial_\theta m_{\theta_t}(s_1,s_2) \rho(\d s_1)\rho(\d s_2)
\\&=
\label{eq:abstractsmstep}
\iint \left(
M^\tar(s_1,\d s_2)-M_{\theta_t}(s_1,\d s_2)
\right)
\partial_\theta m_{\theta_t}(s_1,s_2) \rho(\d s_1)
\end{align}

Define $A_\theta\deq M_\theta-\Id$. By definition of the parameterization
$M_{\theta}(s_1,\d s_2)=\Id+m_{\theta}(s_1,s_2)\rho(\d s_2)$, we have
\begin{equation}
\label{eq:defA}
A_\theta(s_1,\d s_2)=m_{\theta}(s_1,s_2)\rho(\d s_2).
\end{equation}

By a direct but tedious substitution
of $M_{\theta_t}=\Id+A_{\theta_t}$ in the expression \eqref{eq:smtar} for $M^\tar$, we find
\begin{equation}
M^\tar-M_{\theta_t}=\tfrac1t(
\gamma P+\gamma A_{\theta_t} P+ \gamma PA_{\theta_t}-A_{\theta_t}+A_{\theta_t}\gamma PA_{\theta_t}-A_{\theta_t}^2
)
\end{equation}
as operators, with the product of operator denoting composition. (For instance, for a function $f$, the operator $PA$ acts
by $(PAf)(s)=\int P(s,\d
s') A(s',\d s_2)f(s_2)$.)

Let us study the contributions of all the terms of $M^\tar-M_{\theta_t}$
to the gradient step \eqref{eq:abstractsmstep}. The $\gamma P$ term
provides a contribution
\begin{equation}
\iint \gamma P(s_1,\d s_2) \,\partial_\theta m_{\theta_t} (s_1,s_2)
\rho(\d s_1)=\gamma \E_{s\sim \rho,\,s'\sim P(s,\d s')} \,\partial_\theta
m_{\theta_t} (s,s').
\end{equation}
Next, by \eqref{eq:defA} we have
\begin{align}
(A_{\theta_t}P)(s_1,\d s_2)&=
\int  A_{\theta_t}(s_1,\d s){P(s,\d s_2)}
\\
&=\int  m_{\theta_t}(s_1,s)\rho(\d s) {P(s,\d s_2)}
\end{align}
and therefore, the
$\gamma A_{\theta_t}P$ term provides a contribution
\begin{multline}
\iint \gamma (A_{\theta_t}P)(s_1,\d s_2) \,\partial_\theta m_{\theta_t}
(s_1,s_2)\rho(\d s_1)
\\=\gamma 
\iiint m_{\theta_t}(s_1,s)\rho(\d s) {P(s,\d s_2)}\,\partial_\theta
m_{\theta_t}
(s_1,s_2)\rho(\d s_1)
\\=
\gamma \E_{s_1\sim \rho,\, s\sim \rho,\,s'\sim P(s,\d s')}
[m_{\theta_t}(s_1,s)\,\partial_\theta m_{\theta_t}(s_1,s')].
\end{multline}
Next, by \eqref{eq:defA} we have
\begin{equation}
(PA_{\theta_t})(s,\d s_2)=\int {P(s,\d s')} A_{\theta_t}(s',\d s_2)=\E_{s'\sim P(s,\d
s')} \,m_{\theta_t}(s',s_2)\rho(\d s_2)
\end{equation}
and therefore, the
$\gamma P A_{\theta_t}$ term provides a contribution
\begin{equation}
\iint \gamma (PA_{\theta_t})(s_1,\d s_2) \,\partial_\theta m_{\theta_t}
(s_1,s_2)\rho(\d s_1)
=
\gamma \E_{s\sim \rho,\,s'\sim P(s,\d s'),\,s_2\sim
\rho}\,m_{\theta_t}(s',s_2)\,\partial_\theta m_{\theta_t}(s,s_2)
\end{equation}
and by \eqref{eq:defA}, the term $-A_{\theta_t}$ provides a contribution
\begin{equation}
-\iint A_{\theta_t}(s_1,\d s_2) \,\partial_\theta m_{\theta_t}
(s_1,s_2)\rho(\d s_1)
=
-\E_{s\sim \rho,\,s_2\sim
\rho}\,m_{\theta_t}(s,s_2)\,\partial_\theta m_{\theta_t}(s,s_2).
\end{equation}
Next, the term $-A_{\theta_t}^2$ provides a contribution
\begin{multline}
-\iint (A_{\theta_t}^2)(s_1,\d s_2) \,\partial_\theta m_{\theta_t}
(s_1,s_2)\rho(\d s_1)
\\=-\iiint A_{\theta_t}(s_1,\d s) A_{\theta_t}(s,\d s_2) \,\partial_\theta
m_{\theta_t}
(s_1,s_2)\rho(\d s_1)
\\=
-\iiint m_{\theta_t}(s_1,s)\rho(\d s) m_{\theta_t}(s,s_2)\rho(\d
s_2)\,\partial_\theta
m_{\theta_t}
(s_1,s_2)\rho(\d s_1)
\\=-\E_{s\sim \rho,\, s_1\sim \rho,\, s_2\sim \rho}
\,m_{\theta_t}(s_1,s)m_{\theta_t}(s,s_2)
\,\partial_\theta
m_{\theta_t}
(s_1,s_2)
.
\end{multline}
For the final tem $\gamma A_{\theta_t}P
A_{\theta_t}$, observe that
\begin{align}
(A_{\theta_t}P
A_{\theta_t})(s_1,\d s_2)
&=\iint A_{\theta_t}(s_1,\d s)P(s,\d s')A_{\theta_t(s',\d s_2)}
\\&= \iint m_{\theta_t}(s_1,s)\rho(\d s)P(s,\d s')m_{\theta_t(s',s_2)}\rho(\d s_2)
\\&= \E_{s\sim \rho,\,s'\sim P(s,\d s')}\,
m_{\theta_t}(s_1,s)m_{\theta_t}(s',s_2)\rho(\d s_2)
\end{align}
and therefore, the contribution of the term $\gamma A_{\theta_t}P
A_{\theta_t}$ is
\begin{multline}
\gamma \iint (A_{\theta_t}PA_{\theta_t})(s_1,\d s_2) \,\partial_\theta m_{\theta_t}
(s_1,s_2)\rho(\d s_1)
\\=\gamma \E_{s_1\sim \rho,\,s_2\sim \rho,\,s\sim \rho,\,s'\sim P(s,\d s')}
\, m_{\theta_t}(s_1,s)m_{\theta_t(s',s_2)}\,\partial_\theta m_{\theta_t}
(s_1,s_2).
\end{multline}
Collecting everything, we find
\begin{multline}
-\partial_\theta J(\theta)_{|\theta=\theta_t}=
\E_{s_1\sim \rho,\,s_2\sim \rho,\,s\sim \rho,\,s'\sim P(s,\d s')}
\left[
\gamma \partial_\theta m_{\theta_t}(s,s')
+\gamma m_{\theta_t}(s_1,s)\,\partial m_{\theta_t}(s_1,s')
\right.
\\\mathbin{+}
\left.
\gamma m_{\theta_t}(s',s_2)\,\partial m_{\theta_t}(s,s_2)
- m_{\theta_t}(s,s_2)\,\partial m_{\theta_t}(s,s_2)
\right.
\\\mathbin{-}
\left.
m_{\theta_t}(s_1,s)m_{\theta_t}(s,s_2)
\,\partial_\theta
m_{\theta_t}
(s_1,s_2)
+\gamma m_{\theta_t}(s_1,s)m_{\theta_t}(s',s_2)\,\partial_\theta m_{\theta_t}
(s_1,s_2)
\right].
\end{multline}
This is the expectation over $s\sim \rho,\,s'\sim P(s,\d s')$, of the
update \eqref{eq:Mparam}. This formally proves Theorem~\ref{thm:paramBN}
for general state spaces, in expectation over $(s,s')$.


\section{Background on Singular Value Decompositions}
\label{sec:svd}

In the text, we often work with the space of functions over $S$ equipped
with the $L^2(\rho)$ norm. Since $\rho\neq \Id$, we include here a
reminder on how the usual notions of Euclidean vector spaces play out in
non-orthonormal bases. We also include details on what constitutes a
``truncated singular value decomposition''.

A Euclidean vector space $E$ is a finite-dimensional vector space equipped with a dot product; the
dot product is given by some
symmetric, positive definite matrix $q$ in some basis, namely, $\langle
x,y\rangle_E=\transp{x}qy$ for any vectors $x$ and $y$.

If $A\from E_1 \to E_2$ is a linear map between two Euclidean spaces, its
adjoint $A^\ast$ is the map from $E_2$ to $E_1$ such that $\langle
y,Ax\rangle_{E_2}=\langle A^\ast y,x\rangle_{E_1}$ for any vectors $x\in
E_1$
and $y\in E_2$. Expressed in bases of $E_1$ and $E_2$, its matrix is $A^\ast=q_1^{-1}
\transp{A} q_2$, or just $\transp{A}$ if the bases are orthonormal.

Such a map $A$ is orthogonal if $AA^{\ast}=\Id_{E_2}$ and $A^\ast
A=\Id_{E_1}$.

The Hilbert--Schmidt norm for an operator $M$ on a Euclidean vector space
is $\Tr(M^\ast M )$ where $M^\ast$ is the adjoint of $M$. In an
orthonormal basis this is $\Tr(\transp{M}M)$ viewing $M$ as a matrix, but
in a non-orthonormal basis this is $\Tr(q^{-1}\transp{M}qM)$ where $q$ is
the matrix defining the norm in the basis.

A \emph{singular value decomposition} of such a map $A$ is a triplet of
linear maps $U\from \R^{\dim(E_2)}\to E_2$, $D\from \R^{\dim(E_1)}\to
\R^{\dim(E_2)}$ and $V\from \R^{\dim(E_1)}\to E_1$ such that
$A=UDV^\ast$, $U$ and $V$ are orthogonal, and $D$ is rectangular
diagonal. Equivalently, a singular value decomposition can be written as
$Ax=\sum_i u_i d_i \langle v_i,x\rangle_{E_1}$ where each $d_i> 0$, the $u_i$'s make an
orthonormal family in $E_2$, and the $v_i$'s make an orthonormal family
in $E_1$ (or equivalently, an orthonormal family of linear forms on $E_1$
by identifying $v_i$ with the map $x\mapsto \langle v_i,x\rangle_{E_1}$).

\begin{defi}[ (Truncated SVD)]
\label{def:truncsvd}
A linear map $B$ is a \emph{truncated singular value decomposition} of a
linear map $A\from E_1\to E_2$ if there is a singular value decomposition
$A=UDV^\ast$ of $A$ and a rectangular diagonal matrix $D'$ such that $D'$
is obtained from $D$ by replacing some elements with $0$, and
$B=UD'V^\ast$.
\end{defi}

\begin{lem}
\label{lem:truncsvd}
A linear map $B\from E_1\to E_2$ is a truncated singular value
decomposition of $A\from E_1\to E_2$ if and only if $A$ and $B$ are equal
on $(\Ker B)^\bot$ and the image of $\Ker B$ by $A$ is orthogonal to the
image of $B$.
\end{lem}

\begin{proof}
$(\Leftarrow)$
Define $E'_1=\Ker B$ and $E''_1=(\Ker B)^\bot$ so that $E_1=E'_1\oplus
E''_1$.
Let $A'$ and $A''$ be the restrictions of $A$ to $E'_1$ and $E''_1$
respectively, so that $A=A'+A''$. Define $B'$ and $B''$ likewise.

Since $E'_1$ is $\Ker B$, we have $B'=0$ so $B=B''$.

By assumption, $A$ and $B$ are equal on $E''_1$. Therefore,
$A''=B''$, so $B=A''$.

By assumption, the image of $E'_1$ by $A$ is orthogonal to the image of
$B$. The former is $\Img A'$ while the latter is $\Img A''$. Therefore,
$\Img A'\bot \Img A''$.

Consider singular value decompositions of $A'$ and $A''$ as
$A'=\sum_i u'_i d'_i v'_i$ and $A''=\sum_j u''_j d''_j v''_j$, where
the $d'_i$ are positive real numbers, the $u'_i$ are an orthonormal basis
of $\Img A'$, the $v'_i$ are an orthonormal set of linear forms on $E'_1$, and
likewise for $A''$. (Any zero singular values have been dropped in this
decomposition.)

Since $\Img A'\bot \Img A''$, the $u'_i\,$'s are orthogonal to the
$u''_j\,$'s. Likewise, since the decomposition $E_1=E'_1\oplus E''_1$ is
orthogonal, the $v'_i\,$'s are orthogonal to the $v''_j\,$'s as linear
forms on $E_1$.

It follows that $\sum_i u'_i d'_i v'_i+\sum_j u''_j d''_j v''_j$ is a
singular value decomposition of $A$ (with the zero singular values
omitted). Since $B=A''$, $\sum_j u''_j
d''_j v''_j$ is a singular value decomposition of $B$, so that $B$ is a
truncated SVD of $A$.

$(\Rightarrow)$ Let $A=UDV^\ast$ and $B=UD'V^\ast$ as in
Definition~\ref{def:truncsvd}.
Up to swapping rows and columns, we can assume that the
nonzero entries of $D$ and $D'$ are located in the first rows. Let
$k$ be the number of nonzero entries in $D'$. Then $\Ker D'$ is spanned
by the last $\dim(E_1)-k$ basis vectors in $\R^{\dim(E_1)}$, and $(\Ker
D')^\bot$ is spanned by the first $k$ basis vectors. Thus, by construction, $D$
and $D'$ coincide on $(\Ker
D')^\bot$. Moreover, $\Img D'$ is spanned by the first $k$ basis vectors,
and $D(\Ker D')$ is spanned by the last $\dim(E_1)-k$ basis vectors, so
$\Img D'$ and $D(\Ker D')$ are orthogonal.

Since $A=UDV^\ast$ and $B=UD'V^\ast$, and since $U$ is invertible, $A$ and $B$ are equal on $(\Ker
B)^\bot$ if and only if $DV^\ast$ and $D'V^\ast$ are equal on $(\Ker
B)^\bot$. Since $V^\ast$ is invertible, this happens if and only
if $D$ and $D'$ are equal on $V^\ast((\Ker B)^\bot)$. Since $V^\ast$ is orthogonal,
the latter is $(V^\ast(\Ker B))^\bot$.

Since $U$ and
$V$ are orthogonal, hence invertible, one has $\Ker B=\Ker(UD'V^\ast)=\Ker(D'V^\ast)=V(\Ker
D')$. Hence $V^\ast (\Ker B)=\Ker D'$. Thus, we need $D$ and $D'$ to be
orthogonal on $\Ker D'$, which we have established above.

Next, let us prove that $A(\Ker B)\bot \Img B$, namely, that $UDV^\ast(\Ker
B)\bot \Img (UD'V^\ast)$. Since $U$ is orthogonal, this is equivalent to
$DV^\ast(\Ker B)\bot \Img (D'V^\ast)$. We have seen that $V^\ast(\Ker
B)=\Ker D'$; moreover $\Img (D'V^\ast)\subset \Img(D')$, so it is enough
to prove that $D(\Ker D')\bot \Img D'$, which we have established above.
This proves the first part of the equivalence.
\end{proof}

\end{document}